\documentclass{article}

\usepackage[final]{neurips_2026}

\usepackage[utf8]{inputenc} 
\usepackage[T1]{fontenc}    
\usepackage[colorlinks,linkcolor=blue,citecolor=blue,pagebackref=true]{hyperref}
\usepackage{url}            
\usepackage{booktabs}       
\usepackage{amsfonts}       
\usepackage{nicefrac}       
\usepackage{microtype}      
\usepackage{xcolor}         
\usepackage{graphicx}
\makeatletter
\def\@trackname{}
\makeatother
\usepackage{subcaption}
\usepackage{algorithm}
\usepackage{wrapfig}
\usepackage{amsmath}
\usepackage{amssymb}
\usepackage{mathtools}
\usepackage{amsthm}
\usepackage{xcolor}
\usepackage{enumitem}
\definecolor{myhighlight}{RGB}{235,185,150} 
\usepackage[capitalize,noabbrev]{cleveref}

\newtheorem{theorem}{Theorem}
\newtheorem{lemma}{Lemma}
\newtheorem{proposition}{Proposition}

\newtheorem{assumption}{Assumption}
\newtheorem{definition}{Definition}

\crefname{theorem}{Theorem}{Theorems}
\Crefname{theorem}{Theorem}{Theorems}

\crefname{lemma}{Lemma}{Lemmas}
\Crefname{lemma}{Lemma}{Lemmas}

\crefname{proposition}{Proposition}{Propositions}
\Crefname{proposition}{Proposition}{Propositions}

\crefname{corollary}{Corollary}{Corollaries}
\Crefname{corollary}{Corollary}{Corollaries}

\crefname{assumption}{Assumption}{Assumptions}
\Crefname{assumption}{Assumption}{Assumptions}

\crefname{definition}{Definition}{Definitions}
\Crefname{definition}{Definition}{Definitions}

\newcommand{\decp}[1]{\textcolor{red}{\scriptsize$\downarrow$#1\%}}

\usepackage[textsize=tiny]{todonotes}
\usepackage{titlesec}

\titlespacing*{\section}{0pt}{1.5ex plus 0.5ex minus 0.2ex}{0.8ex plus 0.2ex}
\titlespacing*{\subsection}{0pt}{1.2ex plus 0.4ex minus 0.2ex}{0.6ex plus 0.2ex}
\titlespacing*{\subsubsection}{0pt}{1.0ex plus 0.3ex minus 0.2ex}{0.5ex plus 0.2ex}
\titlespacing*{\paragraph}{0pt}{0.8ex plus 0.2ex minus 0.1ex}{0.4ex}

\crefname{theorem}{Theorem}{Theorems}
\Crefname{theorem}{Theorem}{Theorems}

\crefname{lemma}{Lemma}{Lemmas}
\Crefname{lemma}{Lemma}{Lemmas}

\crefname{proposition}{Proposition}{Propositions}
\Crefname{proposition}{Proposition}{Propositions}

\crefname{corollary}{Corollary}{Corollaries}
\Crefname{corollary}{Corollary}{Corollaries}

\crefname{assumption}{Assumption}{Assumptions}
\Crefname{assumption}{Assumption}{Assumptions}

\crefname{definition}{Definition}{Definitions}
\Crefname{definition}{Definition}{Definitions}

\usepackage{amsmath,amsfonts,bm}

\def\eqref#1{equation~\ref{#1}}

\def\1{\bm{1}}

\DeclareMathAlphabet{\mathsfit}{\encodingdefault}{\sfdefault}{m}{sl}
\SetMathAlphabet{\mathsfit}{bold}{\encodingdefault}{\sfdefault}{bx}{n}

\usepackage{xcolor}
\definecolor{forestgreen}{rgb}{0.13, 0.55, 0.13}

\definecolor{frenchblue}{rgb}{0.0, 0.45, 0.73}

\definecolor{cherryblossompink}{rgb}{1.0, 0.72, 0.77}

\definecolor{bittersweet}{rgb}{1.0, 0.44, 0.37}

\definecolor{navyblue}{rgb}{0.0, 0.0, 0.5}

\usepackage{booktabs}       
\usepackage{amsfonts}       
\usepackage{nicefrac}       
\usepackage{microtype}      
\usepackage{graphicx}
\usepackage{float}
\usepackage{mathtools}
\usepackage{bbm}
\usepackage{amssymb}
\usepackage{amsmath}
\usepackage{wrapfig}
\usepackage{color}
\usepackage{hyperref}
\usepackage{url}
\usepackage{amsmath}
\usepackage{algpseudocode}
\usepackage{algpseudocode}

\usepackage{tabularx}
\usepackage{makecell}
\usepackage{booktabs}
\usepackage{svg}
\usepackage{tikz}
\usepackage{multirow}
\usepackage{multicol}
\usepackage{colortbl}
\usepackage{caption}
\usepackage{subcaption}
\usepackage{multirow}
\usepackage{soul}

\definecolor{LightCyan}{rgb}{0.88,1,1}
\definecolor{Gray}{gray}{0.95}

\usepackage{bm}
\usepackage{mathtools}

\usepackage{multirow}
\usepackage{xspace}
\usepackage{tocloft}

\usepackage[textsize=tiny]{todonotes}

\definecolor{cyan}{cmyk}{.3,0,0,0}
\definecolor{darkcyan}{rgb}{0.0, 0.55, 0.55}

\usepackage{amsthm}

\newcommand{\sd}{\text{SD}}

\usepackage{bm}

\usepackage{comment}

\newcommand{\cE}{\mathcal{E}}

\newcommand{\cG}{\mathcal{G}}

\newcommand{\cL}{\mathcal{L}}

\newcommand{\cP}{\mathcal{P}}

\newcommand{\cV}{\mathcal{V}}

\newcommand{\bfA}{\mathbf{A}}

\newcommand{\bfB}{\mathbf{B}}

\newcommand{\bfD}{\mathbf{D}}

\newcommand{\bfI}{\mathbf{I}}

\newcommand{\bfL}{\mathbf{L}}

\newcommand{\bfu}{\mathbf{u}}

\newcommand{\bfW}{\mathbf{W}}

\newcommand{\bsalpha}{\boldsymbol{\alpha}}
\newcommand{\bsbeta}{\boldsymbol{\beta}}

\newcommand{\bslambda}{\boldsymbol{\lambda}}

\title{StructSAM: Structure- and Spectrum-Preserving Token Merging for Segment Anything Models}

\author{Duy M. H. Nguyen$^{1,2,3}$, Tuan A. Tran$^{2,4}$, Duong Nguyen$^{2}$, Siwei Xie$^{1}$, \\ Trung Q. Nguyen$^{2}$, Mai T. N. Truong$^{2}$ 
Daniel Palenicek$^{5}$, An T. Le$^{5,6,7}$, Michael Barz$^{2}$ \\, Eric Hannus$^{8}$, TrungTin Nguyen$^{9}$, Tuan Dam$^{10}$ 
Tran Le$^{11}$, Ngan Le$^{12}$, Minh Vu$^{6,7}$, \\ Khoa Doan$^{7}$, Vien Ngo$^{6,7}$, Pengtao Xie$^{13}$
James Zou$^{14}$, Daniel Sonntag$^{2,4}$\\ Jan Peters$^{2,5}$, Mathias Niepert$^{1,3}$ \\[0.5em]
$^1$University of Stuttgart \quad $^2$DFKI \\
$^3$Max Planck Research School \quad $^4$University of Oldenburg \\
$^5$Technical University of Darmstadt \quad $^6$VinRobotics \\
$^7$VinUniversity \quad $^8$Aalto University \\
$^9$Queensland University of Technology \quad $^{10}$HUST \\
$^{11}$Technical University of Denmark \quad $^{12}$University of Arkansas \\
$^{13}$UC San Diego \quad $^{14}$Stanford University}

\begin{document}

\maketitle
\vspace*{-0.2in}
\begin{abstract}
  \vspace{-0.1in}
Recent token merging techniques for Vision Transformers (ViTs) provide substantial speedups by reducing the number of tokens processed by self-attention, often without retraining. However, their direct application to the Segment Anything Model (SAM) family is nontrivial: SAM’s image encoder mixes windowed and global attention, and its mask decoder relies on dense, prompt-conditioned features for precise boundary prediction.
We systematically evaluate representative token-merging methods on the SAM family models in a strict off-the-shelf setting, and find that existing destination-selection heuristics can erode boundaries and leak prompt information as merge rates increase.
We propose \textbf{StructSAM}, a resolution-preserving framework for SAM that computes lightweight token-energy scores from first-order feature gradients, protects boundary and prompt regions via grid-based flatness screening, and merges tokens in flat areas toward low-energy targets with explicit recovery.
We further present a spectral graph coarsening view, showing that score-guided merging yields bounded Laplacian spectral distortion relative to random or window-restricted baselines.
Across five natural and medical benchmarks, StructSAM reduces encoder FLOPs by 25--30\% (up to 40\%+ with prompt-aware merging) with minor drops in mIoU/Dice, consistently outperforming the latest merging techniques at the same compute.
\end{abstract}
\vspace{-0.1in}

\addtocontents{toc}{\protect\setcounter{tocdepth}{-1}}
\section{Introduction}
\vspace{-0.1in}
SAM~\cite{kirillov2023segment} has recently emerged as a foundation model for segmentation, demonstrating remarkable generalization across diverse visual domains. By combining a powerful Vision Transformer (ViT) image encoder~\cite{dosovitskiy2020image} with prompt-conditioned mask decoding, SAM enables flexible, interactive segmentation using points, boxes, and masks, often without any task-specific training. This versatility has led to rapid adoption beyond natural images, including medical imaging, e.g., MedSAM~\cite{ma2024segment}, robot surgery~\cite{wang2023sam}, and embodied AI systems~\cite{li2025controlvla,noh2025graspsam}, where robust segmentation is a critical component. Despite these strengths, SAM's practical deployment is severely constrained by its computational cost~\cite{zhang2023faster}. In particular, SAM follows a heavy encoder–light decoder design: for large variants such as ViT-L and ViT-H, the image encoder alone accounts for over 98\% of the total model parameters and FLOPs, making inference expensive even when the downstream task is lightweight or interactive.

Recognizing this bottleneck, recent efforts to improve SAM's efficiency have largely focused on model compression strategies, including knowledge distillation to smaller architectures~\cite{zhang2023faster,zhou2025edgesam}, lightweight re-designs of the vision backbone~\cite{xiong2024efficientsam}, and post-training quantization~\cite{lv2024ptq4sam,zhang2025ahcptq}. While effective, these approaches typically require retraining or fine-tuning the image encoder on large-scale datasets or applying task-specific calibration during deployment. Such requirements limit their applicability in scenarios where SAM is used off-the-shelf, training data is unavailable, or where domain-specific fine-tuning is undesirable due to cost and stability concerns.

In practice, SAM's pre-trained representations are already highly expressive across many domains, \textbf{motivating approaches that retain the original model weights while reducing inference cost}. Recently, token merging~\cite{bolya2022token,tran2024accelerating} has emerged as a powerful, training-free strategy to accelerate Vision Transformers. By dynamically grouping redundant patches, it significantly boosts throughput for both classification and segmentation tasks~\cite{bolya2022token,norouzi2024algm} without sacrificing model performance.
\textit{However, directly applying existing token merging methods to SAM is nontrivial} (Fig. \ref{fig:sam_encoder}). Unlike conventional ViTs, (i) SAM's image encoder \textit{interleaves windowed and global attention} and preserves fine-grained spatial details, which is crucial for mask prediction. Moreover, (ii) segmentation tasks inherently \textit{require dense, structured outputs}, making aggressive token reduction~\cite{bolya2023tomesd,kim2024token} incompatible without careful unmerging or feature upsampling mechanisms.

\begin{wrapfigure}{r}{0.5\columnwidth}
    \centering
    \vspace{-0.2in}
    \includegraphics[width=0.5\columnwidth]{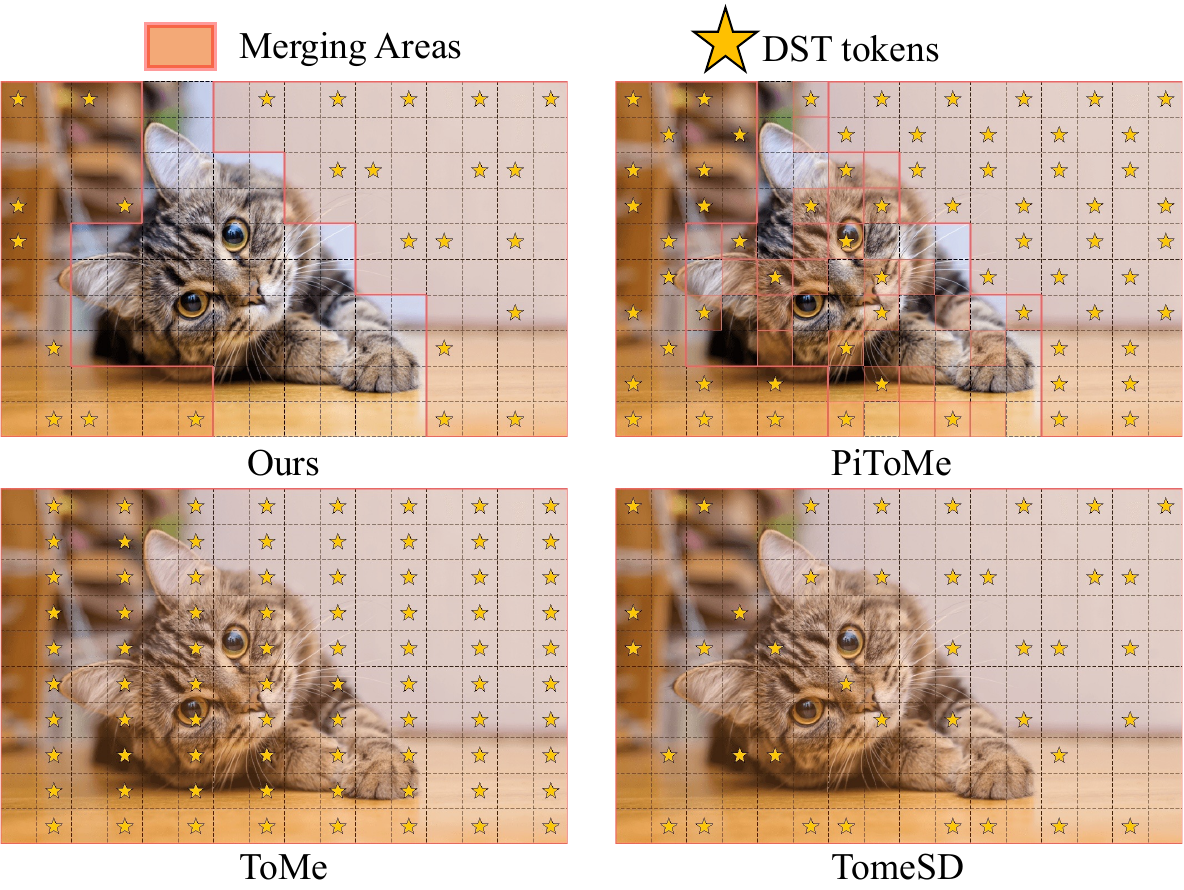}
    \caption{\textbf{Token merging strategies}. ToMe~\cite{tome} and ToMeSD~\cite{bolya2023tomesd} merge tokens using fixed or random patterns within local windows, often ignoring structural importance. PiToMe~\cite{tran2024accelerating} introduces a protected set but degrades at high merge rates. In contrast, StructSAM performs structure-aware merging by preserving boundary- and object-critical tokens (cells outside \colorbox{myhighlight}{Mergable Areas}) while merging tokens in homogeneous regions.}
    \vspace{-0.15in}
    \label{fig:TokenMergingIllustration}
\end{wrapfigure}

Motivated by these challenges, we systematically study representative token-merging techniques originally developed for ViTs in images, videos~\cite{tome,tran2024accelerating,li2024vidtome} and dense segmentation settings~\cite{norouzi2024algm,bolya2023tomesd}, and adapt them to the SAM family, including SAM, Efficient-SAM~\cite{xiong2024efficientsam}, and Medical SAM~\cite{ma2024segment}. We evaluate commonly used SAM backbones (ViT-B and ViT-L) on boundary-sensitive natural-image benchmarks and on cross-domain medical segmentation datasets, tracking tasks in robotic manipulation. Crucially, all experiments are conducted in a \textit{strict off-the-shelf setting}, without any fine-tuning, to assess the accuracy--efficiency limits of inference-time token merging for foundation segmentation models. Our findings show that prior approaches - which typically rely on random, global, or window-restricted destination selection (Fig.~\ref{fig:TokenMergingIllustration}) - struggle to preserve object boundaries and prompt-relevant regions, leading to noticeable degradation as the merge rate increases.

To mitigate this issue, we propose a \textit{structure}- and \textit{spectrum-preserving} token-merging framework tailored to SAM-style architectures, called StructSAM. Our method, as demonstrated in Figure~\ref{fig:fullpage}, (i) identifies boundary-critical tokens using a lightweight energy score computed from first-order finite differences on the encoder feature map~\cite{ziou1998edge,forsyth2002computer}, inspired by connections to spectral graph energy~\cite{balakrishnan2004energy,gutman2006laplacian}. Tokens near object boundaries exhibit large feature gradients and are therefore protected from merging, while tokens in visually flat regions can be safely merged. Building on this distinction, StructSAM (ii) groups tokens into grid-based cells and ranks cells by flatness to select mergeable regions with spatial coherence, and (iii) merges tokens within selected cells toward low-energy destinations while explicitly unmerging to recover the original token resolution required by SAM's mask decoder. When (iv) box prompts are available, a prompt-aware variant further restricts aggressive merging to tokens outside the prompted region for additional speedups. Finally, we show that our merging one admits a spectral graph-theoretic interpretation~\cite{jin2020graph,tran2024accelerating}, under which it provably preserves intrinsic spectral properties of the original token space under mild conditions, offering a principled explanation for its stability and effectiveness.

In summary, our work contributes the following points:
\begin{itemize}[topsep=0pt]  
    \setlength{\itemsep}{4pt}   
    \setlength{\topsep}{1pt}    
    \setlength{\parskip}{0pt}   
    \item We present the first systematic evaluation of \emph{inference-time} token merging for the SAM family in a strict off-the-shelf setting, revealing why existing merging strategies can fail under boundary- and prompt-sensitive segmentation.
    \item We propose a boundary- and prompt-aware, \textit{structure-preserving} merging strategy that leverages gradient-based token energy and cell flatness to protect informative regions while merging redundant background tokens. StructSAM achieves strong accuracy–efficiency trade-offs, reducing FLOPs by 25–30\% (up to 40\%+ with prompt-aware merging) while maintaining segmentation quality across benchmarks, and extends to tracking via Efficient-SAM in robotic manipulation tasks.
    \item We provide a spectral graph-theoretic analysis showing that, under mild assumptions, our score-guided merging yields a provable bound on spectral distortion, offering a principled explanation for its robustness in dense segmentation and the limitations of prior random or similarity-only merging approaches.
\end{itemize}

\section{Related Work}
\begin{figure*}[!hbt]
\centering
\vspace{-0.1in}
\includegraphics[width=0.95\textwidth,height=\textheight,keepaspectratio]{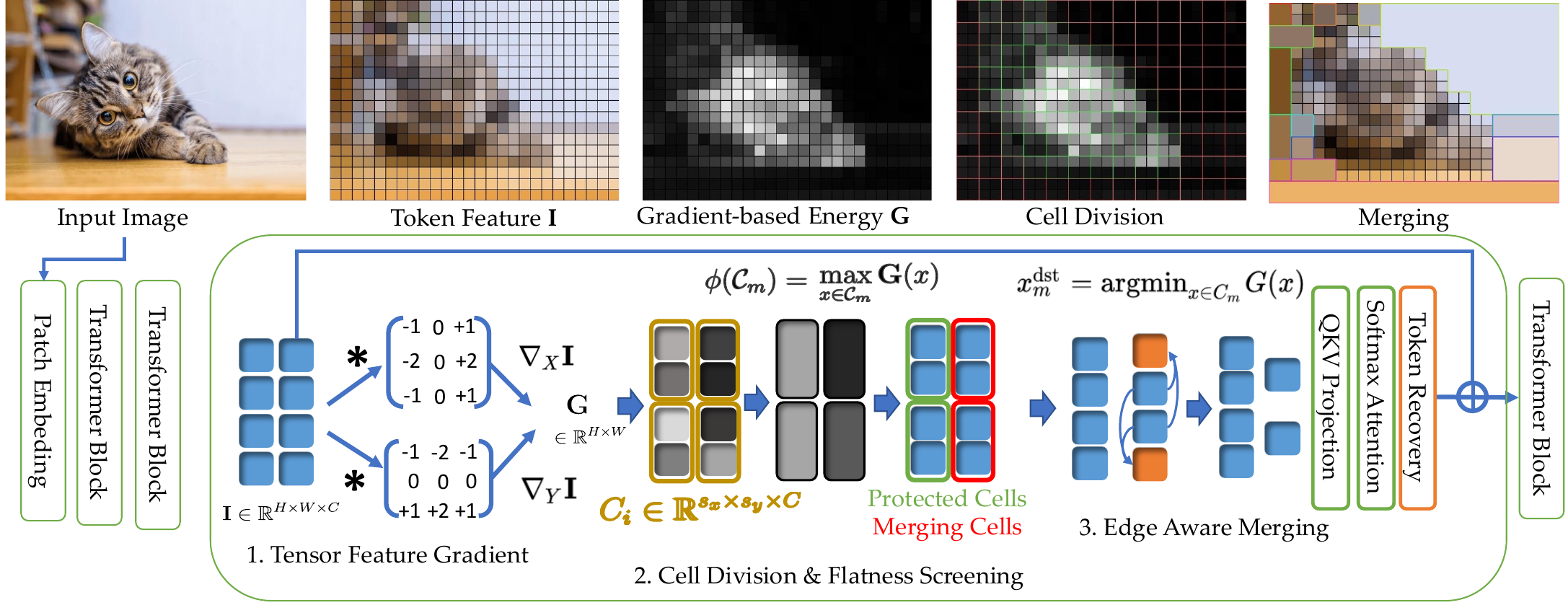}
\caption{\textbf{StructSAM overview.} Feature-gradient energy identifies structurally important tokens, forming a \textbf{protected set} that is kept at full resolution. Visually flat regions are selectively merged (one representative per mergeable cell) and followed by lightweight token recovery (unmerging), resulting in SAM’s mask decoder still receiving a dense feature grid.}
\vspace{-0.1in}
\label{fig:fullpage}
\end{figure*}
\paragraph{Compression, Distillation and Quantization Methods for SAM.}
To improve SAM's efficiency, most existing methods focus on \textit{backbone replacement or structured compression}. Approaches such as MobileSAM~\cite{zhang2023faster}, FastSAM~\cite{zhao2023fast}, EdgeSAM~\cite{zhou2025edgesam}, and EfficientSAM~\cite{xiong2024efficientsam} replace the original ViT-H encoder with lightweight architectures, including TinyViT~\cite{wu2022tinyvit}, or EfficientViT~\cite{zhang2024efficientvit}, often combined with task-specific training strategies. While effective, these methods typically \textit{require training new models from scratch}, incurring high data and training costs and potentially weakening SAM's pre-trained generalization. More recent work compresses the original SAM via structured pruning and distillation~\cite{chen2024slimsam}, which better preserves pre-trained knowledge but introduces additional optimization complexity and limited flexibility across domains and hardware settings.

In contrast, StructSAM investigates \textit{SAM off-the-shelf acceleration} via inference-time token merging, \textit{without modifying model weights or requiring retraining}, offering a lightweight alternative to architectural compression. Additionally, while quantization methods~\cite{liu2024pq,lv2024ptq4sam,xiao2023smoothquant} have been explored to reduce SAM's memory footprint and bit-width requirements, our merging method is fundamentally orthogonal to these approaches; by combining both, we can achieve synergistic gains, further reducing FLOPs and accelerating inference on already quantized models (Fig.~\ref{fig:quantization}).

\vspace{-0.05in}
\paragraph{Token Pruning and Merging in Transformers.}
Existing methods for reducing transformer complexity mainly fall into dynamic token pruning and token merging. \textsc{Token pruning} has been explored in both NLP \cite{goyal2020power, zhong2023revisiting, ahmadpanah2025dynamic} and vision transformers \cite{yin2022vit,wang2023zero,eliopoulos2025pruning}, where tokens are dynamically removed based on learned importance. These methods typically require training and introduce input-dependent token counts, complicating batching and limiting practical speedups. \textsc{Token merging} methods, led by ToMe \cite{tome} and follow-up works \citep{DiffRate,shi2023crossget,kim2024token}, merge similar tokens using lightweight bipartite matching, achieving better efficiency-accuracy trade-offs. While attention- or guidance-based improvements exist, these approaches remain sensitive to token partitioning and distribution imbalance. More principled clustering or graph-based methods \citep{loukas2018spectrally,tran2024accelerating} provide stronger guarantees but introduce significant overhead, limiting their practicality for efficient ViT inference. Moreover, their progressive token reduction across layers is ill-suited to SAM’s dense segmentation.

\textbf{Token Reduction for Semantic Segmentation.}
A number of methods aim to reduce tokens in transformer-based segmentation. Early token halting methods stop processing high-confidence tokens, but may hinder information flow in deeper layers~\cite{tang2023dynamic,liu2024dynamic}. Token clustering approaches, such as ELViT~\cite{liang2022expediting} and AiluRus~\cite{li2023ailurus}, merge tokens within a layer but achieve limited speedup due to clustering overhead and one-shot reduction. Content-aware merging~\cite{lu2023content} introduces a policy network, increasing computational cost, while ALGM~\cite{norouzi2024algm} performs adaptive local-to-global merging based on similarity, improving efficiency but relying on feature-level heuristics.

In contrast, we design a \textit{structure-aware merging} strategy that leverages fast feature-gradient-based energy scores to preserve boundary-critical tokens while merging redundant regions, enabling more effective and principled token reduction with minimal overhead.

\vspace{-0.05in}
\section{Method} \label{sec:method}
\vspace{-0.05in}
\subsection{SAM architecture}
\vspace{-0.05in}
SAM~\cite{kirillov2023segment} employs a transformer-based image encoder for dense image representation learning.
Images are first embedded into a grid of visual tokens via a patch embedding layer and
processed by a hierarchical encoder that produces multi-scale features, which are then
consumed by a lightweight mask decoder.

\begin{wrapfigure}{r}{0.4\columnwidth}
    \centering
    \vspace{-0.12in}
    \includegraphics[width=0.4\columnwidth]{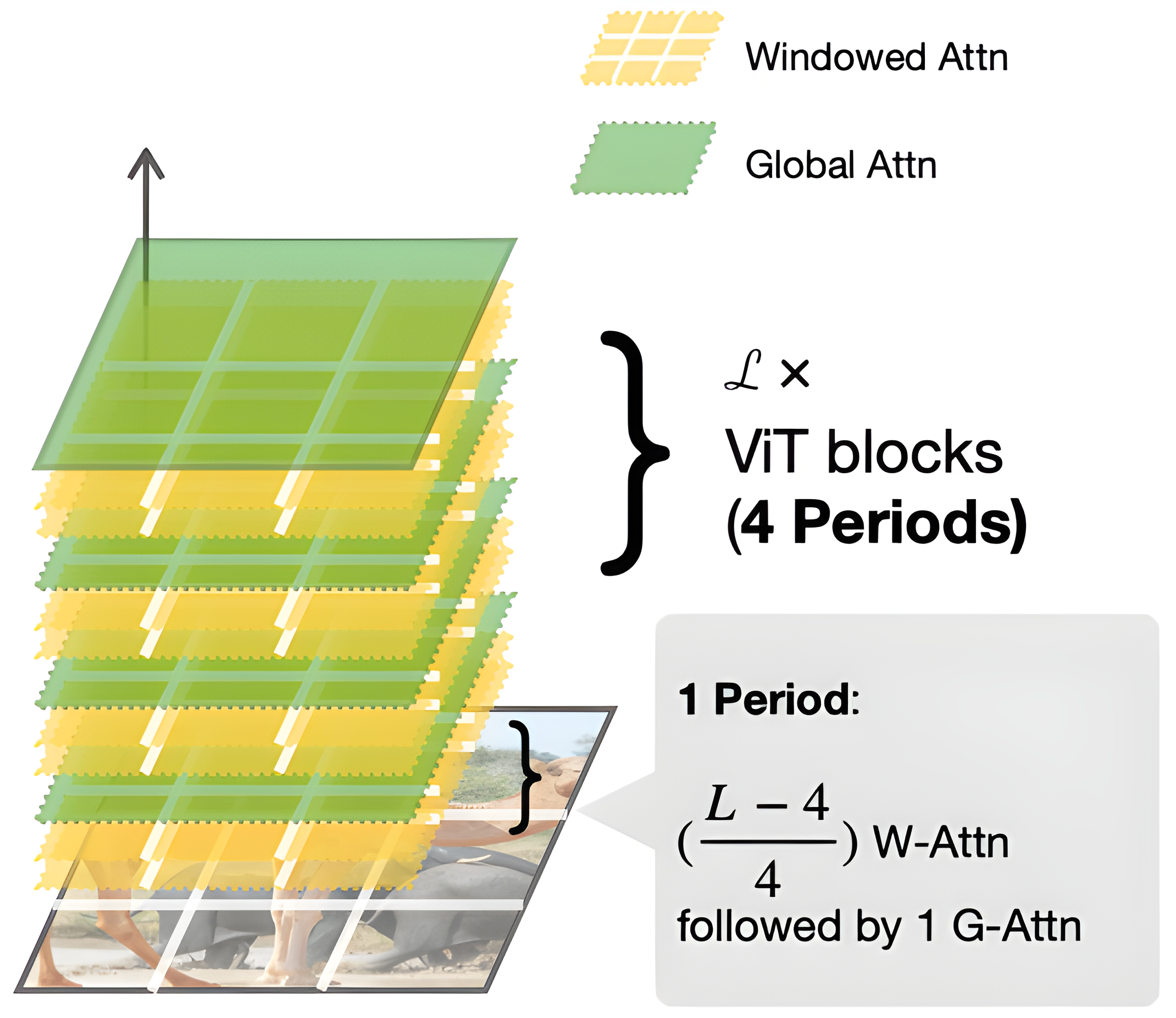}
    \caption{SAM’s ViT encoder mixes windowed and global attention, making standard token merging methods designed for discriminative ViTs difficult to apply directly.}
    \label{fig:sam_encoder}
\end{wrapfigure}

To balance efficiency and global context modeling, SAM interleaves
\textit{local} within a window size and \textit{global attention} in each encoder
(Figure~\ref{fig:sam_encoder}). Given an image token set
\( \mathcal{X} = \{x_1, \dots, x_N\} \), tokens are partitioned into disjoint spatial
windows \( \{\mathcal{P}_k\}_{k=1}^{K} \) such that
\setlength{\abovedisplayskip}{2pt}
\setlength{\belowdisplayskip}{2pt}
\setlength{\abovedisplayshortskip}{2pt}
\setlength{\belowdisplayshortskip}{2pt}
\begin{equation*}
    \bigcup_{k=1}^{K} \mathcal{P}_k = \mathcal{X}, \qquad
\mathcal{P}_i \cap \mathcal{P}_j = \emptyset \ \text{for } i \neq j .
\end{equation*}
At most layers, self-attention is computed independently within each window to capture
local spatial dependencies. For a token \( x_i \), local attention is defined as
\begin{equation*}
    \mathrm{Attn}_{\mathrm{local}}(x_i) =
\sum_{x_j \in \mathcal{P}(i)}
\mathrm{softmax}_j
\left(
\frac{\mathbf{q}_i^\top \mathbf{k}_j}{\sqrt{d}}
\right)
\mathbf{v}_j ,
\end{equation*}
where \( \mathbf{q}_i = \mathbf{q}(x_i) \),
\( \mathbf{k}_j = \mathbf{k}(x_j) \), and
\( \mathbf{v}_j = \mathbf{v}(x_j) \).

To enable long-range information exchange, SAM periodically applies
\textit{global attention} over locally updated tokens
\setlength{\abovedisplayskip}{2pt}
\setlength{\belowdisplayskip}{2pt}
\setlength{\abovedisplayshortskip}{2pt}
\setlength{\belowdisplayshortskip}{2pt}
\( \tilde{x}_i = \mathrm{Attn}_{\mathrm{local}}(x_i) \):
\begin{equation*}
    \mathrm{Attn}_{\mathrm{global}}(\tilde{x}_i) =
\sum_{j=1}^{N}
\mathrm{softmax}_j
\left(
\frac{\tilde{\mathbf{q}}_i^\top \tilde{\mathbf{k}}_j}{\sqrt{d}}
\right)
\tilde{\mathbf{v}}_j .
\end{equation*}
This design captures global context while amortizing the cost of full self-attention across layers.

\vspace{-0.1in}
\paragraph{Why SAM needs explicit token recovery.}
Unlike standard ViTs that can progressively reduce tokens, SAM requires a \emph{dense} feature grid for its mask decoder and relies on a consistent 2D token layout. Therefore, we adopt a \emph{merge--compute--unmerge} scheme within each block, reducing attention cost while restoring full-resolution tokens for subsequent layers and decoding.

\vspace{-0.05in}
\subsection{Resolution-preserving merge--unmerge interface}
To reduce SAM’s self-attention cost, we introduce a \textbf{token merging framework} tailored to its mixed local–global attention. At an encoder layer \(\ell\), let the full-resolution token set be
\( \mathcal{X}^{(\ell)} = \{x_1^{(\ell)}, \dots, x_{N_\ell}^{(\ell)}\} \).
We define a merging operator:
\[
f_\ell: \mathcal{X}^{(\ell)} \rightarrow \widetilde{\mathcal{X}}^{(\ell)}, \qquad
|\widetilde{\mathcal{X}}^{(\ell)}| < |\mathcal{X}^{(\ell)}|,
\]
which aggregates spatially redundant tokens while preserving semantically relevant information.
Self-attention is computed on the merged tokens and then lifted back to the original resolution by an explicit unmerging operator \(f_\ell^{-1}\):
\begin{equation*}
\widetilde{\mathcal{X}}^{(\ell+1)}
=\mathrm{Attn}\!\left(\widetilde{\mathcal{X}}^{(\ell)}\right),
\qquad
\mathcal{X}^{(\ell+1)}
=f_\ell^{-1}\!\left(\widetilde{\mathcal{X}}^{(\ell+1)}\right),
\end{equation*}
where \(\mathrm{Attn}\)(.) denotes either local or global attention depending on the layer.
This design enables efficient attention computation while preserving \emph{the segmentation accuracy of SAM}.

\paragraph{Plug-in baselines.}
Under this merge--unmerge interface, existing token merging strategies (e.g., \textsc{ToMe}~\cite{tome}, \textsc{Pi-ToMe}~\cite{tran2024accelerating}, \textsc{ToMe-SD}~\cite{bolya2023tomesd}, \textsc{VidToMe}~\cite{li2024vidtome}, \textsc{ALGM}~\cite{norouzi2024algm}, etc) correspond to different choices of how \(f_\ell\) selects merge candidates and destinations.
We compare these methods using the same interface in the Experiments section, validating across different SAM encoder sizes and datasets.

\vspace{-0.05in}
\subsection{StructSAM: gradient-guided structure-aware token merging}
\vspace{-0.05in}
We overview our method in Figure~\ref{fig:fullpage}.
At each transformer layer \( \ell \) of the SAM image encoder, we consider a set of image tokens
$    \mathcal{X}^{(\ell)} = \{x_1^{(\ell)}, x_2^{(\ell)}, \dots, x_{N_\ell}^{(\ell)}\},$
which can be reshaped into a tensor feature map
$
\mathbf{I}^{(\ell)} \in \mathbb{R}^{H_\ell \times W_\ell \times C_\ell},\ \mathrm{ where  }\ |\mathcal{X}^{(\ell)}| = H_\ell \times W_\ell.
$
Each token \( x_i^{(\ell)} \) is associated with a spatial position
\(
p(i) = (h(i), w(i)),
\)
corresponding to its height and width indices on the token grid.
These spatial indices match those used in SAM’s decomposed relative positional embedding, providing a consistent notion of spatial locality across attention and feature processing.
Our goal is to reduce the number of tokens participating in self-attention at layer
\( \ell \) by merging spatially redundant tokens, while preserving tokens that are
critical for object boundaries and prompt-conditioned segmentation. This structure-aware selection enables targeted merging in redundant regions while preserving important boundaries.
\vspace{-0.05in}
\paragraph{Feature gradient--based energy estimation.}
We interpret token features at layer \( \ell \) as a discrete feature field
\[
\mathbf{I}^{(\ell)} : (h,w) \mapsto \mathbf{f}^{(\ell)}_{h,w} \in \mathbb{R}^{C_\ell},
\]
and approximate local feature gradients via finite differences (or Sobel operators~\cite{gonzalez2009digital}) over adjacent tokens:
\begin{equation*}
    \nabla_x \mathbf{I}^{(\ell)}(h,w)
 \approx
\mathbf{f}^{(\ell)}_{h,\,w+1}
-
\mathbf{f}^{(\ell)}_{h,\,w-1};\ 
\nabla_y \mathbf{I}^{(\ell)}(h,w)
 \approx
\mathbf{f}^{(\ell)}_{h+1,\,w}
-
\mathbf{f}^{(\ell)}_{h-1,\,w}.
\end{equation*}

The gradient magnitude
\[
\mathbf{G}^{(\ell)}(h,w)
=
\sqrt{
\left\lVert \nabla_x \mathbf{I}^{(\ell)}(h,w) \right\rVert_2^2
+
\left\lVert \nabla_y \mathbf{I}^{(\ell)}(h,w) \right\rVert_2^2
}
\]
acts as a lightweight \emph{energy} score (see Figure~\ref{fig:sam-all_models_miou_flops_avg_combined-right}), where larger values \emph{tend to reflect} stronger local feature variations (e.g., boundaries) and are thus more likely to be preserved during merging. As shown in Table~\ref{tab:cost_energy}, this first-order measure provides a competitive proxy to graph-based energy formulations~\cite{tran2024accelerating}, while incurring substantially lower computational overhead (saving 65-75\% FLOPs). Additional visualizations demonstrate that this energy consistently highlights semantically important tokens even in deeper layers (Figure~\ref{fig:pca_layers}, Appendix).
\vspace{-0.05in}
\paragraph{Cell partitioning aligned to attention windows.}
To enable controlled and spatially coherent merging, we partition the token grid into non-overlapping spatial cells of size \( s \times s \) (Fig.~\ref{fig:fullpage}, \textit{Cell Division}).
For local-attention layers, we apply this partition \emph{within each attention window} \(\mathcal{P}_k\) to prevent cross-window collapse; for global-attention layers we treat the entire grid as a single window.
Let
\[
\mathcal{C}^{(\ell)} =
\{\mathcal{C}_1^{(\ell)}, \dots, \mathcal{C}_M^{(\ell)}\},
\quad
\bigcup_{m=1}^{M} \mathcal{C}_m^{(\ell)} = \mathbf{I}^{(\ell)},
\quad
\mathcal{C}_i^{(\ell)} \cap \mathcal{C}_j^{(\ell)} = \emptyset.
\]
\vspace{-0.1in}
\paragraph{Cell flatness and protected set.}
We define a cell-level flatness score as
$
\phi(\mathcal{C}_m^{(\ell)})
=
- \max_{(h,w) \in \mathcal{C}_m^{(\ell)}} \mathbf{G}^{(\ell)}(h,w)
$, 
where higher values indicate smoother regions with weaker structural variation. Intuitively, cells with low gradient responses are more likely to lie in homogeneous areas, while those with strong responses tend to contain boundaries or semantically important structures.

Given a target \emph{merge rate} \(r\in[0,1)\) (fraction of tokens removed at layer \(\ell\)), we sort cells by \(\phi\) and select the first \(M_{\mathrm{merge}}\) cells as \emph{mergeable}, where \(M_{\mathrm{merge}}\) is chosen so that the post-merge token count matches \((1-r)H_\ell W_\ell\).
Since each mergeable \(s\times s\) cell reduces \(s^2\) tokens to one token, each selected cell removes \(s^2-1\) tokens, and we choose \(M_{\mathrm{merge}}\) accordingly.
All remaining cells form the \emph{protected set} (Figure \ref{fig:TokenMergingIllustration}) whose tokens are kept at full resolution. This structure-aware selection enables targeted merging in redundant regions while preserving important boundaries

\paragraph{Destination and source token selection.}
Within each mergeable cell \( \mathcal{C}_m^{(\ell)} \), we designate a single
\emph{destination token} as the most stable representative in that region:
\[
x^{(\ell)}_{m,\mathrm{dst}}
=
\arg\min_{x_i^{(\ell)} \in \mathcal{C}_m^{(\ell)}}
\mathbf{G}^{(\ell)}(h(i), w(i)).
\]
All remaining tokens in the same cell form the \emph{source set}
$
\mathcal{S}_m^{(\ell)}
=
\mathcal{C}_m^{(\ell)} \setminus \{x^{(\ell)}_{m,\mathrm{dst}}\}.
$

\paragraph{Cell-wise token merging.}
For each mergeable cell, we merge its source tokens into the destination token by running bipartite soft matching and then averaging:
$
\widetilde{\mathbf{f}}^{(\ell)}_{x^{(\ell)}_{m,\mathrm{dst}}}
=
(1/|\widetilde{\mathcal{C}_m}^{(\ell)}|)
\sum_{x_i^{(\ell)} \in \widetilde{\mathcal{C}_m}^{(\ell)}}
\mathbf{f}_i^{(\ell)},
$ where $\widetilde{\mathcal{C}_m}^{(\ell)}$ be a set of matched tokens for each ${x^{(\ell)}_{m,\mathrm{dst}}}$.
The merged token set \(\widetilde{\mathcal{X}}^{(\ell)}\) is formed by (i) all tokens in protected cells, and (ii) the destination tokens for mergeable cells.

\paragraph{Token recovery (unmerging).}
After computing self-attention on the merged token set, we restore token representations to their original spatial resolution through a lightweight \emph{token recovery} operation.
For each mergeable cell \(m\), the updated destination feature is duplicated back to all tokens in
\(\mathcal{S}_m^{(\ell)} \cup \{x^{(\ell)}_{m,\mathrm{dst}}\}\).
Tokens in protected cells keep their own updated features.
This unmerging step preserves the original token layout required for dense, prompt-conditioned mask prediction, while incurring negligible computational overhead.
We provide a pseudocode in Algorithm 1 in the appendix.

\paragraph{Prompt-aware variant.}
While SAM’s image encoder is prompt-independent, prompts specify \emph{which} structures the decoder will emphasize.
When box prompts are available, we map the box to token coordinates and apply a lower merge rate inside the prompted region and a higher merge rate outside, preserving prompt-relevant detail while enabling additional speedups in background areas.
\vspace{-0.05in}
\section{Graph Coarsening View and Spectral Stability in StructSAM}\label{sec_theory_structsam}
\vspace{-0.1in}
We analyze token merging in \emph{StructSAM} through the lens of \emph{spectral graph theory}.
At a high level, tokens within each attention window define a weighted graph, and StructSAM’s
merge--unmerge procedure induces a \emph{graph coarsening} followed by a canonical \emph{lifting}
that restores the original resolution required for dense mask prediction.
Concretely, mergeable cells act as coarse nodes (one destination per cell), while protected regions
remain at full resolution; the score used to select mergeable cells and destinations is the
feature-gradient energy from \cref{fig:fullpage} (see Sec.~\ref{sec:method}).

At encoder layer \( \ell \), tokens in each attention window \( \mathcal{P}_k \)
(spatial windows for local attention or \( \mathcal{P}_1=\mathcal{X} \) for global
attention) form a graph \( \mathcal{G}_{\ell,k} \) with normalized Laplacian
\( \mathcal{L}_{\ell,k} \).
After merging and lifting, we obtain a lifted graph
\( \mathcal{G}_{\ell,k,l} \) on the original token set.
We quantify structural distortion at layer \( \ell \) via the \emph{spectral discrepancy}
\(\sd_{\ell}
\;\triangleq\;
\sum_{k=1}^{K_\ell}
\bigl\|
\bslambda_{\ell,k}
-
\bslambda_{\ell,k,l}
\bigr\|_1,\)
where \( \bslambda_{\ell,k} \) and \( \bslambda_{\ell,k,l} \) denote the eigenvalues of the
original and lifted Laplacians (sorted order).

Our assumptions are intentionally aligned with empirical properties of SAM features.
In Appendix~\ref{appendix_proofs}, \cref{ass_A1,ass_A2} (within-region concentration and margin separation) are supported by the
PCA visualisations of encoder features in \cref{fig:pca_layers}, which show coherent,
region-consistent structure across layers under aggressive merging, and sparse difference maps.
StructSAM’s key additional requirement, \cref{ass_A3_gradsep}, is a \emph{checkable} gradient-separation
condition: boundary (mixed) cells exhibit large \(\max_{x\in\mathcal{C}}\mathbf{G}^{(\ell)}(x)\) while
interior cells stay low-gradient, so flatness screening excludes boundary cells with probability
\(1-\delta_{\ell,k,s}\), where \(\delta_{\ell,k,s}\to 0\) as the grid is refined
(see Appendix~\ref{appendix_proofs} and \cref{fig:fullpage,fig:pca_layers}). See Appendix~\ref{sec:a3_practicality_pca} for the practicality of these assumptions.

\begin{theorem}[Informal: Layerwise spectrum stability of score-guided merging]\label{thm_structsam_spectrum_stability}
Fix an encoder layer \( \ell \) with windows \( \{\mathcal{P}_k\}_{k=1}^{K_\ell} \).
Let \( \sd_{\ell}(\mathrm{SG}) \) denote the spectral discrepancy induced by StructSAM’s
score-guided merging, and \( \sd_{\ell}(\mathrm{Base}) \) that of a non--score-guided
baseline (e.g., random or stride-based destination selection as ToMe-SD~\cite{bolya2023tomesd}).
Then, under \cref{ass_A1,ass_A2,ass_A3_gradsep} in Appendix \ref{sec_ass}:
\begin{enumerate}
\item \textbf{(StructSAM: vanishing drift)} Flatness screening protects boundary cells and forces merges to occur within coherent regions; hence \( \mathbb{E}[\sd_{\ell}(\mathrm{SG})] \to 0 \).
\item \textbf{(Baselines: irreducible drift)} Non--score-guided dst selection can hit boundary-crossing cells with non-negligible probability (Appendix~\ref{sec_proof_thm_structsam_spectrum_stability_formal}); under \cref{ass_A2}, this induces cross-region merges and implies \( \liminf\,\mathbb{E}[\sd_{\ell}(\mathrm{Base})] > 0 \).
\end{enumerate}
\end{theorem}

\cref{thm_structsam_spectrum_stability}, proved in Appendix~\ref{appendix_proofs}, explains why StructSAM
preserves segmentation quality at high merge rates, while heuristic destination selection can degrade
dense prediction performance.
\vspace{-0.1in}
\section{Experiments}

\vspace{-0.1in}
\textbf{Q1. Boundary and Thin Structure Preservation.}
We evaluate StructSAM’s ability to preserve boundary precision and fine structural details under token reduction. To this end, we conduct experiments on four boundary-sensitive datasets: (i) \textsc{DIS5K}~\cite{qin2022highly}, which provides high-resolution images with pixel-accurate boundary annotations; (ii) \textsc{ThinObject-5K} and \textsc{COIFT}~\cite{liew2021deep}, which focus on globally thin structures (e.g., wires, poles) and objects with fine interior details; (iii) \textsc{HRSOD}~\cite{zeng2019towards}, which emphasizes precise boundary delineation in complex high-resolution scenes.
We validate StructSAM across two SAM variants, ViT-B and ViT-L, to verify robustness across architectures.

\textbf{Token Merging Comparison.}  We compare against five representative token merging baselines. \textsc{ToMe}~\cite{tome} uses bipartite soft matching to merge similar tokens, while ToMeSD~\cite{bolya2023tomesd} extends this with timestep-aware merging for diffusion models in image generation. PiToMe~\cite{tran2024accelerating} improves selection via pivot-based strategies to preserve important tokens, and VidToMe~\cite{li2024vidtome} leverages temporal redundancy in video transformers. ALGM~\cite{norouzi2024algm} introduces adaptive local-to-global merging to the ViT architecture for dense segmentation. 

\textbf{Observations.} We evaluate performance using mIoU and boundary mIoU to measure segmentation quality, along with GFLOPs, memory, and inference speed (imgs/s) to assess efficiency. Table \ref{tab:dis5k_thinobject5k_coift_hrsod} and Figure~\ref{fig:sam-all_models_miou_flops_avg_combined} indicate that StructSAM achieves a favorable efficiency–accuracy trade-off compared to the baseline, reducing FLOPs by $\sim$28.5\% (ViT-B) and $\sim$21.9\% (ViT-L) while lowering memory usage, with largely comparable performance across most datasets. It maintains strong accuracy on boundary-sensitive benchmarks, with only minor degradation in some cases (e.g., COIFT), indicating effective preservation of structural details under token reduction.

Compared with existing token merging methods, StructSAM exhibits more stable performance across varying merge rates (35\%–65\%), often achieving higher or comparable accuracy, particularly on datasets requiring precise boundary delineation. Notably, even at high merge rates, performance degradation remains limited and can match or exceed the baseline on some datasets (e.g., HRSOD). While StructSAM \textit{does not always achieve the highest throughput}, it provides a more balanced trade-off by prioritizing segmentation quality and structural fidelity (Fig~\ref{fig:side_by_side-right}).

\begin{table}[H]
\footnotesize
\renewcommand{\arraystretch}{0.9}
\setlength{\tabcolsep}{3pt}
\centering
\vspace{-0.2in}
\caption{Performance of merging methods on DIS5K, ThinObject5K-TE, COIFT, and HRSOD datasets at the merging rate $r=$ 55\%. Note that \textsc{ToMe} and \textsc{PiToMe} can not work with $r > 50\%$ .}
\vspace{0.05in}
\label{tab:dis5k_thinobject5k_coift_hrsod}
\resizebox{0.95\textwidth}{!}{
\begin{tabular}{llcc
                cccc
                cccc
                cccc
                cccc}
\toprule
& & \multirow{2}{*}{\textbf{GFlops}} & \multirow{2}{*}{\textbf{Mem (GB)}}
& \multicolumn{2}{c}{\textbf{DIS5K}}
& \multicolumn{2}{c}{\textbf{ThinObject5K-TE}}
& \multicolumn{2}{c}{\textbf{COIFT}}
& \multicolumn{2}{c}{\textbf{HRSOD}} \\

\textbf{Model} & \textbf{Method} &  &
& mIoU & b-mIoU
& mIoU & b-mIoU
& mIoU & b-mIoU
& mIoU & b-mIoU \\
\midrule

\multirow{5}{*}{ViT-B}
& Baseline
& 486.4 & 3.53
& 55.30 & 46.97 
& 63.28 & 52.65 
& 89.14 & 82.54 
& 86.64 & 76.56 \\

& TomeSD
& 362.3 $_{\textcolor{gray}{\downarrow 25.5\%}}$ & \textbf{2.68}
& 51.39 & 43.16 
& 58.09 & 48.24 
& 87.29 & 80.39 
& 85.34 & 75.66 \\

& ViDTome
& 399.8 $_{\textcolor{gray}{\downarrow 17.8\%}}$ & 3.43
& 43.80 & 35.77 
& 49.74 & 37.21 
& 80.52 & 69.55 
& 75.89 & 62.20 \\

& ALGM
& 381.1 $_{\textcolor{gray}{\downarrow 21.6\%}}$ & 2.93
& 51.26 & 42.30 
& 50.24 & 40.97 
& 49.10 & 39.48 
& 47.33 & 37.04 \\

& \cellcolor{gray!20}StructSAM
& \cellcolor{gray!20}\textbf{347.8}$_{\textcolor{blue}{\downarrow 28.5\%}}$ & \cellcolor{gray!20}\textbf{2.68}
& \cellcolor{gray!20}\textbf{54.61} & \cellcolor{gray!20}\textbf{45.57} 
& \cellcolor{gray!20}\textbf{63.30} & \cellcolor{gray!20}\textbf{52.06} 
& \cellcolor{gray!20}\textbf{87.74} & \cellcolor{gray!20}\textbf{80.85} 
& \cellcolor{gray!20}\textbf{86.67} & \cellcolor{gray!20}\textbf{76.84} \\
\midrule

\multirow{5}{*}{ViT-L}
& Base Model
& 1493.8 & 5.97
& 62.27 & 53.94 
& 75.50 & 64.71 
& 92.65 & 87.40 
& 89.67 & 82.65 \\

& TomeSD
& 1188.2 $_{\textcolor{gray}{\downarrow 20.4\%}}$ & \textbf{4.84}
& 60.32 & 50.81 
& 73.68 & 61.98 
& 90.51 & \textbf{84.67} 
& \textbf{88.58} & \textbf{80.99} \\

& ViDTome
& 1274.8 $_{\textcolor{gray}{\downarrow 14.7\%}}$ & 5.83
& 39.44 & 30.36 
& 47.51 & 31.05 
& 61.73 & 46.16 
& 60.79 & 46.56 \\

& ALGM
& 1249.6 $_{\textcolor{gray}{\downarrow 16.3\%}}$ & 5.17
& 56.93 & 44.44 
& 54.42 & 40.88 
& 52.82 & 38.17 
& 49.93 & 33.64 \\

& \cellcolor{gray!20}StructSAM
& \cellcolor{gray!20}\textbf{1167.1}$_{\textcolor{blue}{\downarrow 21.9\%}}$ & \cellcolor{gray!20}\textbf{4.84}
& \cellcolor{gray!20}\textbf{61.01} & \cellcolor{gray!20}\textbf{51.36} 
& \cellcolor{gray!20}\textbf{75.80} & \cellcolor{gray!20}\textbf{63.81} 
& \cellcolor{gray!20}\textbf{90.73} & \cellcolor{gray!20}84.26 
& \cellcolor{gray!20}88.39 & \cellcolor{gray!20}80.46 \\
\bottomrule
\end{tabular}
}
\end{table}


\begin{figure}[htbp]
\centering
\vspace{-0.22in}
\begin{subfigure}{0.57\linewidth}
    \centering
    \vspace{-0.1in}
    \includegraphics[width=0.97\linewidth]{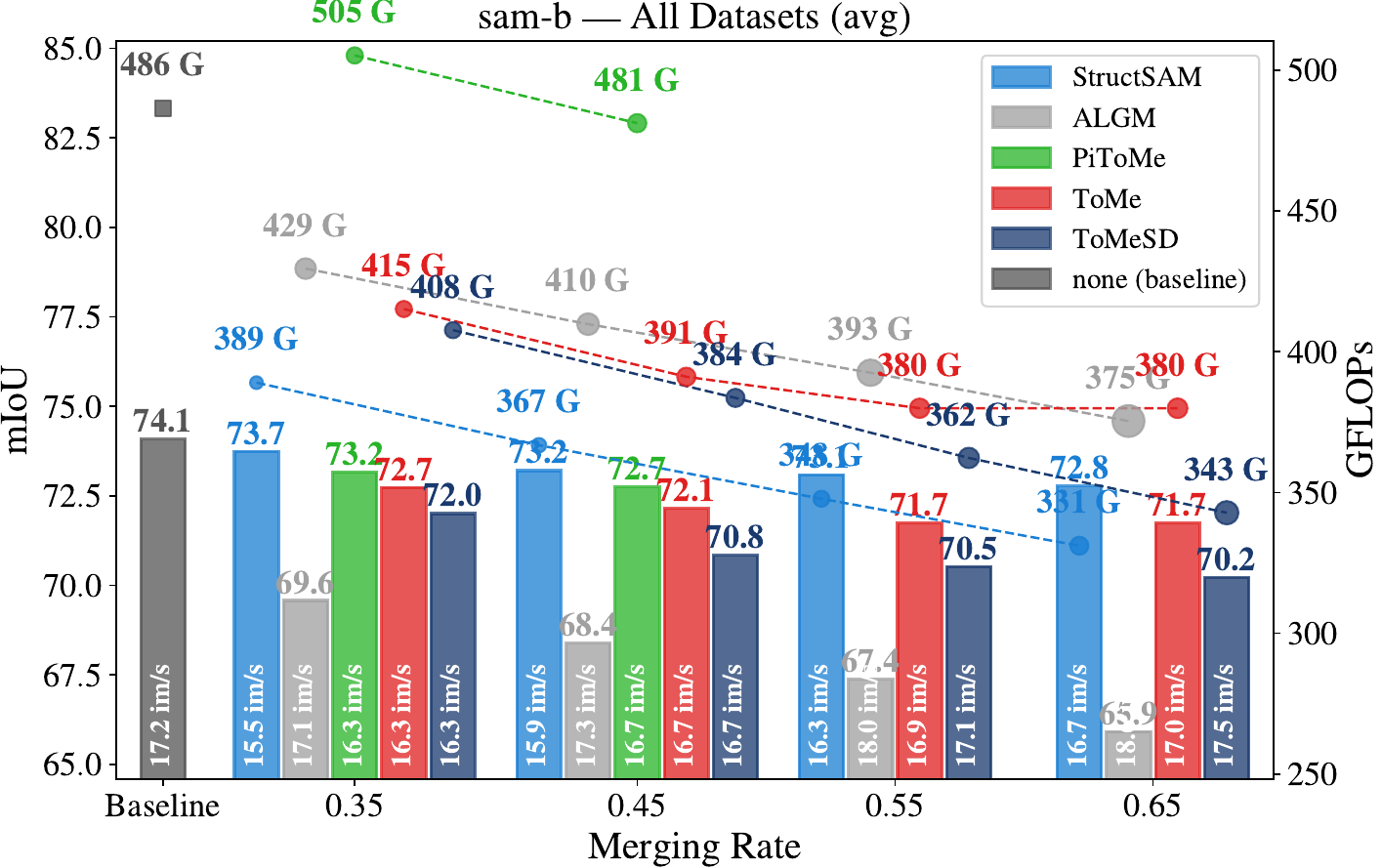}
    \vspace{-0.05in}
    \caption{}
    \label{fig:sam-all_models_miou_flops_avg_combined-left}
\end{subfigure}
\hfill
\begin{subfigure}{0.40\linewidth}
    \centering
    \includegraphics[width=\linewidth]{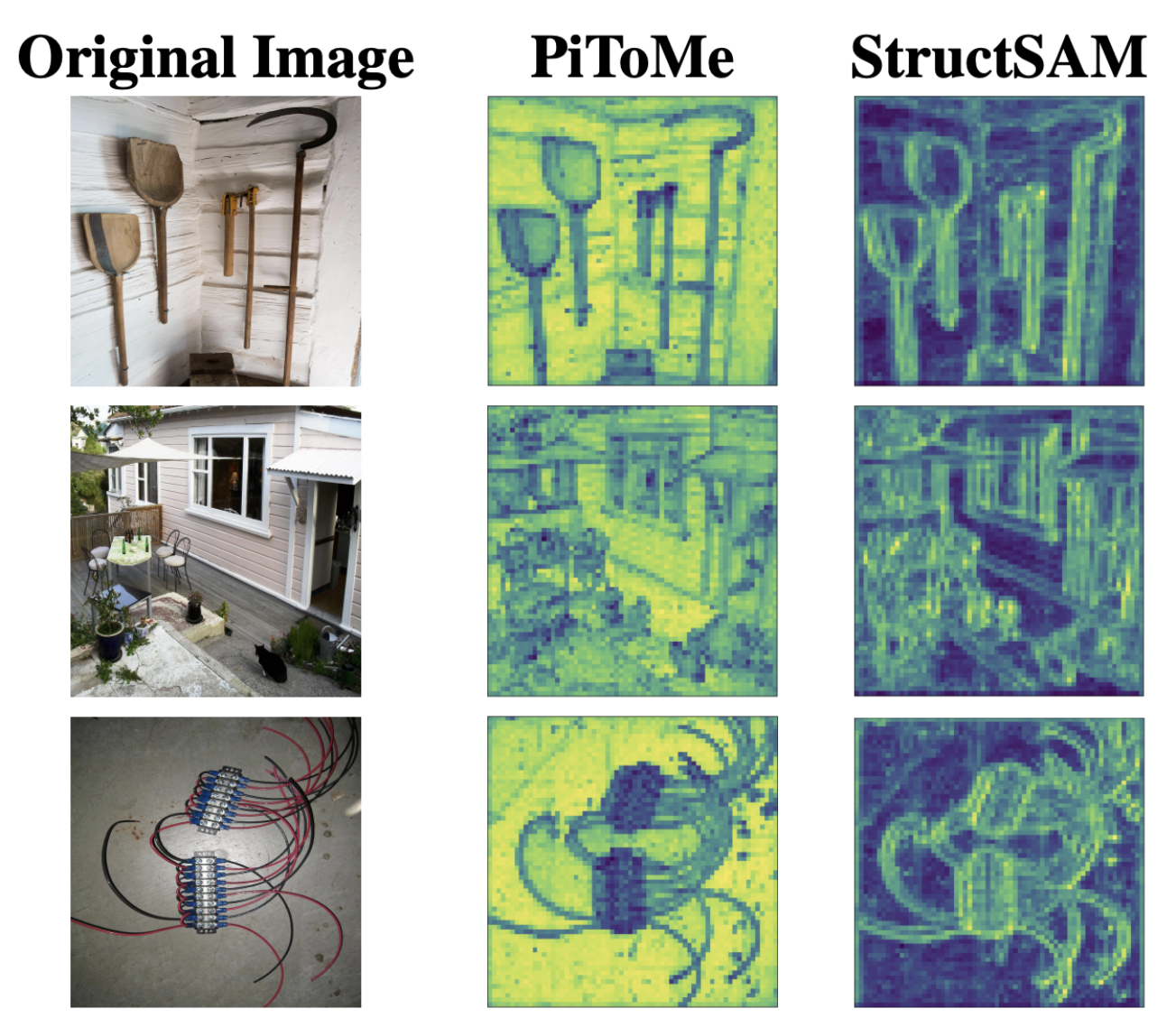}
    \vspace{-0.5cm}
    \caption{}
    \label{fig:sam-all_models_miou_flops_avg_combined-right}
\end{subfigure}
\vspace{-0.12in}
\caption{(a): Comparison across merging methods, showing mIoU, GFLOPs, and throughput (img/s) at \textbf{different merge rates} on SAM-B. Results for other architectures are in Appendix; (b): Qualitative comparisons of original images, energy scores computed as PiToMe~\cite{tran2024accelerating} and our.}
\vspace{-0.1in}
\label{fig:sam-all_models_miou_flops_avg_combined}
\end{figure}
\textbf{Q2. Generalization to SAM Variants.}
We further evaluate the generalization of StructSAM across different SAM-family models, including (i) \textsc{Medical SAM}~\cite{ma2024segment} and (ii) \textsc{Efficient-SAM}~\cite{xiong2024efficientsam}. For \textbf{Medical SAM}, we assess performance on the \textsc{INbreast}~\cite{moreira2012inbreast}, a full-field digital mammography repository containing 410 high-resolution DICOM mammograms with polygonal segmentation annotations under both \textit{prompt-based and prompt-free settings} to examine robustness in medical imaging scenarios requiring high precision. 

For \textbf{Efficient-SAM}, we integrate StructSAM into its lightweight architecture and extend it from static segmentation to \textit{video tracking} (Figure~\ref{fig:tracking}), enabling consistent object segmentation across frames. This setting is crucial for robotic manipulation, where both efficiency and temporal stability are required. We then incorporate tracking outputs into the vision-language-action model (VLA) Open-VLA~\cite{kim2024openvla} and evaluate it on real-robot object arrangement tasks. Results are reported as task success rates. Our goal is to demonstrate that, in some applications, such as relatively structured or less cluttered scenes, a StructSAM-enhanced Efficient-SAM tracker can serve as a more efficient alternative to heavier models like SAM-2~\cite{ravi2024sam}, while maintaining comparable performance. Details about the extension and experimental setups are presented in the Appendix.

\textbf{Observations.} On MedSAM, Fig.~\ref{fig:side_by_side-left} shows that the StructSAM lowers GFLOPs by up to 28.5\% (and 41.8\% with prompt-aware merging) with small marginal changes in Dice score, and maintains robust performance as the merge rate increases, consistently outperforming other token merging methods. When integrated into Efficient-SAM for tracking within VLA pipelines, it achieves comparable task success rates to SAM-2 while providing a ~45\% speedup, demonstrating effective trade-offs between efficiency and performance in both medical imaging and robotic applications.

\begin{figure}[!hbt]
    \centering
    \begin{subfigure}{0.48\linewidth}
        \centering
        \begin{minipage}{\linewidth}
            \centering
            \footnotesize
            \label{tab:medsam}
            \resizebox{1.0\textwidth}{!}{%
            \begin{tabular}{lcc}
            \toprule
            \textbf{Method} & \textbf{GFLOPs} & \textbf{Dice Score} \\
            \midrule
            \textbf{Base Model} & 486.4 & 75.43 \\
            \hline
            {TomeSD}   & 362.3$_{\textcolor{gray}{\downarrow 25.5\%}}$ & 73.33 \\
            {ViDTome}  & 399.8$_{\textcolor{gray}{\downarrow 17.8\%}}$ & 73.32 \\
            {ALGM}     & 381.1$_{\textcolor{gray}{\downarrow 21.6\%}}$ & 69.83 \\
            \cellcolor{gray!20}StructSAM & \cellcolor{gray!20}347.8$_{\textcolor{blue}{\downarrow 28.5\%}}$ & \cellcolor{gray!20}\textbf{74.81}\\
            \cellcolor{gray!20}StructSAM + \textbf{prompt-aware} & \cellcolor{gray!20}\textbf{283.0}$_{\textcolor{blue}{\downarrow 41.8\%}}$ & \cellcolor{gray!20}74.72\\
            \bottomrule
            \end{tabular}
            }
        \end{minipage}
        \vspace{0.5em} 

        \begin{minipage}{\linewidth}
            \centering
            \includegraphics[width=\linewidth]{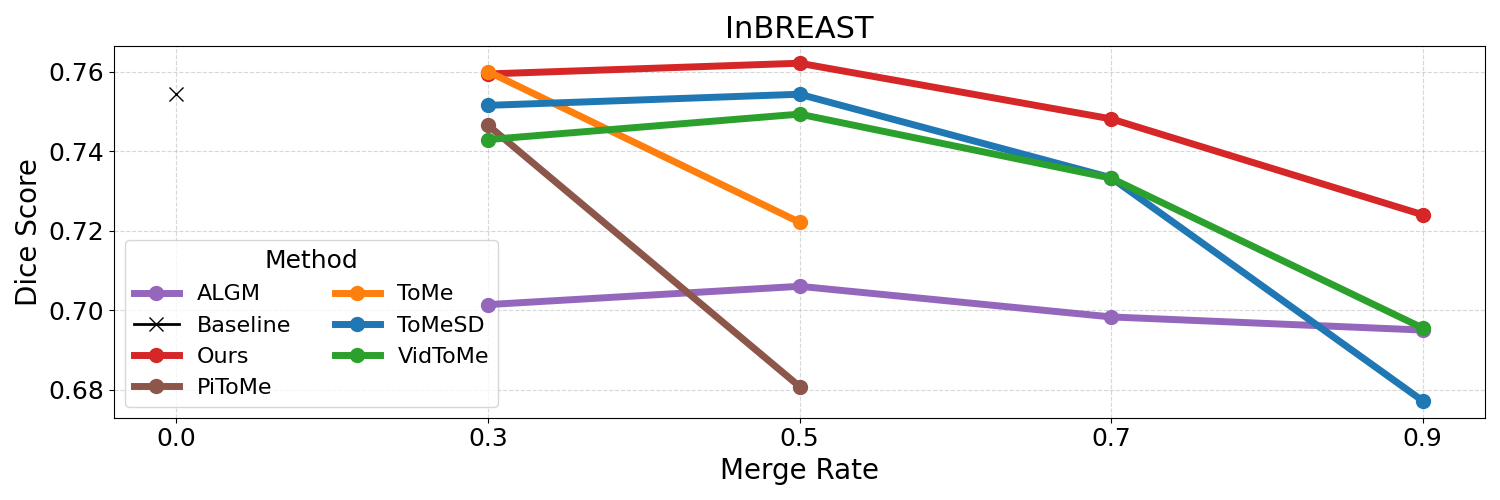}
        \end{minipage}
        \vspace{-0.05in}
        \caption{}
        \label{fig:side_by_side-left}
    \end{subfigure}
    \hfill
    \begin{subfigure}{0.48\linewidth}
        \centering
        \includegraphics[width=\linewidth]{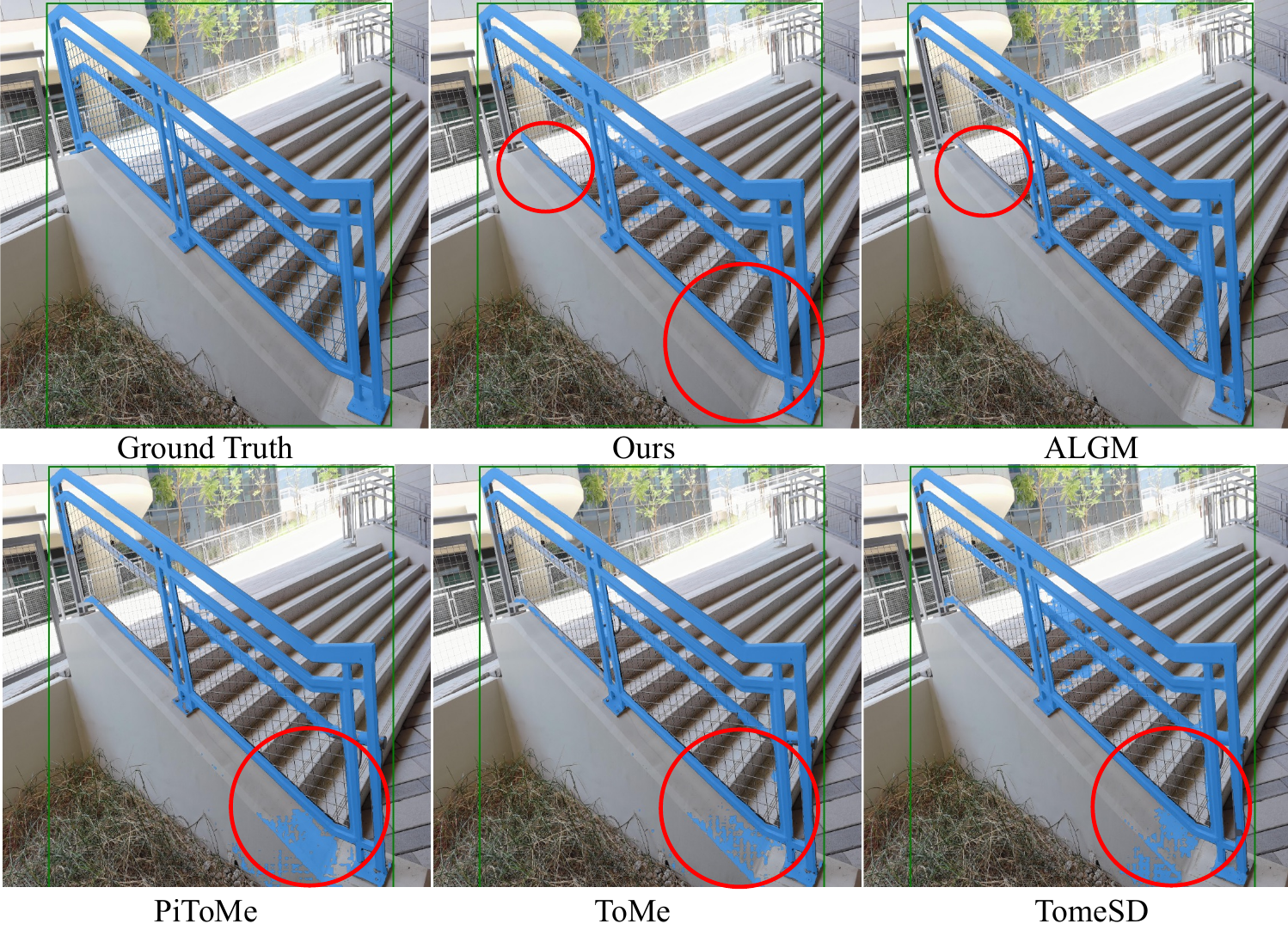}
        \vspace{-0.05in}
        \caption{}
        \label{fig:side_by_side-right}
    \end{subfigure}
    \vspace{-0.1in}
    \caption{(a) \textit{Top}: results on INbreast with \textbf{MedSAM}, including a prompt-aware StructSAM variant that restricts token processing for targeted efficiency. \textit{Bottom}: performance across varying merge rates.
(b) Qualitative comparison showing that StructSAM better preserves fine structures and detailed regions, while other methods often miss boundaries or over-merge objects into the background.}
    \label{fig:side_by_side}
\end{figure}

\begin{figure}[!htb]
    \begin{minipage}{0.58\textwidth}
    \centering
    \vspace{-0.25in}
    \includegraphics[width=1.0\linewidth]{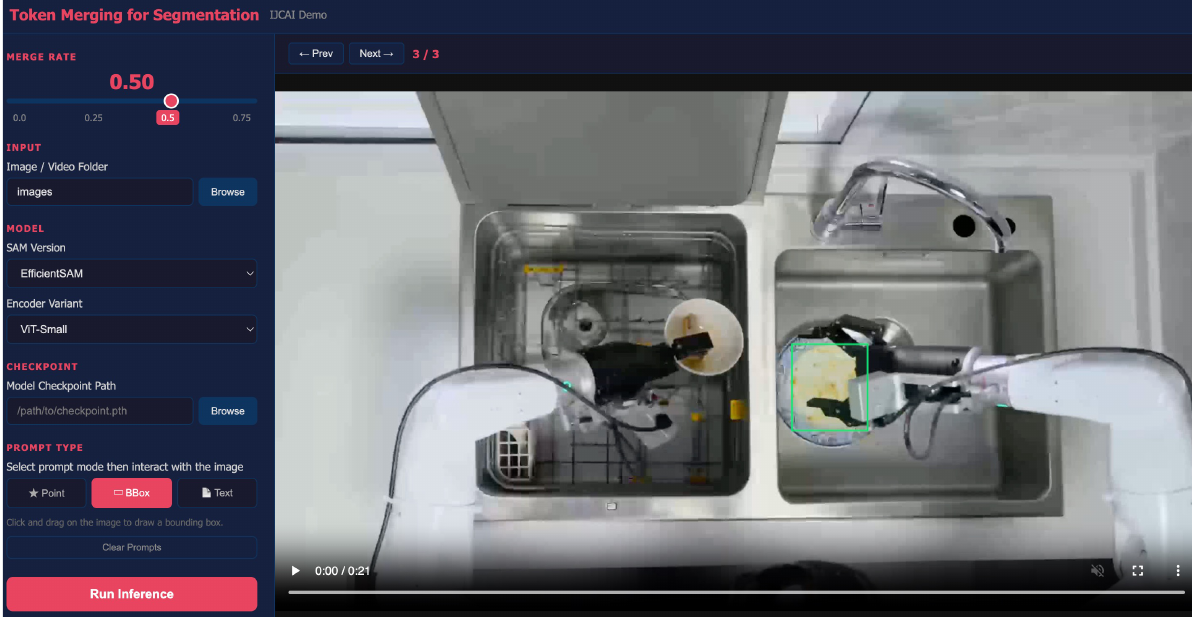}
    \captionof{figure}{Interface of \textbf{Efficient-SAM} integrated with StructSAM for tracking tasks. The left panel allows selection of SAM variants and prompt types, while the right panel displays predicted masks and bounding boxes.}
    \label{fig:tracking}
\end{minipage}
\hfill
\begin{minipage}{0.4\textwidth}
\centering
\vspace{-0.1in}
\includegraphics[width=1.0\linewidth]{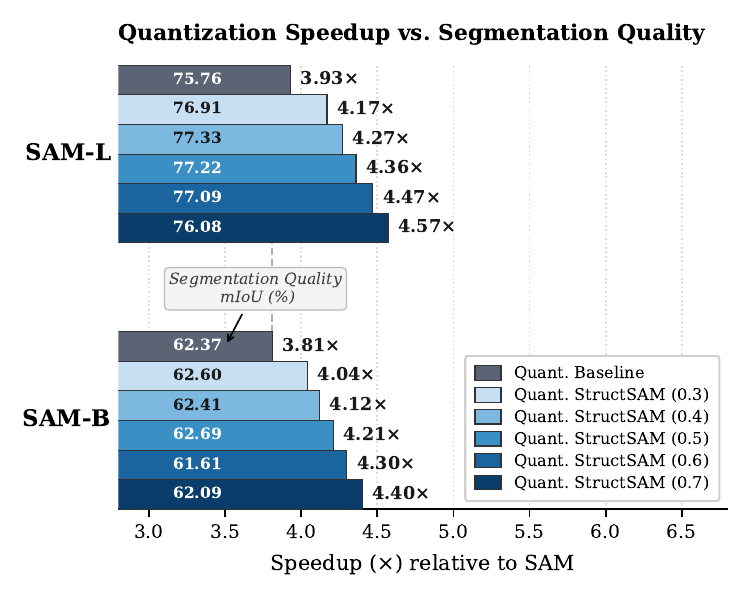}
\captionof{figure}{\small{Speedup comparison of quantized baseline and StructSAM (relative to the unquantized baseline) on ThinObject5K across merge rates (30–70\%). Numbers indicate mIoU.}}
\label{fig:quantization}

\end{minipage}
\end{figure}
\begin{figure}[!htb]
    \centering
    \vspace{-0.3in}
    \begin{subfigure}{0.58\linewidth}
        \centering
        \includegraphics[width=1.0\linewidth]{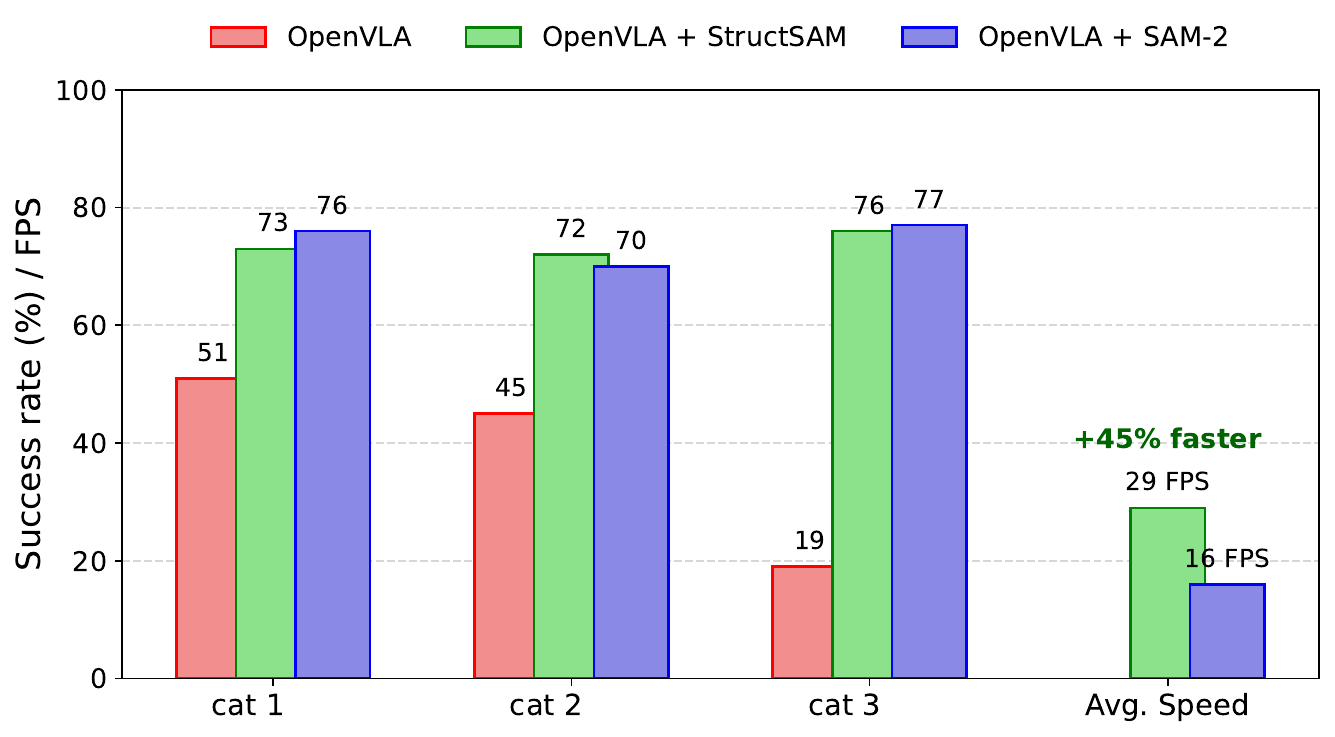}
        \caption{}
    \end{subfigure}
    \hfill
    \begin{subfigure}{0.38\linewidth}
        \centering
        \includegraphics[width=0.8\linewidth]{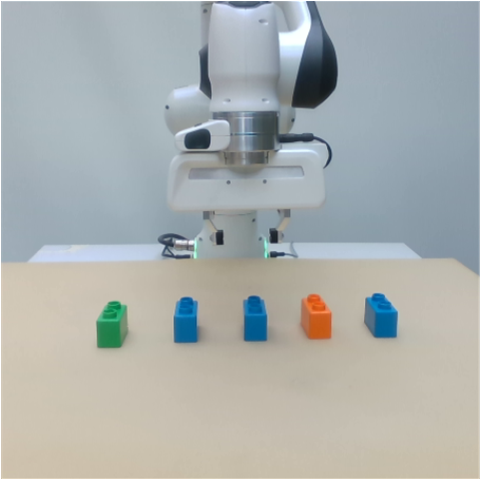}
        \caption{}
    \end{subfigure}
    \vspace{-0.08in}
    \caption{(a): \small{Comparison of success rates across objects and average tracking speed. StructSAM outperforms OpenVLA and remains competitive with OpenVLA + SAM-2, achieving a 45\% speedup. Success rates are reported over 30 trials. (b): Real-robot pick-and-place setup evaluated under 3 categories of difficulty, where the robot lifts the correct object based on a human prompt, and places it in the designated location.}}
    \vspace{-0.2in}
    \label{fig:two_figures}
\end{figure}

\vspace{-0.1in}
\textbf{Q3. Compatibility with Quantization. }
We investigate whether StructSAM is complementary to model quantization and can be applied on top of quantized SAM variants. Among existing quantization approaches for SAM~\cite{liu2024pq,lv2024ptq4sam,xiao2023smoothquant}, we adopt \emph{SmoothQuant}~\cite{xiao2023smoothquant} due to its strong compatibility with GPU kernels and efficient post-training deployment. 
Our results in Figure \ref{fig:quantization} show that StructSAM can be seamlessly integrated with a quantized model, providing additional gains in inference speed and memory efficiency while maintaining accuracy, demonstrating its complementary and orthogonal nature to quantization methods.

\textbf{Q4. Analysis Other Choices.}
Table~\ref{tab:ablation_results_sam_b} reports an ablation study on SAM-B evaluated on \textsc{DIS5K}, using both \textit{mIoU} and \textit{boundary IoU (B-IoU)} under two mask ratios. 
Across all settings, our \textbf{full method} achieves the best or tied-best performance, demonstrating the effectiveness of the proposed design. The ablation for \textbf{COIFT} is presented in the Appendix.


\textbf{- Effect of gradient estimation.}
Replacing the Sobel operator with a simple central difference (\textit{Central-Diff}) in the token energy score degrades performance, especially in B-IoU (Table~\ref{tab:ablation_results_sam_b}), highlighting the importance of accurate gradient estimation for preserving boundaries during merging. As shown in Figure~\ref{fig:sam-all_models_miou_flops_avg_combined-right}, our energy score effectively captures fine structures while remaining computationally efficient (Table~\ref{tab:cost_energy}). Additional discussion of failure cases is included in the Appendix.

\textbf{- Other results}. We provide in the Appendix evaluations with (i) point-based prompts instead of box prompts. We also (ii) visualize the heatmap difference between token compression across layers, comparing merged and original tokens to better understand the model’s behavior.

\begin{figure}[!htb]
\centering

\begin{minipage}{0.48\textwidth}
\centering
\vspace{-0.1in}
\captionof{table}{\small{Ablation study on SAM-B.
StructSAM (Full) denotes the complete model.}}
\vspace{-0.1in}
\label{tab:ablation_results_sam_b}

\setlength{\tabcolsep}{4pt}
\resizebox{\textwidth}{!}{%
\begin{tabular}{c|l|cc|cc}
\toprule
\multirow{2}{*}{Dataset} & \multirow{2}{*}{Method}
& \multicolumn{2}{c|}{$r=0.35$} & \multicolumn{2}{c}{$r=0.55$} \\
& & mIoU & B-IoU & mIoU & B-IoU \\
\midrule
\multirow{8}{*}{DIS5K} & StructSAM (Full) & \textbf{54.7} & \textbf{46.0 }& \textbf{54.6} & \textbf{45.6} \\
 & Central-Diff & 53.8 & 45.6 & 53.3 & 44.6 \\
 & Mean-Flatness & 54.5 & 45.8 & 54.4 & 45.4 \\
 & No-Cell & 53.9 & 44.5 & 54.2 & 43.9 \\
 & Rand-Cell & 53.8 & 43.9 & 53.5 & 43.2 \\
 & Max-Dst & 54.5 & 45.9 & 53.9 & 45.2 \\
 & Rand-Dst & \textbf{54.7} & \textbf{46.0} & 54.3 & 45.4 \\
\bottomrule
\end{tabular}}
\vspace{-0.1in}
\end{minipage}
\hfill
\begin{minipage}{0.48\textwidth}
\centering
\vspace{-0.05in}
\captionof{table}{\small{FLOP counts analysis between StructSAM energy scoring and graph-based methods such as PiToMe~\cite{tran2024accelerating}.}}
\label{tab:cost_energy}

\resizebox{\textwidth}{!}{%
\begin{tabular}{lll}
\hline
\textbf{Attention}   & \textbf{Method} &  \makecell{\textbf{($10^9$)} \\ \textbf{FLOPs/im}↓} \\
\hline
\multirow{3}{*}{Global} 
  & PiToMe & 1.0737 \\
  & \textbf{StructSAM}\,(Central Diff) & \makecell{0.2684 \decp{75.00}}\\
  & \textbf{StructSAM}\,(Sobel) & \makecell{0.2732 \decp{74.56}}\\
\hline

\multirow{3}{*}{Window} 
  & PiToMe & 0.0615 \\
  & \textbf{StructSAM}\,(Central Diff) & \makecell{0.0154 \decp{74.96}}\\
  & \textbf{StructSAM}\,(Sobel) & \makecell{0.0210 \decp{65.85}}\\
\hline
\end{tabular}}
\vspace{-0.1in}
\end{minipage}

\end{figure}
\vspace{-0.1in}
\section{Discussion and Limitations}
\vspace{-0.1in}
We show that token merging methods designed for classification ViTs often struggle to preserve structure when applied to SAM in dense, prompt-conditioned settings. To address this, we introduce a lightweight gradient-based energy that identifies boundary-critical tokens and enables efficient, structure-aware merging. This design generalizes well across applications, from natural and medical image segmentation to tracking in robotic manipulation tasks.
Our method relies on gradient-based cues derived from model features, making its effectiveness dependent on representation quality. When features are weak or noisy, e.g., in low-texture or highly cluttered scenes, the resulting gradients may be less informative, reducing the gains from structure-aware merging. In such cases, jointly retraining the backbone with token merging could improve robustness and better align features with the merging strategy. Future work can include extending StructSAM to 3D models (e.g., SAM-3~\cite{chen2025sam}) and optimizing ranking/sorting to improve GPU parallelism~\cite{liu2026gpu}.

\clearpage

\bibliographystyle{plain}
\bibliography{example_paper}



\newpage
\appendix
\begin{center}
\textsc{\Large Supplementary Material for \\ \vspace{.2em} ``StructSAM: Structure- and Spectrum-Preserving Token Merging for Segment Anything Models''}
\end{center}
\addcontentsline{toc}{section}{Supplementary Material}
{
  \hypersetup{linkcolor=black}
  \tableofcontents
  \addtocontents{toc}{\protect\setcounter{tocdepth}{2}}
}
\vspace{-0.1in}
\captionsetup{font=normalsize}
\newpage

\section{Additional Discussion on StructSAM Limitations}
Despite these advantages, StructSAM has several limitations.
First, our current formulation relies on fixed, hand-crafted gradient operators (e.g., Sobel or central differences), which may be suboptimal for highly textured regions or domain-specific imagery.
Second, the use of predefined cell partitions introduces an additional design choice that may require tuning for different input resolutions or architectures.
Finally, while our method is evaluated primarily in the context of SAM-based segmentation, its effectiveness for other vision tasks or transformer architectures remains to be fully explored.

These limitations suggest several promising directions for future work.
An interesting extension would be to learn adaptive or task-specific gradient operators that retain the efficiency of first-order information while improving robustness.
In addition, dynamically adjusting cell structures or sampling strategies based on content or model depth could further enhance flexibility.
More broadly, we believe that StructSAM opens up new opportunities for rethinking token merging as a \emph{local, structure-driven} problem rather than a global graph optimization task, and we hope this perspective will inspire more efficient designs for scalable vision transformers.

\section{Implementation Details}

\subsection{Pseudo Code}

We provide pseudocode for StructSAM's token partitioning procedure in Algorithm~\ref{alg:algorithm_StructSAM}.

\begin{algorithm}[H]
\caption{StructSAM token partitioning and merge map (cell-wise)}
\label{alg:algorithm_StructSAM}
\begin{algorithmic}[1]
\Require Token tensor $T \in \mathbb{R}^{H \times W \times C}$; cell size $s_x, s_y$; merge rate $r\in[0,1)$
\Ensure Source tokens $\mathbf{A}$ (to be merged); kept tokens $\mathbf{B}$; assignment map $\pi$ for unmerging

\State Initialize empty sets $\mathbf{A}$, $\mathbf{B}$ and empty map $\pi$
\State Partition $T$ into non-overlapping cells $\{ C_i \in \mathbb{R}^{s_x \times s_y \times C} \}$

\For{each cell $C_i$}
    \State Compute gradient magnitudes $G(t)$ for tokens $t \in C_i$
    \State Compute cell flatness $\phi_i \gets - \max_{t\in C_i} G(t)$
\EndFor
\State Sort cells by $\phi_i$ in decreasing order \Comment{Higher $\phi_i$ = flatter (more mergeable)}

\State $M_{\mathrm{merge}} \gets \left\lceil \dfrac{r\cdot H\cdot W}{s_x s_y - 1} \right\rceil$ \Comment{\#cells needed to remove $rHW$ tokens}
\State $\mathcal{M} \gets$ first $M_{\mathrm{merge}}$ cells in the sorted list \Comment{mergeable cells}
\State $\mathcal{P} \gets$ remaining cells \Comment{protected cells}

\For{each protected cell $C_i \in \mathcal{P}$}
    \State $\mathbf{B} \gets \mathbf{B} \cup C_i$ \Comment{keep all tokens}
\EndFor

\For{each mergeable cell $C_i \in \mathcal{M}$}
    \State Compute $G(t)$ for tokens $t \in C_i$
    \State $t_{\mathrm{dst}} \gets \arg\min_{t\in C_i} G(t)$ \Comment{stable destination token}
    \State $\mathbf{B} \gets \mathbf{B} \cup \{t_{\mathrm{dst}}\}$
    \For{each token $t \in C_i \setminus \{t_{\mathrm{dst}}\}$}
        \State $\mathbf{A} \gets \mathbf{A} \cup \{t\}$
        \State $\pi(t) \gets t_{\mathrm{dst}}$ \Comment{unmerge target}
    \EndFor
\EndFor

\State \Return $\mathbf{A}, \mathbf{B}, \pi$
\end{algorithmic}
\end{algorithm}

\paragraph{Bounding Box Generation}
Following the standard SAM evaluation protocol, we use bounding box prompts derived from ground truth segmentation masks. For each ground truth mask, we compute the tight axis-aligned bounding box by extracting the minimum and maximum coordinates of foreground pixels (threshold > 128). The resulting boxes are provided to SAM in the format of top-left and bottom-right corners as spatial prompts. This deterministic box generation ensures reproducible evaluation while simulating realistic user-provided region annotations. For our prompt-aware token merging strategy, pixel-space boxes are converted to token-space coordinates by dividing by the patch size (16 pixels), enabling differentiated merging policies for tokens inside versus outside the prompted region.

\section{PCA Visualization}

We visualize the feature space using PCA projections across different layers of SAM-B in Fig.~\ref{fig:pca_layers} and Fig.~\ref{fig:pca_layers_parrot}. Even with 65\% of the tokens merged, the features remain faithful and preserve fine-grained forground object details.

\begin{figure}[htbp]
    \centering
    \setlength{\tabcolsep}{2pt}
    \begin{tabular}{c ccc}
        & \textbf{Feat. PCA (0\% merge)} & \textbf{Feat. PCA (65\%)} & \textbf{Difference} \\[0.5em]
        \rotatebox{90}{\textbf{Layer 0}} &
        \includegraphics[width=0.3\textwidth]{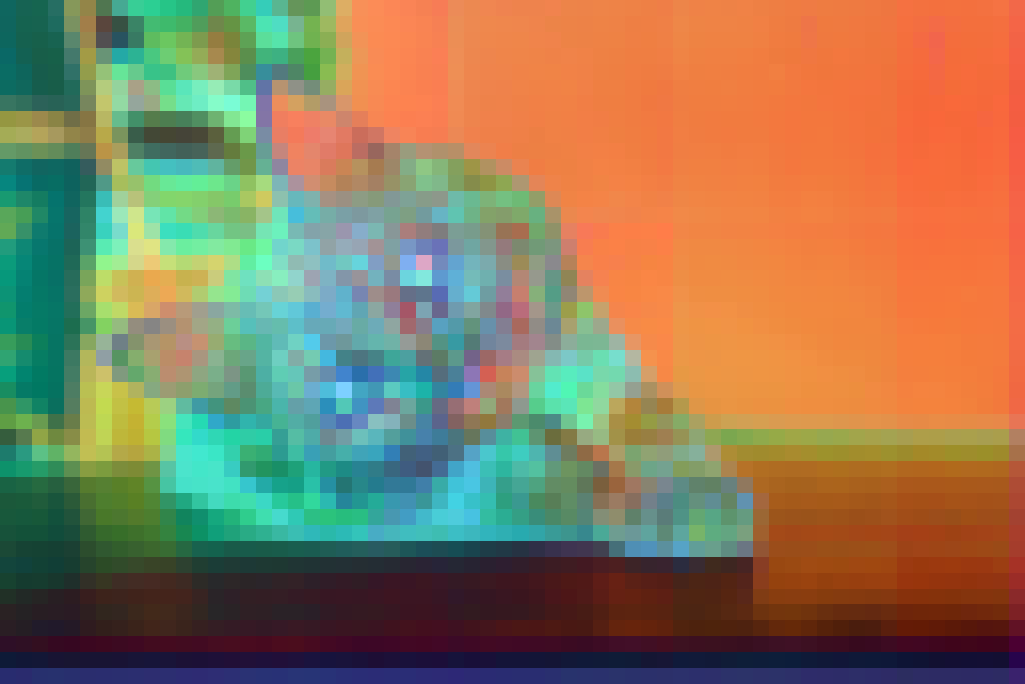} &
        \includegraphics[width=0.3\textwidth]{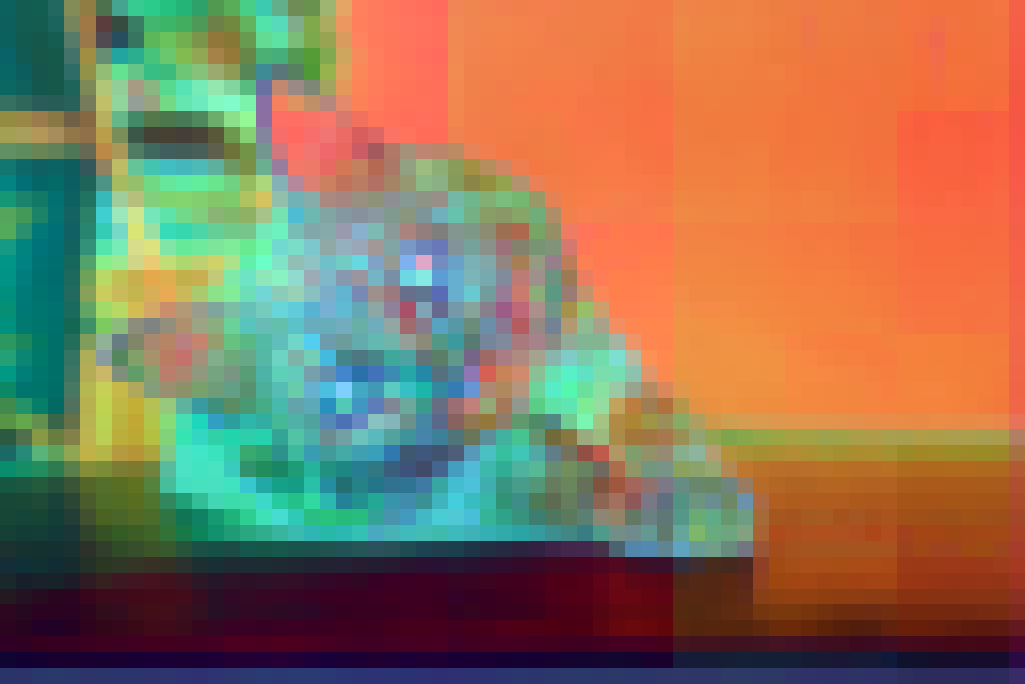} &
        \includegraphics[width=0.3\textwidth]{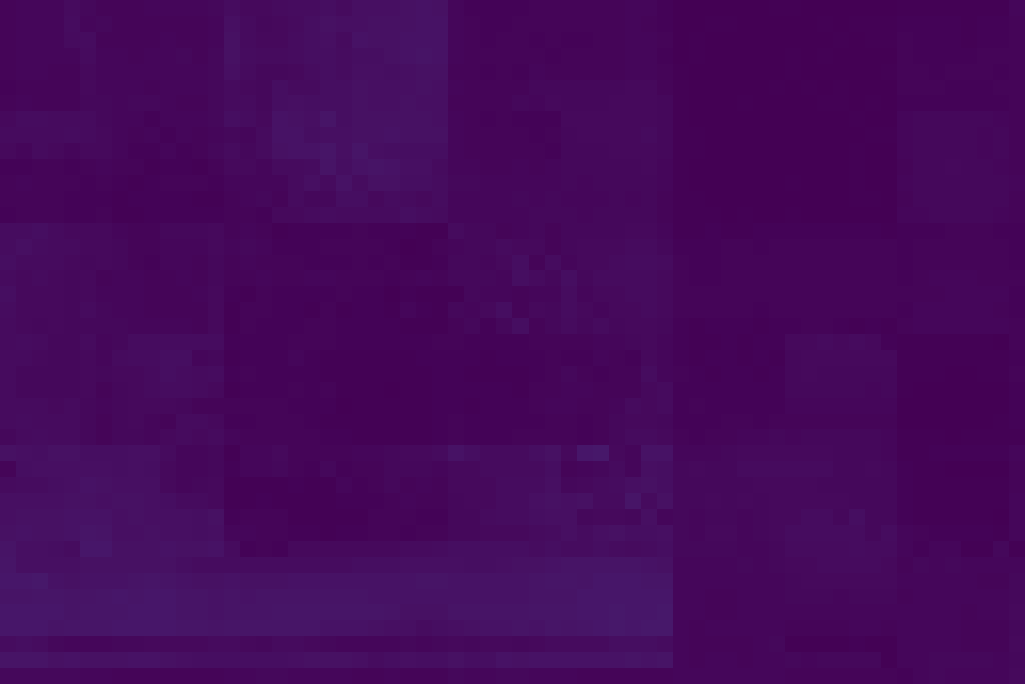} \\[0.5em]
        \rotatebox{90}{\textbf{Layer 5}} &
        \includegraphics[width=0.3\textwidth]{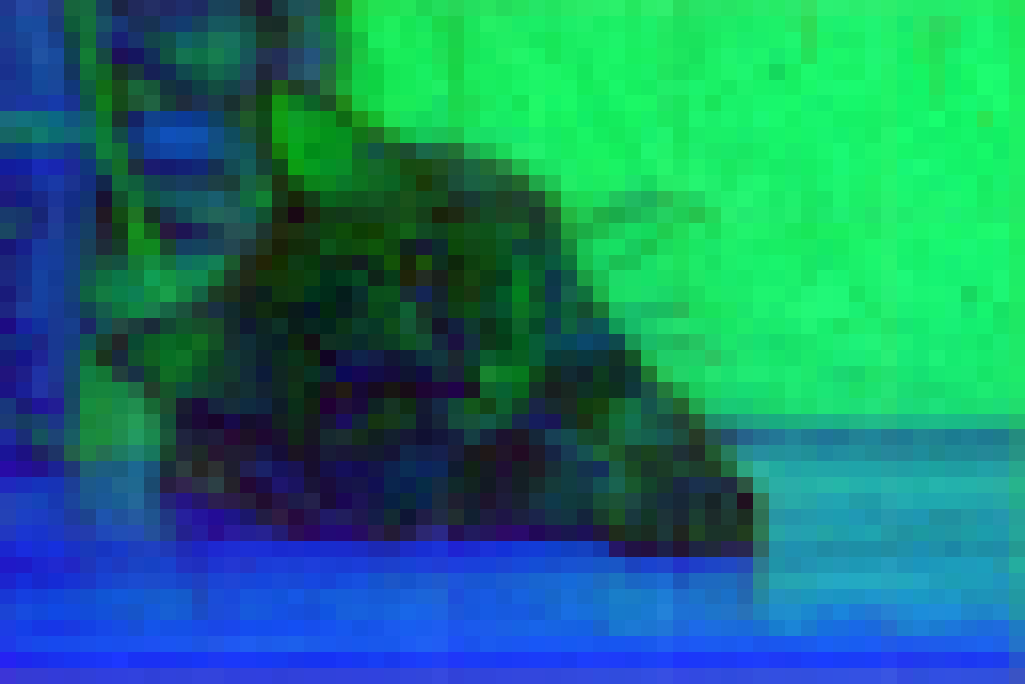} &
        \includegraphics[width=0.3\textwidth]{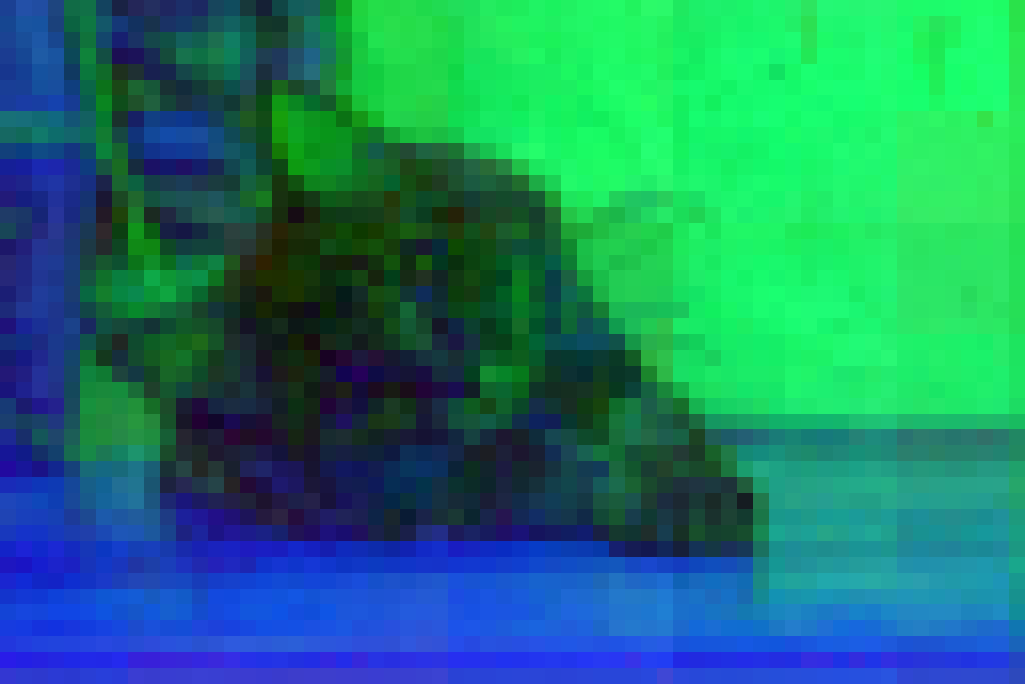} &
        \includegraphics[width=0.3\textwidth]{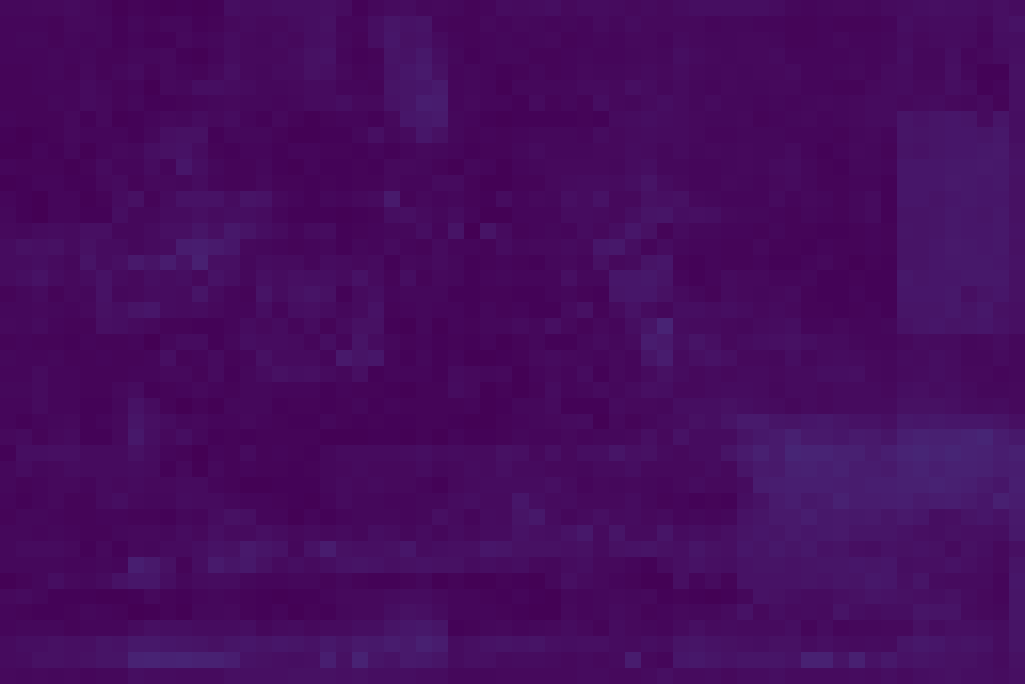} \\[0.5em]
        \rotatebox{90}{\textbf{Layer 11}} &
        \includegraphics[width=0.3\textwidth]{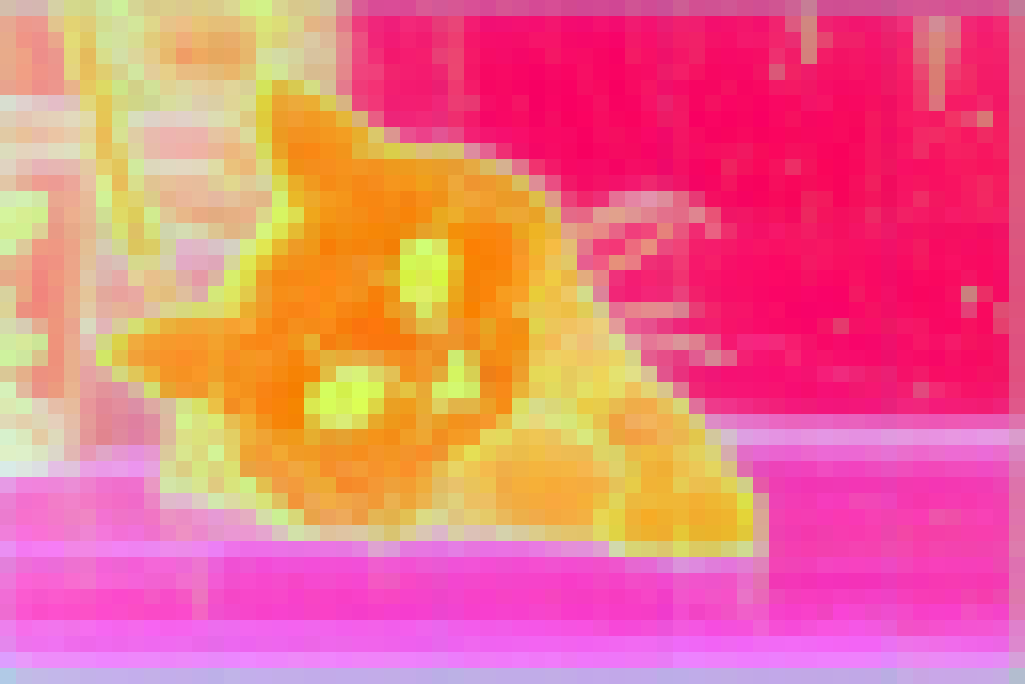} &
        \includegraphics[width=0.3\textwidth]{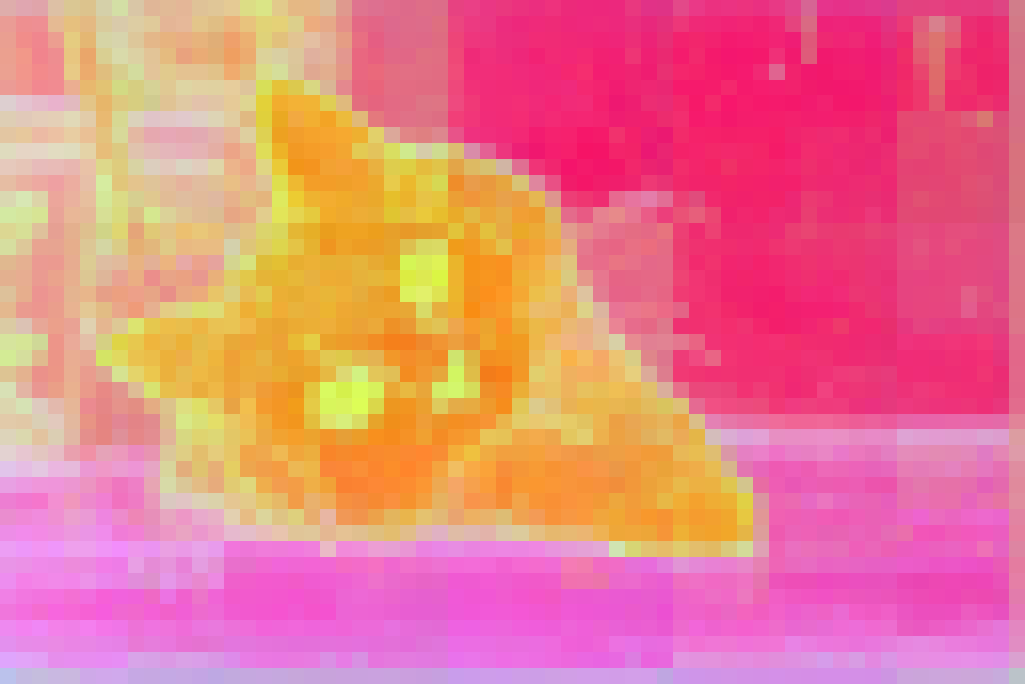} &
        \includegraphics[width=0.3\textwidth]{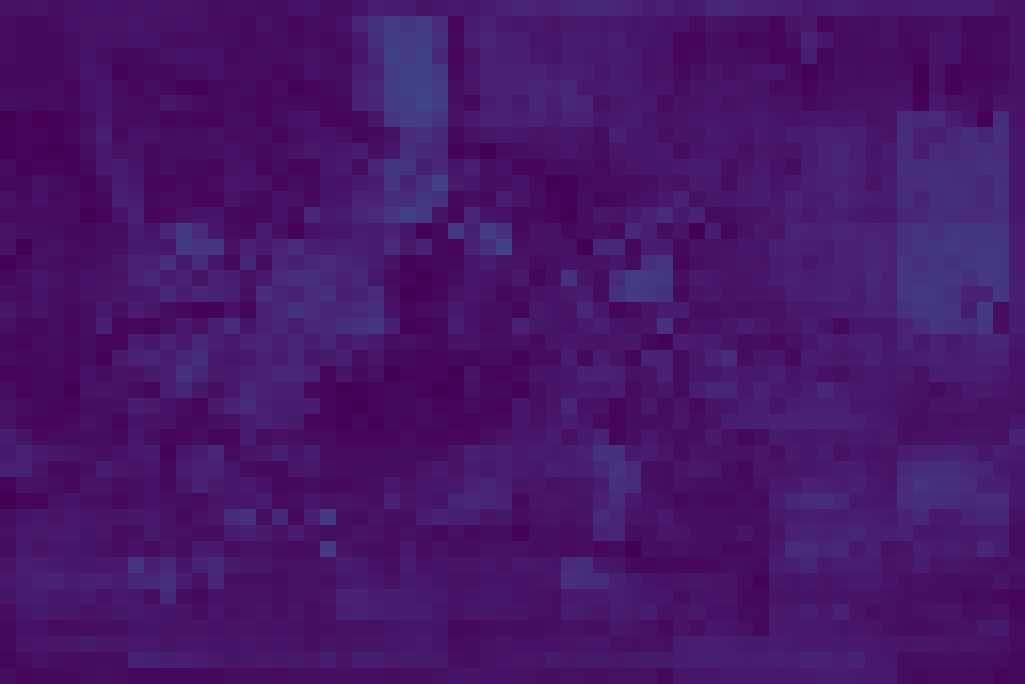} \\
    \end{tabular}
    \caption{PCA visualizations across image encoder layers.}
    \label{fig:pca_layers}
\end{figure}

\section{When the StructSAM energy will fail}

\begin{figure}[htbp]
    \centering

    \begin{subfigure}{1.0\textwidth}
        \centering
        \includegraphics[width=\linewidth]{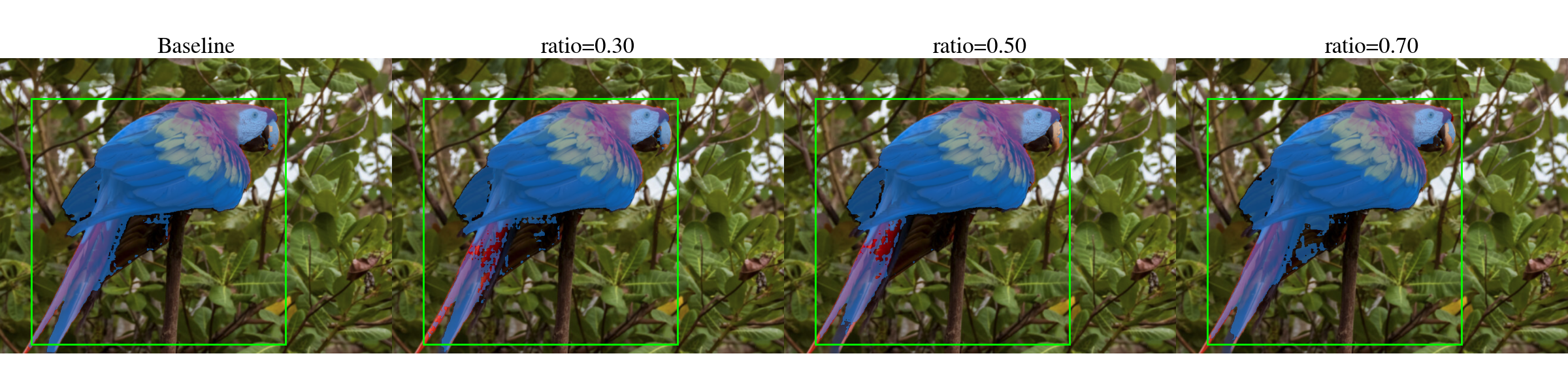}
        \caption{The baseline models exhibit low confidence around the parrot’s tail region; consequently, token merging becomes inconsistent across different merging rates.}
    \end{subfigure}
    

    \begin{subfigure}{1.0\textwidth}
        \centering
        \includegraphics[width=\linewidth]{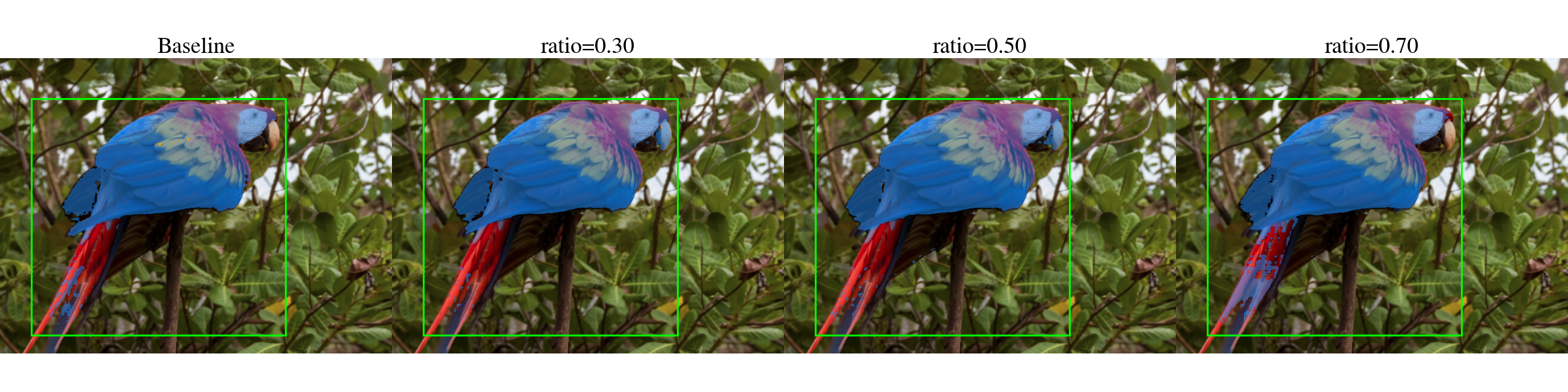}
        \caption{The baseline models exhibit low confidence around the parrot’s tail region; even a slight adjustment of the bounding box leads to significant changes in the segmentation output..}
    \end{subfigure}
    

    \begin{subfigure}{1.0\textwidth}
        \centering
        \includegraphics[width=\linewidth]{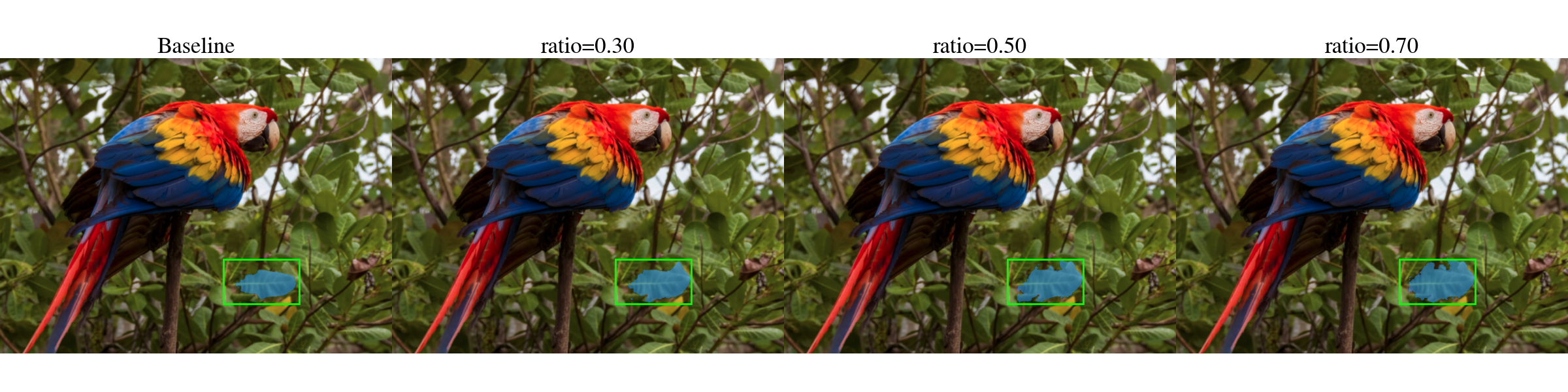}
        \caption{Leaf segmentation with token merging. Background regions are more easily merged due to weaker feature importance.}
    \end{subfigure}

    \caption{Segmentation results of SAM under different token merging rates. The parrot examples show that token merging becomes unstable in regions where the model has low confidence, particularly around challenging areas such as the tail. In contrast, the leaf example illustrates that non-foreground regions are more aggressively merged due to their lower feature saliency, which can lead to segmentation failures.}
    
    \label{fig:failure_cases}
\end{figure}

\begin{figure}[htbp]
    \centering
    \setlength{\tabcolsep}{2pt}
    \begin{tabular}{c ccc}
        & \textbf{Feat. PCA (0\% merge)} & \textbf{Feat. PCA (65\%)} & \textbf{Difference} \\[0.5em]
        \rotatebox{90}{\textbf{Layer 0}} &
        \includegraphics[width=0.3\textwidth]{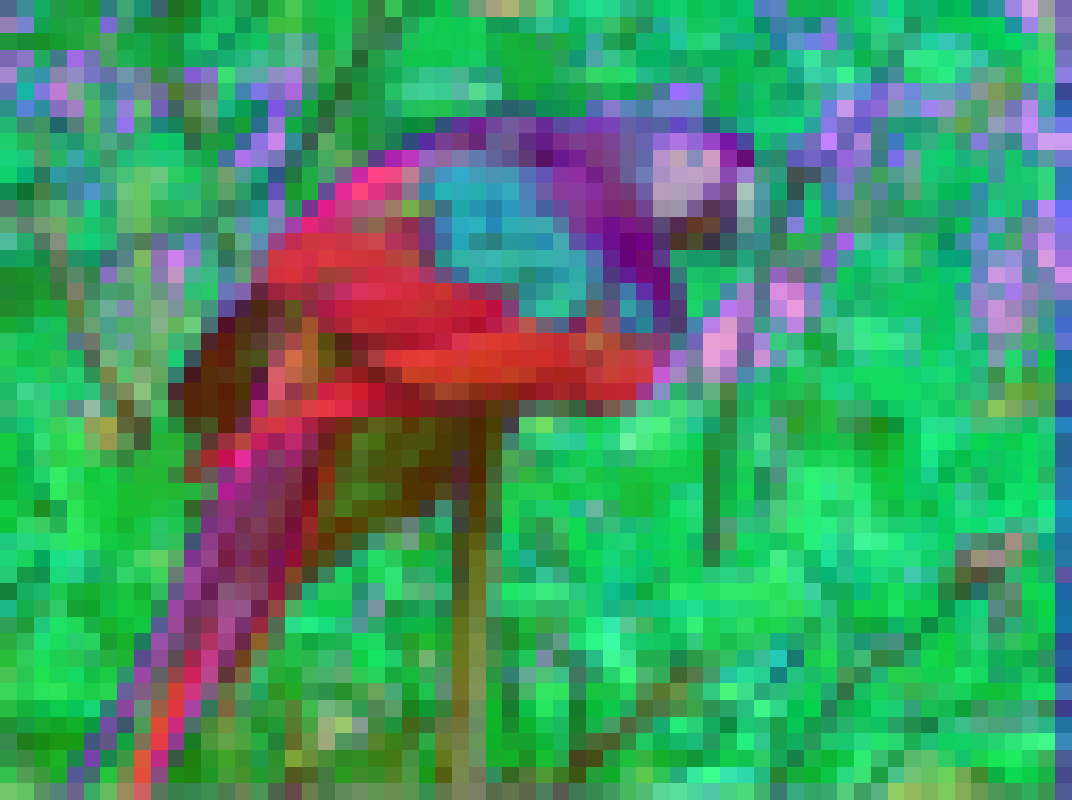} &
        \includegraphics[width=0.3\textwidth]{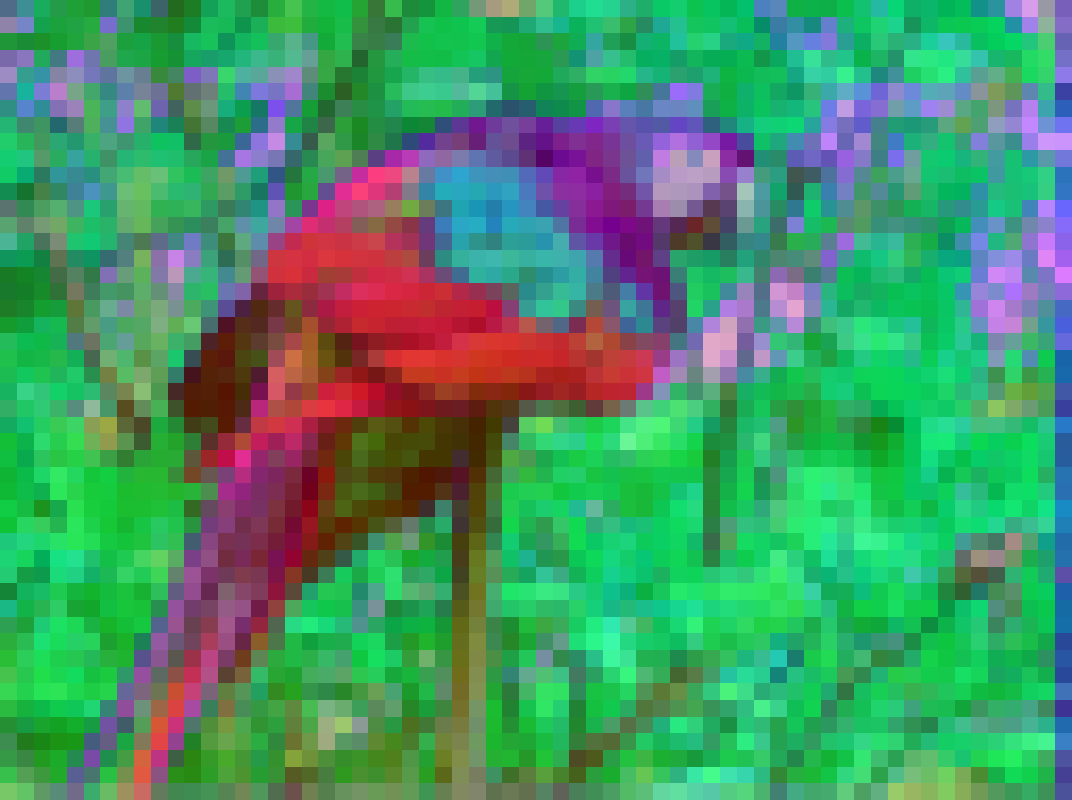} &
        \includegraphics[width=0.3\textwidth]{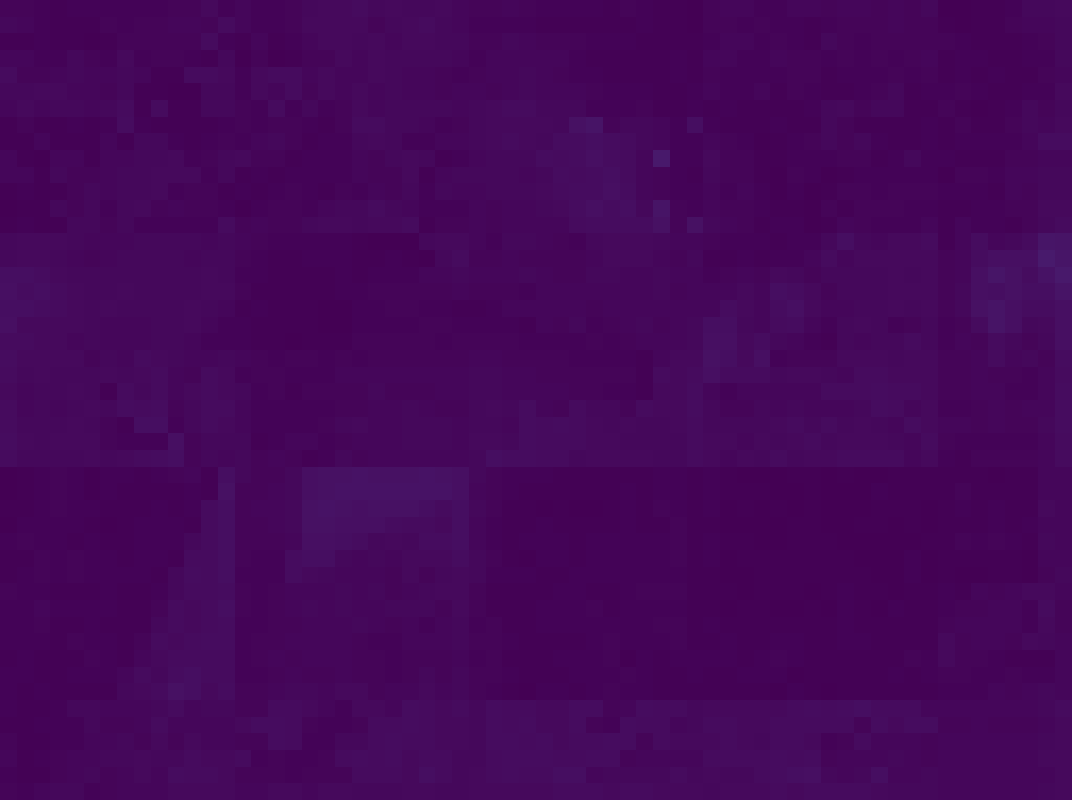} \\[0.5em]
        \rotatebox{90}{\textbf{Layer 5}} &
        \includegraphics[width=0.3\textwidth]{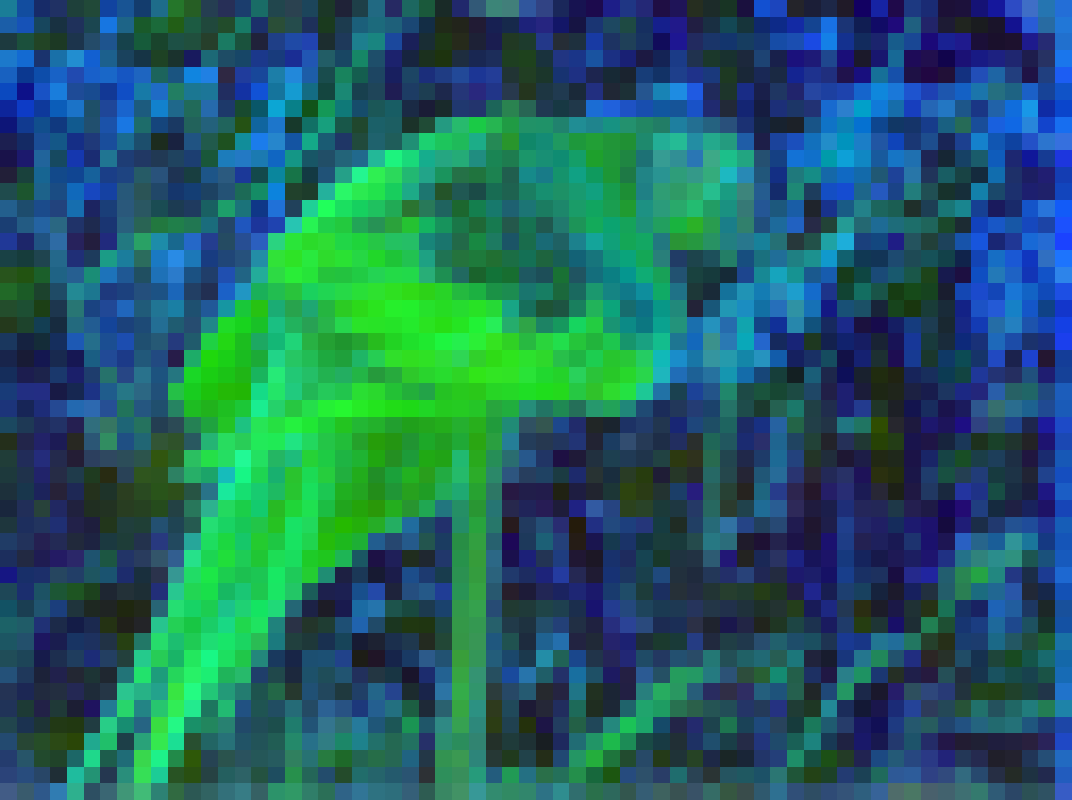} &
        \includegraphics[width=0.3\textwidth]{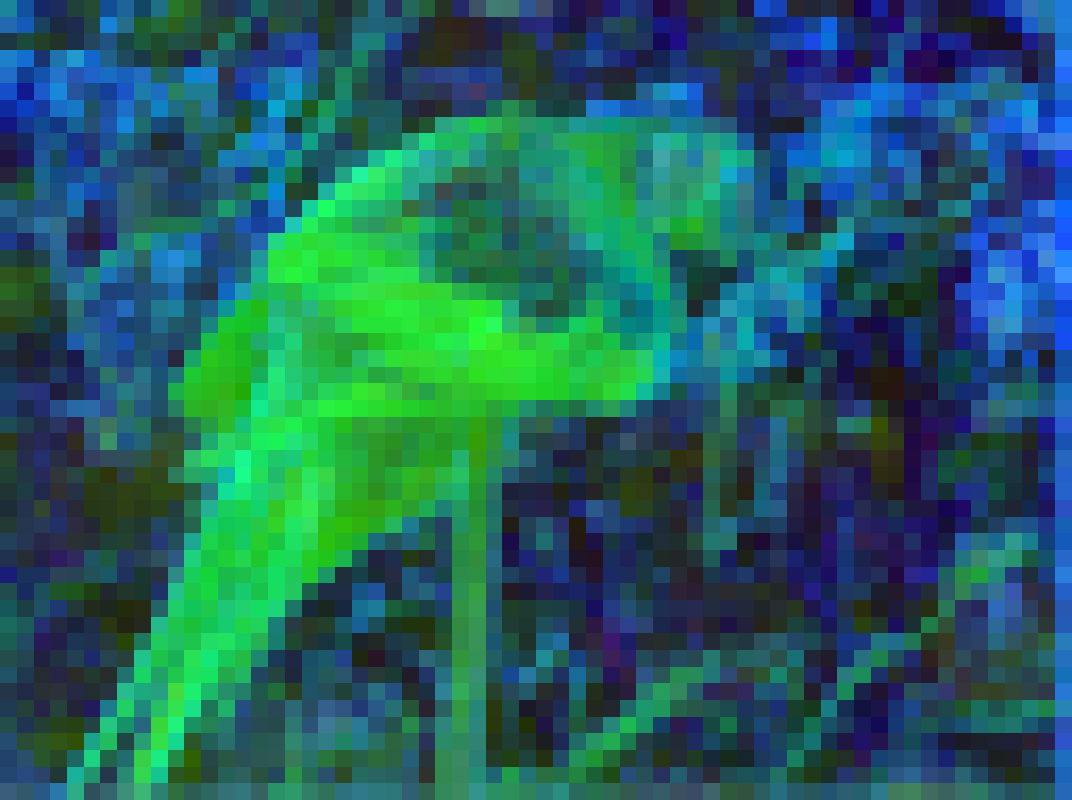} &
        \includegraphics[width=0.3\textwidth]{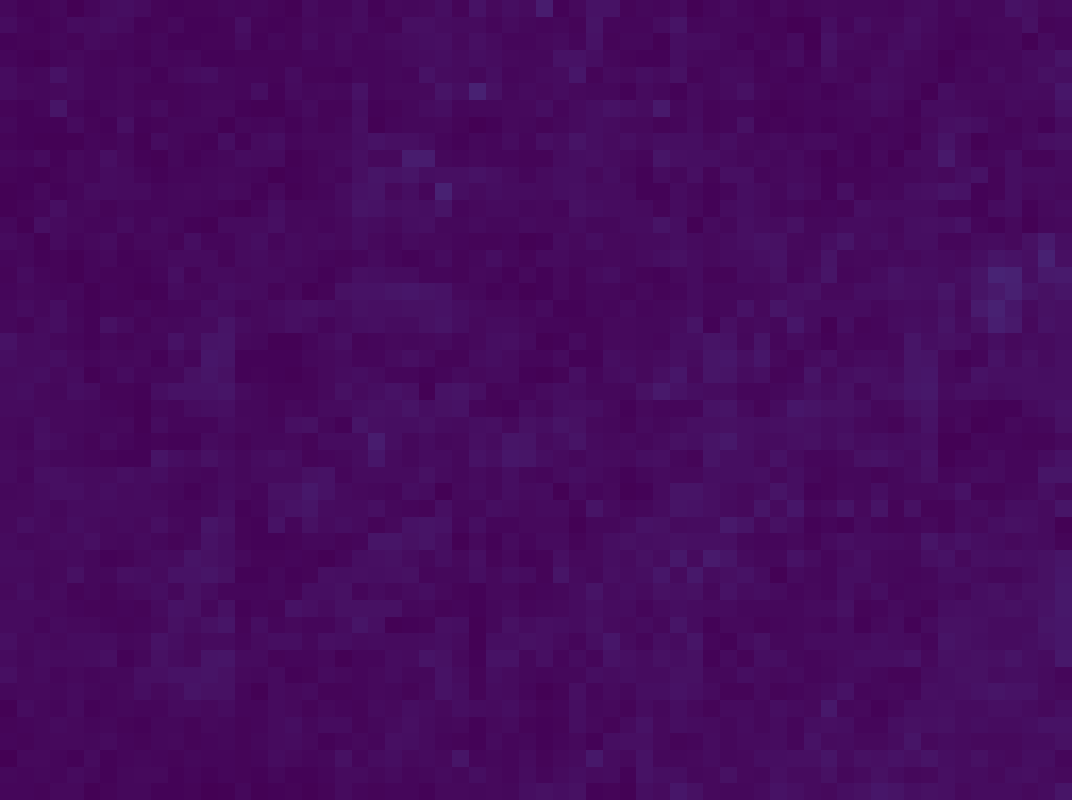} \\[0.5em]
        \rotatebox{90}{\textbf{Layer 11}} &
        \includegraphics[width=0.3\textwidth]{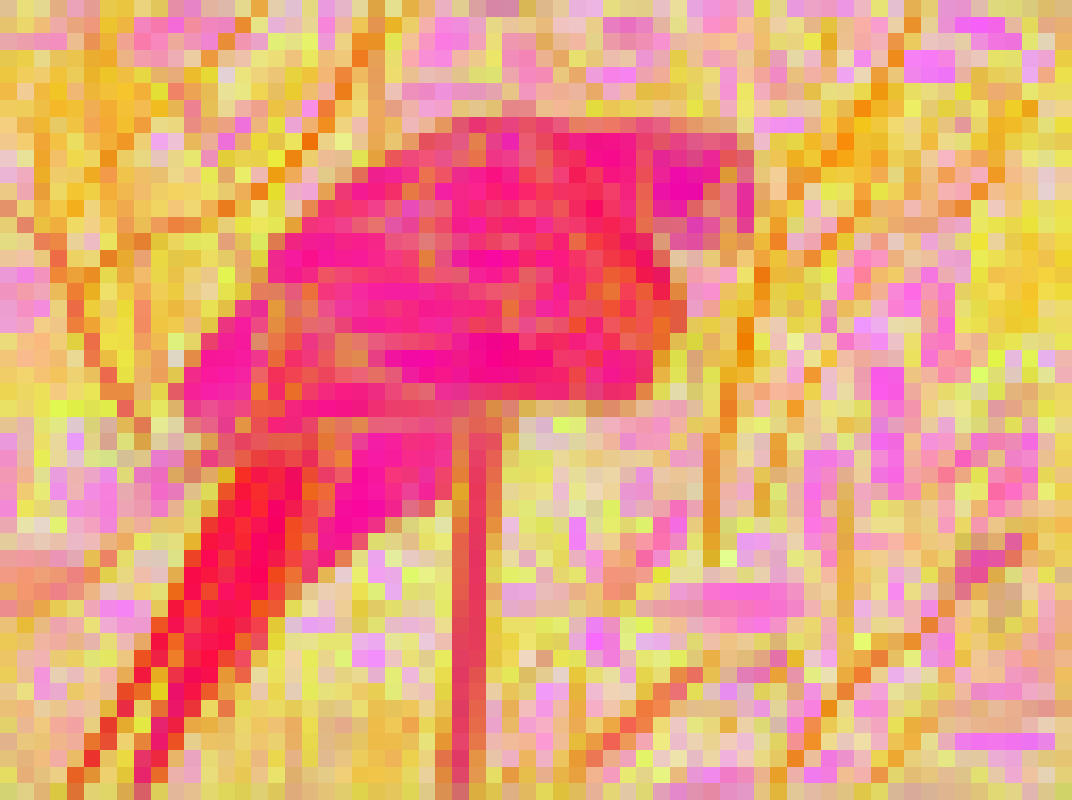} &
        \includegraphics[width=0.3\textwidth]{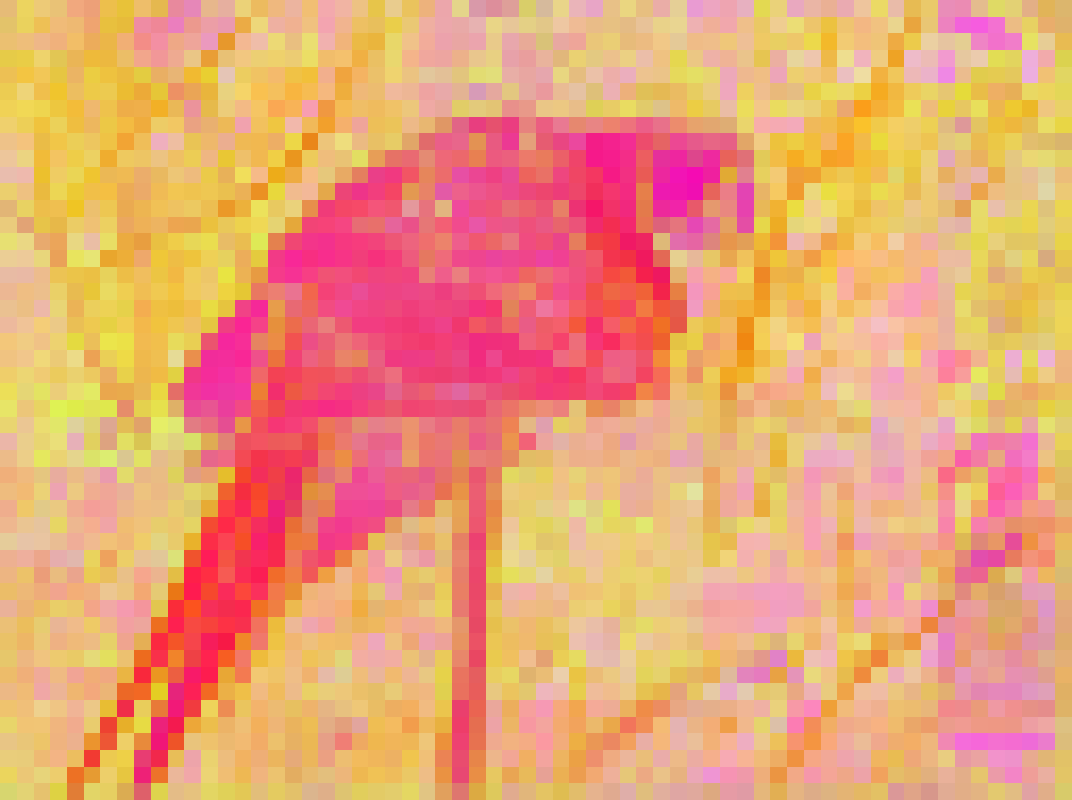} &
        \includegraphics[width=0.3\textwidth]{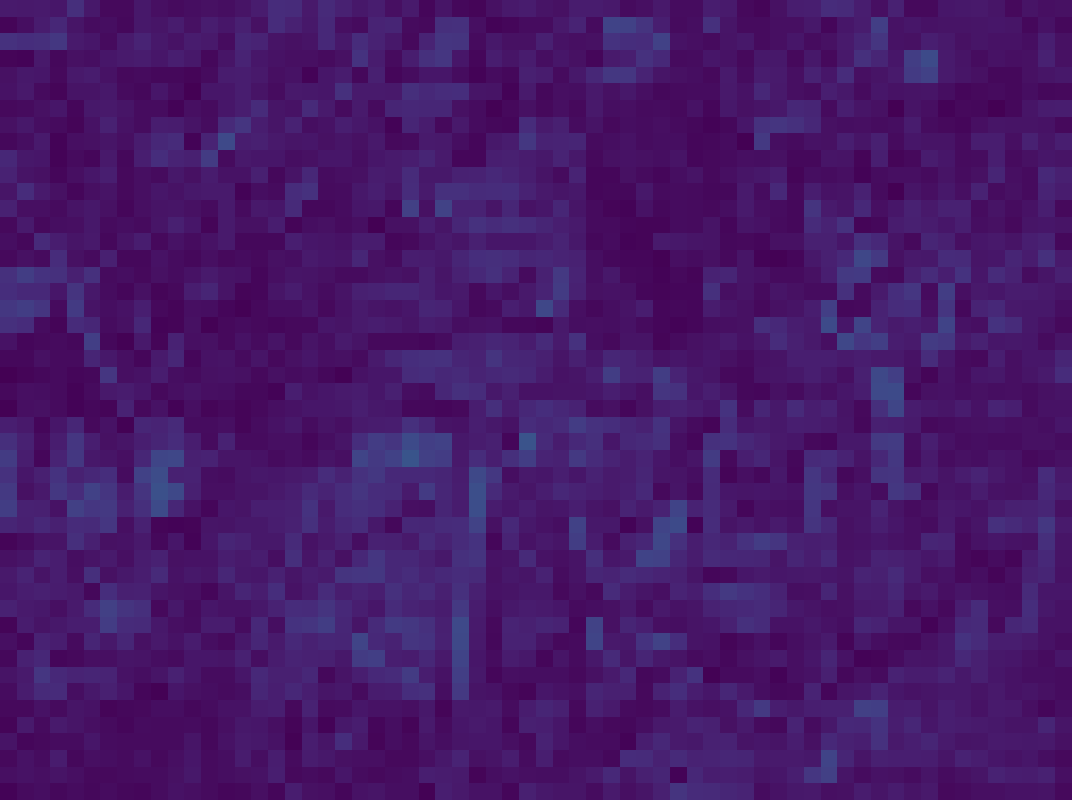} \\
    \end{tabular}
    \caption{PCA visualizations across image encoder layers.}
    \label{fig:pca_layers_parrot}
\end{figure}

StructSAM becomes more unstable when processing regions where the original model exhibits low confidence or high ambiguity induced by the prompt. The top two rows of Fig.~\ref{fig:failure_cases} illustrate this failure mode: the tail of the parrot is an ambiguous region, where even slight changes in the bounding box can lead to large variations in the original SAM output. In such cases, StructSAM shows increased instability across different merging rates.

Another failure case arises from StructSAM’s tendency to prioritize merging background features while preserving highly distinctive and salient foreground objects. Although this behavior can improve segmentation quality for foreground regions, it makes background segmentation more challenging, as shown in the last row of Fig.~\ref{fig:failure_cases} and the last row of Fig.~\ref{fig:pca_layers_parrot}.

\section{Additional Ablation Studies}
\paragraph{Cell Size}: Table~\ref{tab:cell_size_effect} presents an ablation study on cell size. To ensure compatibility, we select cell sizes that are divisible by the SAM model's window size. Our results indicate that $2 \times 2$ cells yield the highest quality, as they offer superior spatial preservation compared to larger configurations.

\paragraph{Additional ablation on the COIFT dataset:}
We further extend the ablation study presented in the main paper to the COIFT dataset in Table~\ref{tab:ablation_sam_b_coift}, providing additional evidence to support the design choices adopted in the final model.

\begin{table}[!h]
\scriptsize
\centering
\caption{\label{tab:ablation_sam_b_coift}Ablation study on SAM-B. mIoU and boundary IoU (B-IoU) are reported in \%. GradCell (Full) denotes the complete model using Sobel-based flatness and structured cell sampling.}

\setlength{\tabcolsep}{5pt}
\renewcommand{\arraystretch}{1.1}
\begin{tabular}{l l cc cc}
\toprule
& & \multicolumn{2}{c}{$r=0.35$} & \multicolumn{2}{c}{$r=0.55$} \\
\cmidrule(lr){3-4} \cmidrule(lr){5-6}
Dataset & Method & mIoU & B-IoU & mIoU & B-IoU \\
\midrule
\multirow{7}{*}{COIFT}
& GradCell (Full) & \textbf{64.0} & \textbf{53.5} & \textbf{63.3} & \textbf{52.1} \\
& Central-Diff     & 60.5 & 50.4 & 59.5 & 49.1 \\
& Mean-Flatness    & 63.8 & 53.3 & 62.9 & 52.1 \\
& No-Cell          & 62.7 & 50.6 & 59.2 & 47.7 \\
& Rand-Cell        & 63.9 & 52.5 & 63.0 & 51.1 \\
& Max-Dst          & 63.3 & 52.7 & 62.0 & 51.4 \\
& Rand-Dst         & 63.7 & 53.2 & 62.8 & 51.9 \\
\bottomrule
\end{tabular}
\end{table}

\begin{table}[!h]
  \centering
  \caption{Effect of cell size on INbreast dataset segmentation performance. The image encoder processes tokens on a $64 \times 64$ grid for global attention layers and a $14 \times 14$ grid for window attention layers, so the cell sizes must be divisors of their respective grid dimensions.}
  \label{tab:cell_size_effect}
  \begin{tabular}{cccc}
    \toprule
    \textbf{Window Cell Size} & \textbf{Global Cell Size} & \textbf{Dice Score} & \textbf{GFLOPs} \\
    \midrule
    $2 \times 2$   & $2 \times 2$   & 0.7551 & 348.3 \\
    $7 \times 7$   & $8 \times 8$   & 0.7481 & 347.8 \\
    $14 \times 14$ & $16 \times 16$ & 0.7364 & 347.8 \\
    \bottomrule
  \end{tabular}
\end{table}

\section{\textcolor{black}{Extension StructSAM to Efficient-SAM for Video Tracking}}

We extend StructSAM to Efficient-SAM to enable lightweight video tracking by leveraging token merging for improved efficiency. While merging may introduce minor accuracy degradation, the resulting segmentation quality remains sufficient for tracking scenarios that rely on coarse prompts such as bounding boxes. This design prioritizes speed and scalability, making it suitable for real-time applications. Our goal is to evaluate whether a StructSAM-enhanced Efficient-SAM can serve as a practical alternative to more powerful but computationally intensive models such as SAM-2 for video tracking.


We illustrate the tracking quality of EfficientSAM in Fig.~\ref{fig:tracksam_quantitative}. Despite its simplicity, the method with a 70\% merging rate performs on par with SAM2 while requiring significantly less memory and computation. Fig.~\ref{fig:mergingrates} shows the gains in throughput and memory consumption when applying StructSAM to EfficientSAM; at a 70\% merging rate, the system achieves real-time performance at 30 frames per second. Fig.~\ref{fig:IAVLA} demonstrates that replacing SAM2 with EfficientSAM (70\% merging rate) in our tracking pipeline maintains comparable performance on the robot stacking task with IA-VLA, while substantially improving efficiency and speed.
\begin{figure*}[!hbt]
\centering
\includegraphics[width=1.0\textwidth,height=\textheight,keepaspectratio]{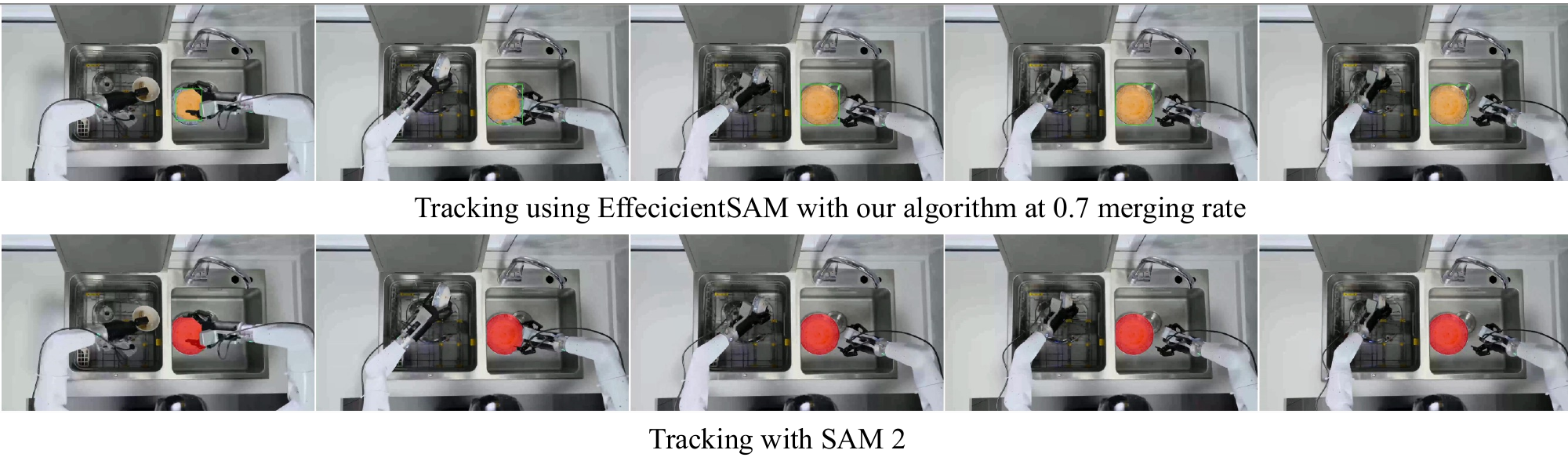}

\caption{\label{fig:tracksam_quantitative}\textbf{Results.} We compare TrackSAM (with a merge rate of 0.70) against the base model EfficientSAM and a SAM2. While the overall performance is similar, our method enables real-time processing with lower memory
consumption. }
\end{figure*}
\begin{figure}[!hbt] 
  \centering
  \includegraphics[width=0.8\columnwidth]{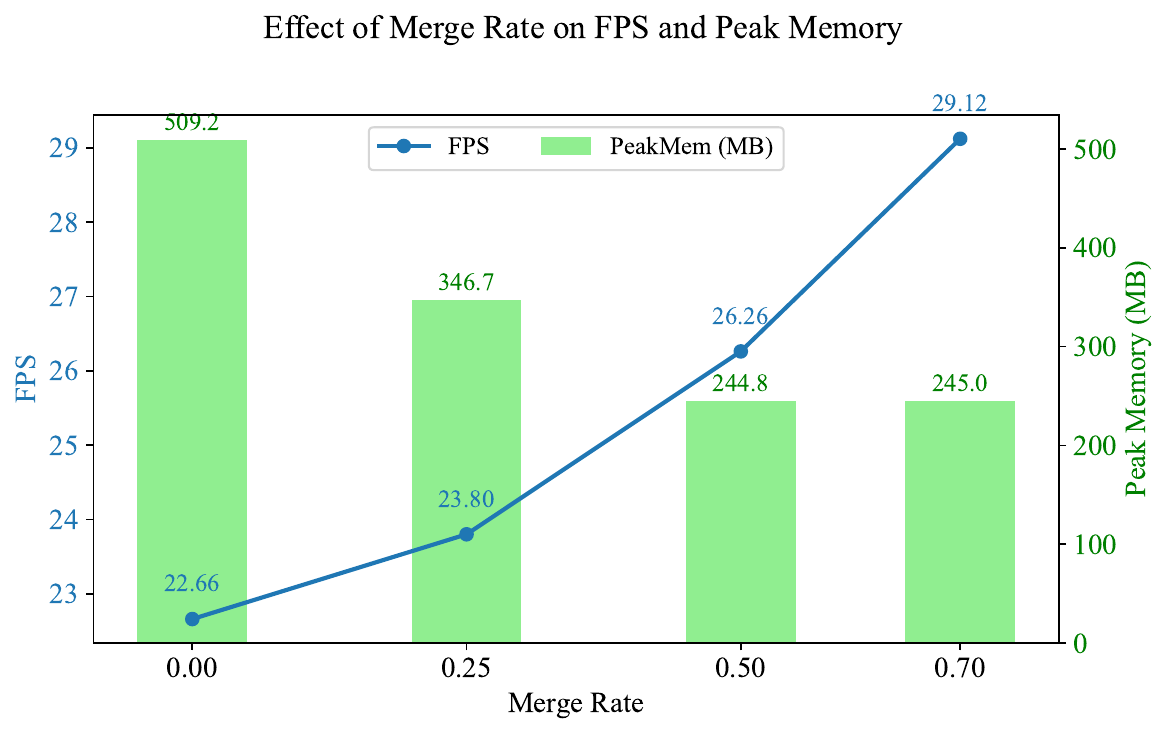}
  \caption{At 70\% merging rate, we boost up robot execution speed to 3x while still maintaining success rate}
  \label{fig:mergingrates}
\end{figure}
\begin{figure}[!hbt] 
  \centering
  \includegraphics[width=0.8\columnwidth]{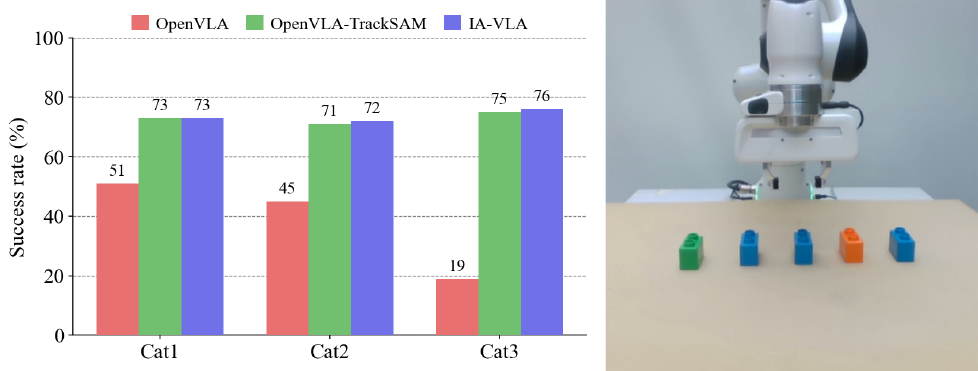}
  \caption{At 70\% merging rate, our method (Green) boost up robot execution speed to 1.4x while keeping the same performance as IA-VLA (Blue).}
  \label{fig:IAVLA}
\end{figure}

\paragraph{\textcolor{black}{Extension for efficient video object tracking}}

StructSAM extends to video co-tracking and segmentation by propagating masks across frames while restricting computation to relevant regions. For a video sequence ${I_i}$, the bounding box $\partial_i$ derived from the previous mask $M_{i-1}$ is used as a prompt to segment the current frame, yielding $M_i = \text{SAM}(I_i, \partial_i)$. Computation is focused within $\partial_i$, while tokens outside are merged or skipped for efficiency. The next bounding box is updated via $\partial_{i+1} = \mathcal{B}(M_i)$, enabling iterative mask propagation. This region-focused strategy allows efficient and accurate segmentation over time by leveraging temporal consistency between consecutive frames.


\begin{figure}[!h]
    \centering

    \begin{minipage}{0.19\textwidth}
        \centering
        \includegraphics[width=\linewidth]{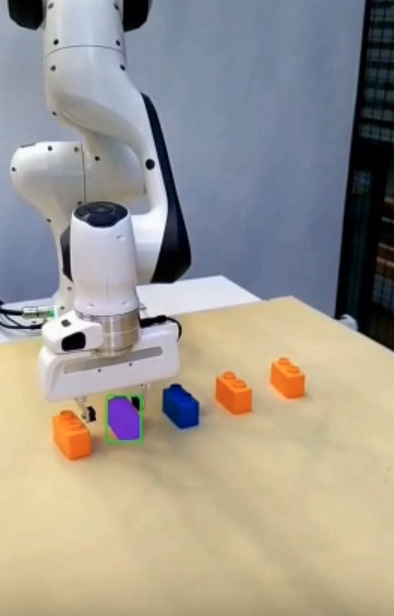}
        \caption*{Frame 1}
    \end{minipage}
    \hfill
    \begin{minipage}{0.19\textwidth}
        \centering
        \includegraphics[width=\linewidth]{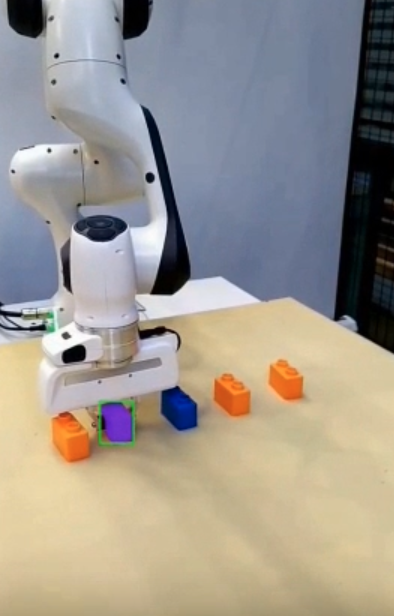}
        \caption*{Frame 30}
    \end{minipage}
    \hfill
    \begin{minipage}{0.19\textwidth}
        \centering
        \includegraphics[width=\linewidth]{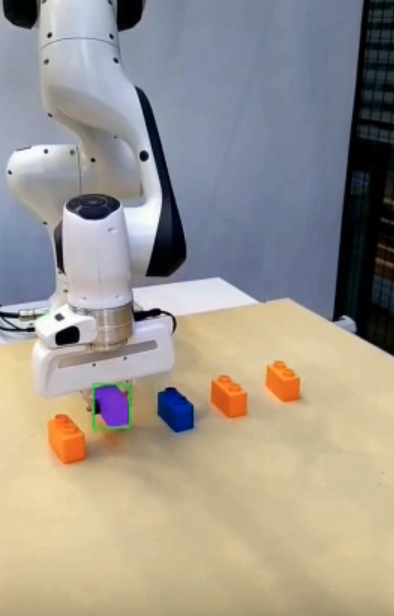}
        \caption*{Frame 60}
    \end{minipage}
    \hfill
    \begin{minipage}{0.19\textwidth}
        \centering
        \includegraphics[width=\linewidth]{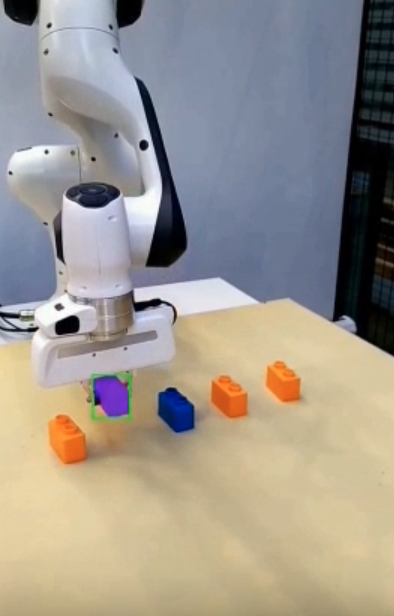}
        \caption*{Frame 90}
    \end{minipage}
    \hfill
    \begin{minipage}{0.19\textwidth}
        \centering
        \includegraphics[width=\linewidth]{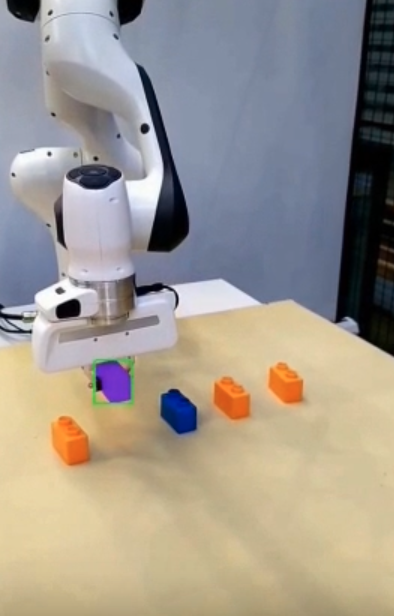}
        \caption*{Frame 120}
    \end{minipage}

    \caption{\label{fig:effsam_segment_quant} Segmentation and tracking across video frames, merging rate at 0.8}
\end{figure}

\paragraph{\textcolor{black}{Quantitative analysis of the effect of StructSAM on EfficientSAM}}
In this part, we assess the change in performance of EfficientSAM when it is extended by StructSAM. We illustrate the tracking performance in a robotic task in Figure~\ref{fig:effsam_segment_quant}. 
For evaluation, we compare the bounding box and segmentation results of EfficientSAM tracking with those of EfficientSAM + StructSAM tracking (using the algorithm described above), in terms of mIoU over a 15s video, as shown in Fig.~\ref{fig:effsam_segment_quant}. Table ~\ref{tab:structsam_efficientsam} demonstrate the quantitative results. Even with the high merging ratio, StructSAM still shows comparative results to original EfficientSAM, with Fig.~\ref{fig:effsam_segment_quant} showing the success of StructSAM at merging ratio of 0.8. Even though the task is simple, which explains the mIoU and bbox IoU even at high merging ratio, it shows that in many cases StructSAM can be used to speed up an already efficient algorithm significantly while still maintaining decent performance. 

\begin{table}[ht]                                            
      \centering      
      \caption{Segmentation and bounding box quality of applying StructSAM on EfficientSAM}
      \label{tab:structsam_efficientsam}                 \begin{tabular}{lcc}                               \toprule                                           Merging rate & mIoU & Bbox IoU \\                  \midrule                                           $R = 0.0$ (baseline) & 0.920 & 0.916 \\            $R = 0.6$            & 0.929 & 0.942 \\            $R = 0.7$            & 0.928 & 0.942 \\            $R = 0.8$            & 0.931 & 0.947 \\            $R = 0.9$            & 0.922 & 0.932 \\
       \bottomrule                                       \end{tabular}
  \end{table}

\paragraph{\textcolor{black}{Robot Setup for VLA experiments}}
To assess the benefits of StructSAM in downstream robotic applications, we integrate it into Efficient-SAM for object tracking within vision-language-action (VLA) pipelines. We then evaluate its performance through real-world experiments designed to test robustness against distractor objects. In these experiments, a Franka Research 3 robot is instructed to pick up the correct Lego block based on a human prompt (see Figure~\ref{fig:blocks}), requiring both semantic and visual reasoning to accurately identify the target.

The language instructions follow structured patterns such as \textit{“lift the {leftmost / rightmost} {orange / green / blue} block”} and \textit{“lift the {second / third / fourth / fifth} {orange / green / blue} block from the {left / right}.”}

We collect 120 demonstrations across 12 distinct language instructions covering a subset of the instructions above. Each scene is defined by the number, color, and spatial arrangement of the blocks, with semantic concepts grounded in color and positional references.

To evaluate generalization of the VLA models, we divide tasks into three difficulty categories. Category 1 includes instructions seen during training. Category 2 combines familiar positional references with novel color assignments. Category 3 introduces previously unseen ways of referring to object positions.

We use OpenVLA with raw robot observations as a baseline. To mitigate the impact of distractors, inspired by \cite{hannus2025iavlainputaugmentationvisionlanguageaction}, we augment the input with a highlighted mask of the target object, generated using StructSAM-enhanced Efficient-SAM and SAM2. We then compare baseline performance against OpenVLA augmented with SAM2 and with StructSAM-enhanced Efficient-SAM. The results show that input augmentation consistently improves performance across all task categories. In addition, StructSAM integrated into Efficient-SAM achieves task success rates comparable to SAM2 while delivering a 45\% speedup, highlighting an effective balance between efficiency and performance in both robotic and medical imaging contexts.
\begin{figure}[t]
\centering
    \includegraphics[width=.6\textwidth]{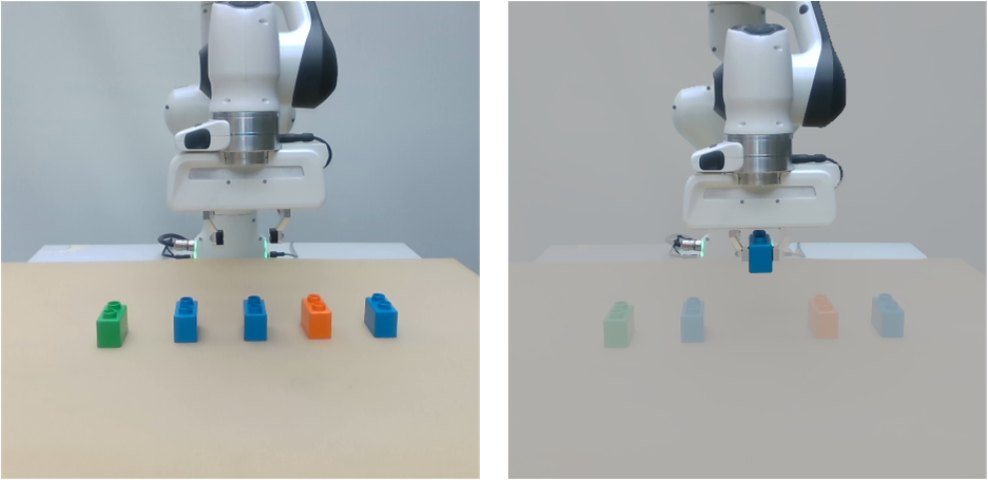}
    \setlength{\belowcaptionskip}{-6pt}
    \caption{Task: \textit{"lift the second blue block from the right"}. The raw robot observation is on the left, and the augmented observation with propagated masks at the end of the task is at the right.} 
    \label{fig:blocks}
\end{figure}

\section{Additional Results}
We present in Figure~\ref{fig:inbreast_vis} a visualization of token-merging outputs from different algorithms on a representative image from the INbreast dataset. Additionally, we provide a detailed comparison between StructSAM and baseline methods across different merging ratios on four high-quality datasets and two SAM variants (ViT-B and ViT-L backbones), as shown in Figs.~\ref{fig:samb_flops_ds0},~\ref{fig:samb_flops_ds1}, and~\ref{fig:saml_flops_ds0}. We show additional analysis for peak GPU memory consumptions from Figs.~\ref{fig:samb_peakmem_ds0}~and~\ref{fig:samb_peakmem_ds1}.

In Table~\ref{tab:point_prompt_results}, we additionally report segmentation performance under point-prompt evaluation. We follow the original SAM evaluation protocol, where points are sampled from the ground-truth mask and used sequentially to progressively refine the predicted segmentation.

\begin{figure*}[!hbt]
\centering
\includegraphics[width=\textwidth,height=\textheight,keepaspectratio]{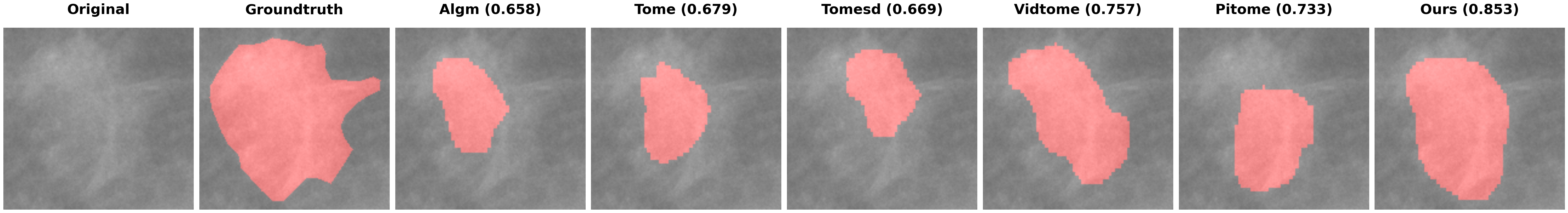}
\caption{Segmentation results on INbreast dataset using MedSAM}
\label{fig:inbreast_vis}
\end{figure*}

\begin{table}[ht]
  \centering
  \caption{Point-prompt segmentation results across all datasets (model: sam-b, merge ratio: 0.55). \textcolor[HTML]{006400}{\textbf{Green}} = best, \textcolor{blue}{blue} = second-best per column.}
  \label{tab:point_prompt_results}
  \begin{tabular}{ll cc cc cc}
    \toprule
    Dataset & Method & \multicolumn{2}{c|}{1 Point} & \multicolumn{2}{c|}{2 Points} & \multicolumn{2}{c}{3 Points} \\
    & & mIoU & b-mIoU & mIoU & b-mIoU & mIoU & b-mIoU \\
    \midrule
    \multirow{5}{*}{\rotatebox{90}{DIS5K}} & Baseline (No Merging) & 15.98 & 13.58 & \textcolor{blue}{34.54} & \textcolor[HTML]{006400}{\textbf{29.10}} & \textcolor[HTML]{006400}{\textbf{47.86}} & \textcolor[HTML]{006400}{\textbf{39.94}} \\
     & ToMeSD & \textcolor{blue}{16.72} & \textcolor{blue}{13.81} & 34.21 & 28.08 & 46.43 & 37.99 \\
     & ALGM & 13.09 & 10.11 & 30.76 & 23.94 & 43.14 & 33.68 \\
     & vidtome & 11.52 & 9.59 & 26.61 & 21.74 & 39.75 & 31.90 \\
     & \textbf{StructSAM (Ours)} & \textcolor[HTML]{006400}{\textbf{17.03}} & \textcolor[HTML]{006400}{\textbf{13.96}} & \textcolor[HTML]{006400}{\textbf{34.65}} & \textcolor{blue}{28.35} & \textcolor{blue}{47.57} & \textcolor{blue}{38.89} \\
    \midrule
    \multirow{5}{*}{\rotatebox{90}{ThinObject5K}} & Baseline (No Merging) & 46.70 & 39.02 & 69.04 & 58.54 & \textcolor[HTML]{006400}{\textbf{78.43}} & \textcolor{blue}{67.88} \\
     & ToMeSD & \textcolor{blue}{50.04} & \textcolor{blue}{42.05} & \textcolor{blue}{70.84} & \textcolor{blue}{59.89} & 77.69 & 67.50 \\
     & ALGM & 48.60 & 40.32 & 69.97 & 59.87 & 77.54 & 66.83 \\
     & vidtome & 40.92 & 31.80 & 60.63 & 47.10 & 70.75 & 56.45 \\
     & \textbf{StructSAM (Ours)} & \textcolor[HTML]{006400}{\textbf{52.54}} & \textcolor[HTML]{006400}{\textbf{44.40}} & \textcolor[HTML]{006400}{\textbf{71.54}} & \textcolor[HTML]{006400}{\textbf{60.79}} & \textcolor{blue}{78.43} & \textcolor[HTML]{006400}{\textbf{67.95}} \\
    \midrule
    \multirow{5}{*}{\rotatebox{90}{HRSOD}} & Baseline (No Merging) & 47.33 & 39.67 & 66.18 & 56.37 & 75.68 & 64.88 \\
     & ToMeSD & \textcolor{blue}{49.92} & \textcolor{blue}{42.08} & \textcolor{blue}{66.60} & \textcolor{blue}{56.99} & \textcolor{blue}{75.77} & \textcolor{blue}{65.13} \\
     & ALGM & 47.46 & 38.79 & 64.82 & 54.13 & 73.96 & 62.51 \\
     & vidtome & 36.96 & 29.02 & 54.68 & 43.30 & 68.00 & 53.84 \\
     & \textbf{StructSAM (Ours)} & \textcolor[HTML]{006400}{\textbf{52.07}} & \textcolor[HTML]{006400}{\textbf{44.10}} & \textcolor[HTML]{006400}{\textbf{68.62}} & \textcolor[HTML]{006400}{\textbf{58.74}} & \textcolor[HTML]{006400}{\textbf{76.34}} & \textcolor[HTML]{006400}{\textbf{65.90}} \\
    \midrule
    \multirow{5}{*}{\rotatebox{90}{COIFT}} & Baseline (No Merging) & 37.51 & 29.52 & 58.08 & 45.57 & \textcolor[HTML]{006400}{\textbf{69.73}} & \textcolor[HTML]{006400}{\textbf{55.15}} \\
     & ToMeSD & \textcolor{blue}{40.29} & \textcolor{blue}{31.14} & \textcolor{blue}{58.67} & \textcolor{blue}{45.62} & 69.05 & 54.31 \\
     & ALGM & 38.81 & 28.56 & 56.23 & 41.54 & 66.16 & 49.18 \\
     & vidtome & 27.73 & 19.29 & 46.02 & 31.97 & 57.80 & 40.46 \\
     & \textbf{StructSAM (Ours)} & \textcolor[HTML]{006400}{\textbf{41.42}} & \textcolor[HTML]{006400}{\textbf{31.97}} & \textcolor[HTML]{006400}{\textbf{59.32}} & \textcolor[HTML]{006400}{\textbf{46.26}} & \textcolor{blue}{69.64} & \textcolor{blue}{54.92} \\
    \bottomrule
  \end{tabular}
\end{table}

\begin{figure}[!hbt]
\centering
\includegraphics[width=0.9\linewidth]{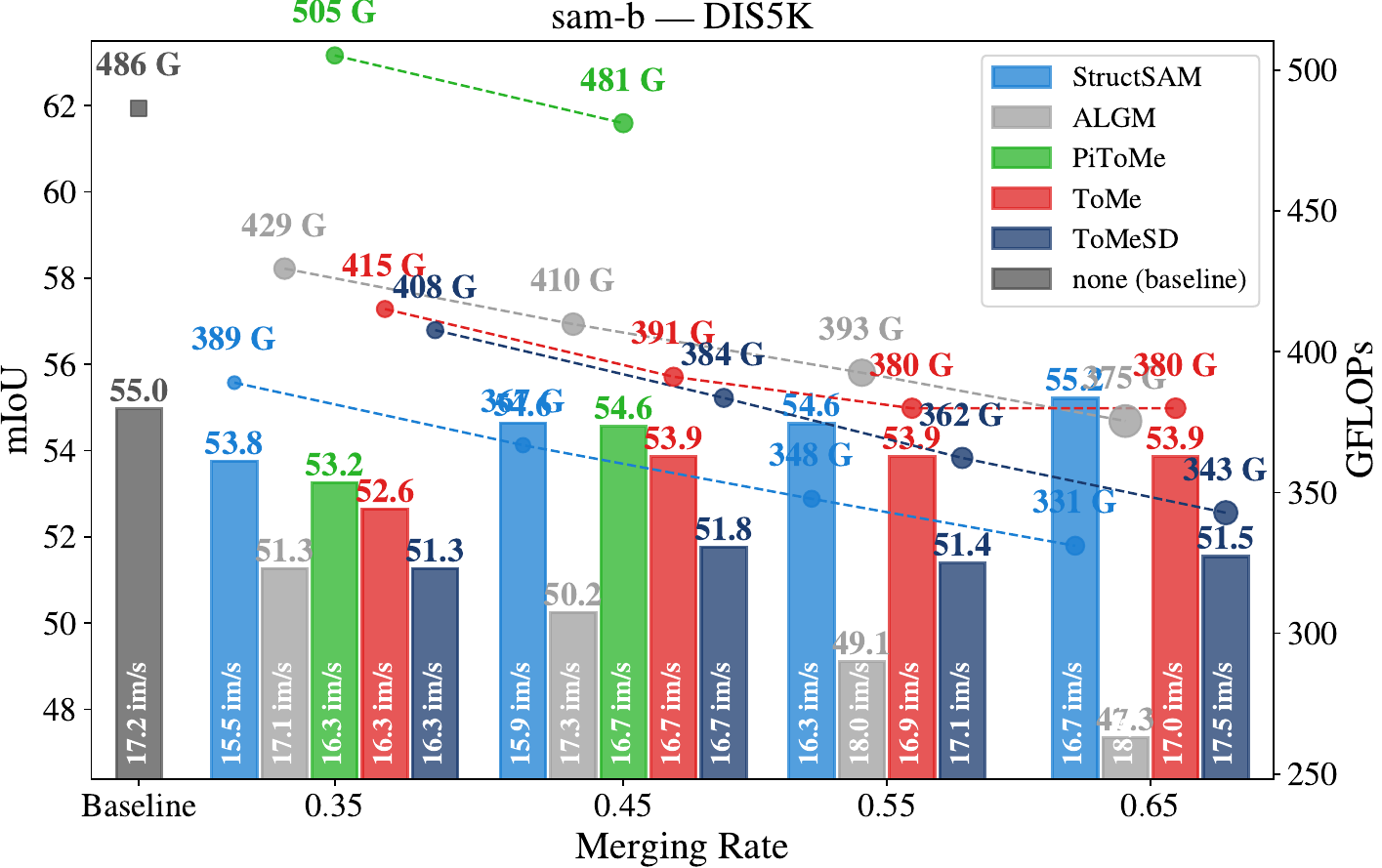}
\caption{Results on DIS5K dataset, SAM with ViT-B backbone.}
\label{fig:samb_flops_ds0}
\end{figure}

\begin{figure}[!hbt]
\centering
\includegraphics[width=0.9\linewidth]{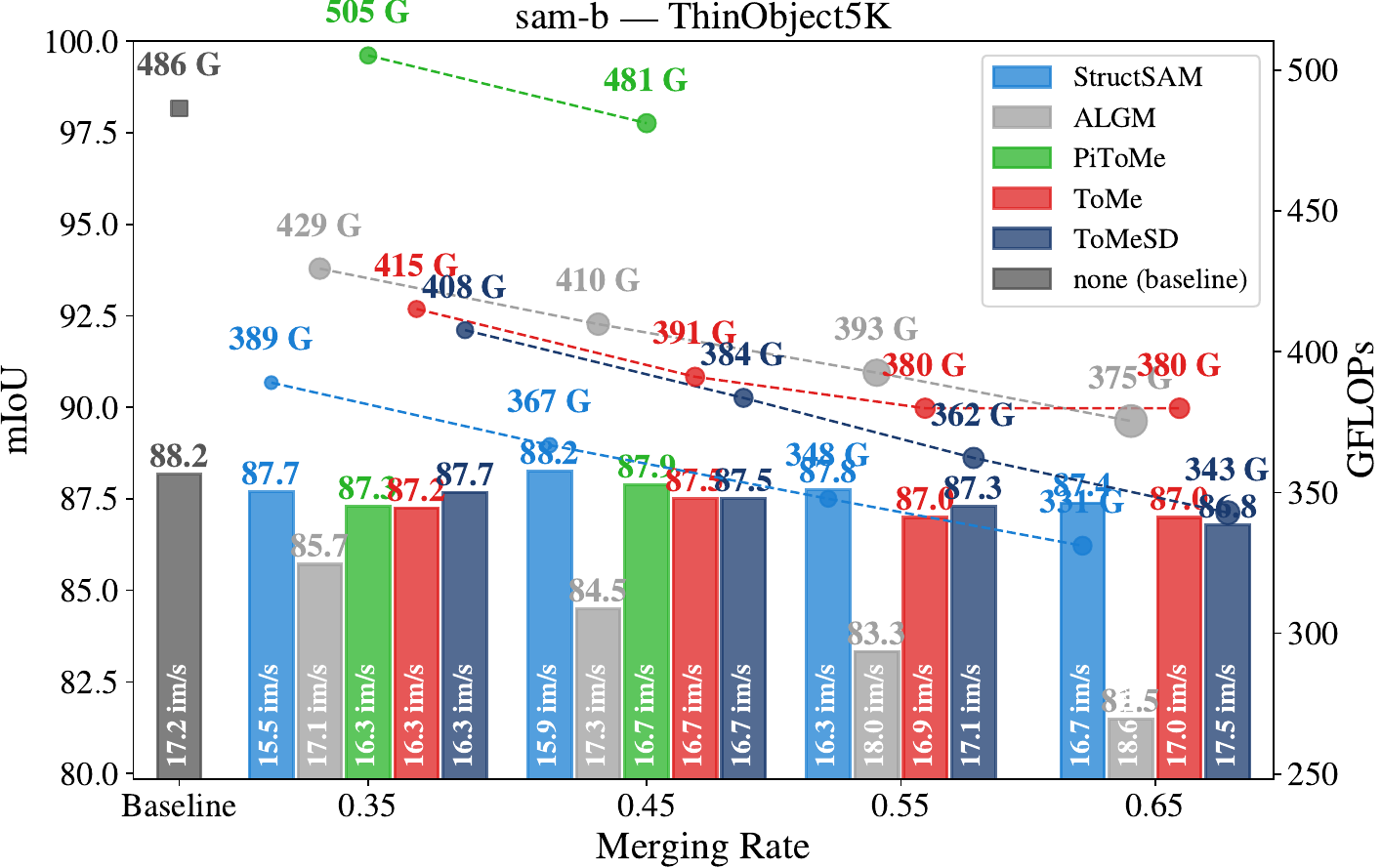}
\caption{Results on ThinObject5K dataset, SAM with ViT-B backbone.}
\label{fig:samb_flops_ds1}
\end{figure}



\begin{figure}[!hbt]
\centering
\includegraphics[width=0.9\linewidth]{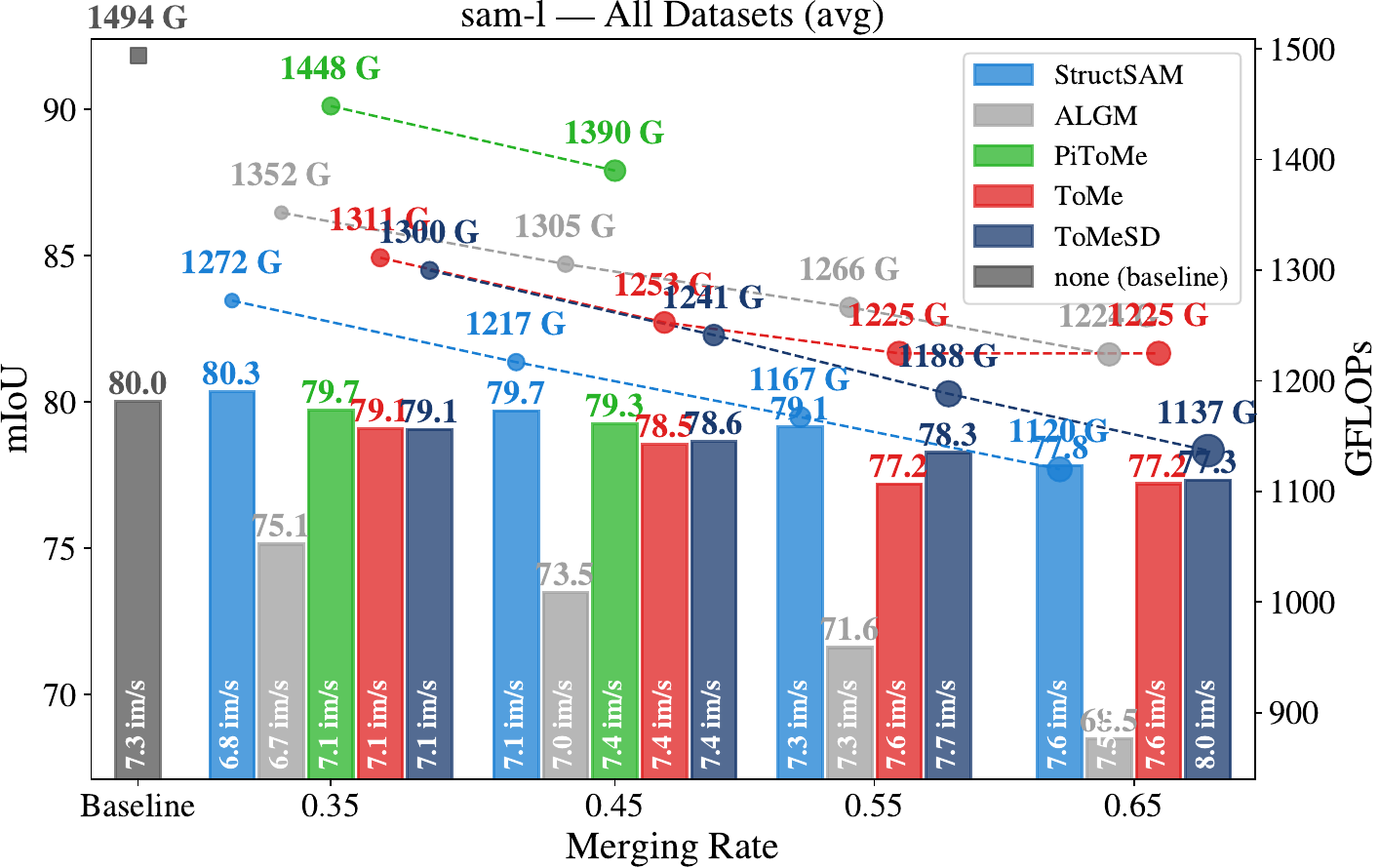}
\caption{Average results on all 4 datasets, SAM with ViT-L backbone.}
\label{fig:saml_flops_ds0}
\end{figure}




\begin{figure}[!hbt]
\centering
\includegraphics[width=0.9\linewidth]{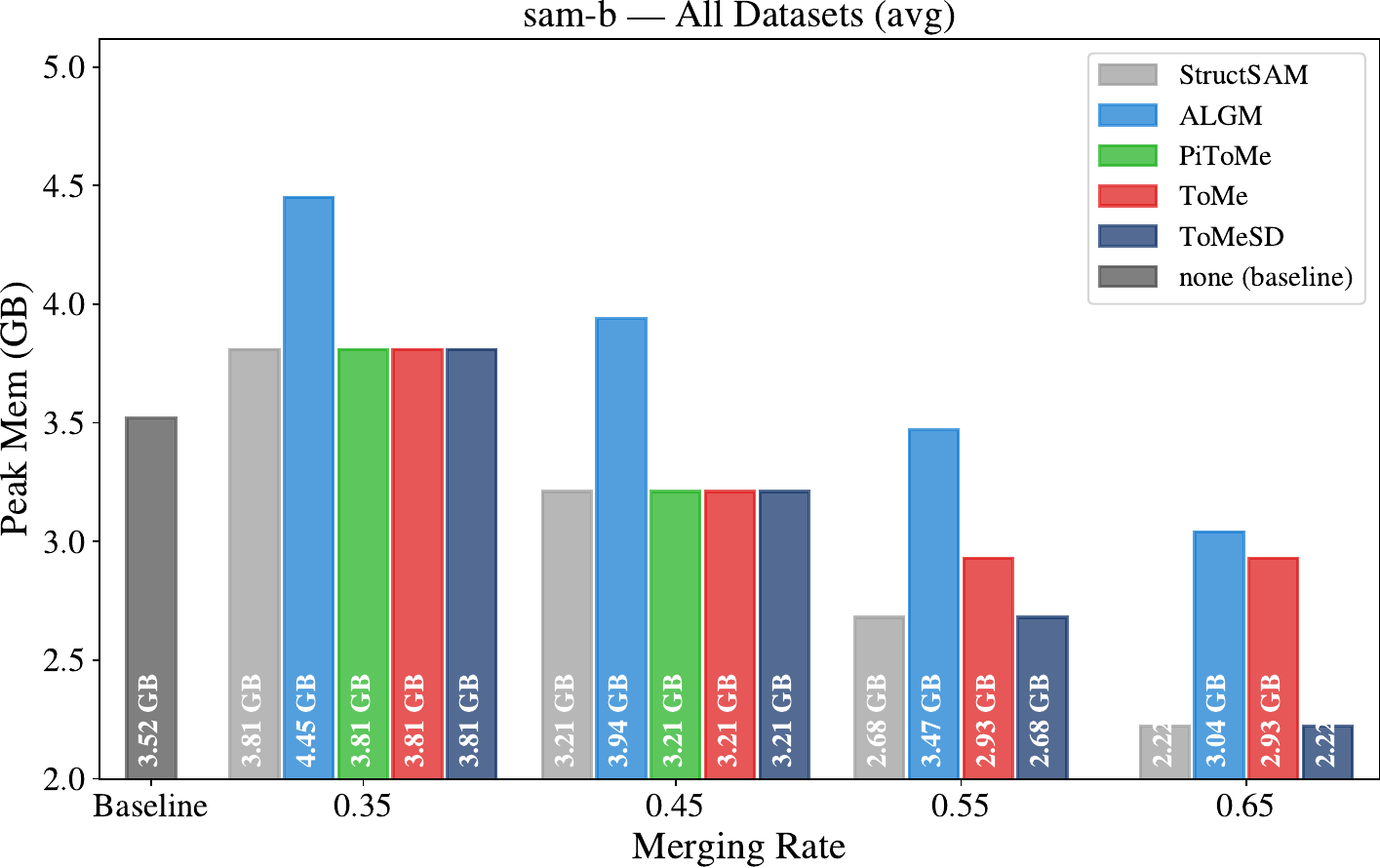}
\caption{Average peak memory results on all 4 datasets, SAM with ViT-b backbone. Note that ToMe does not applicable at above 50\% merging rate, we note the memory consumption of ToMe at 50\% for 0.55 and 0.65}
\label{fig:samb_peakmem_ds0}
\end{figure}

\clearpage

\begin{figure}[!hbt]
\centering
\includegraphics[width=0.9\linewidth]{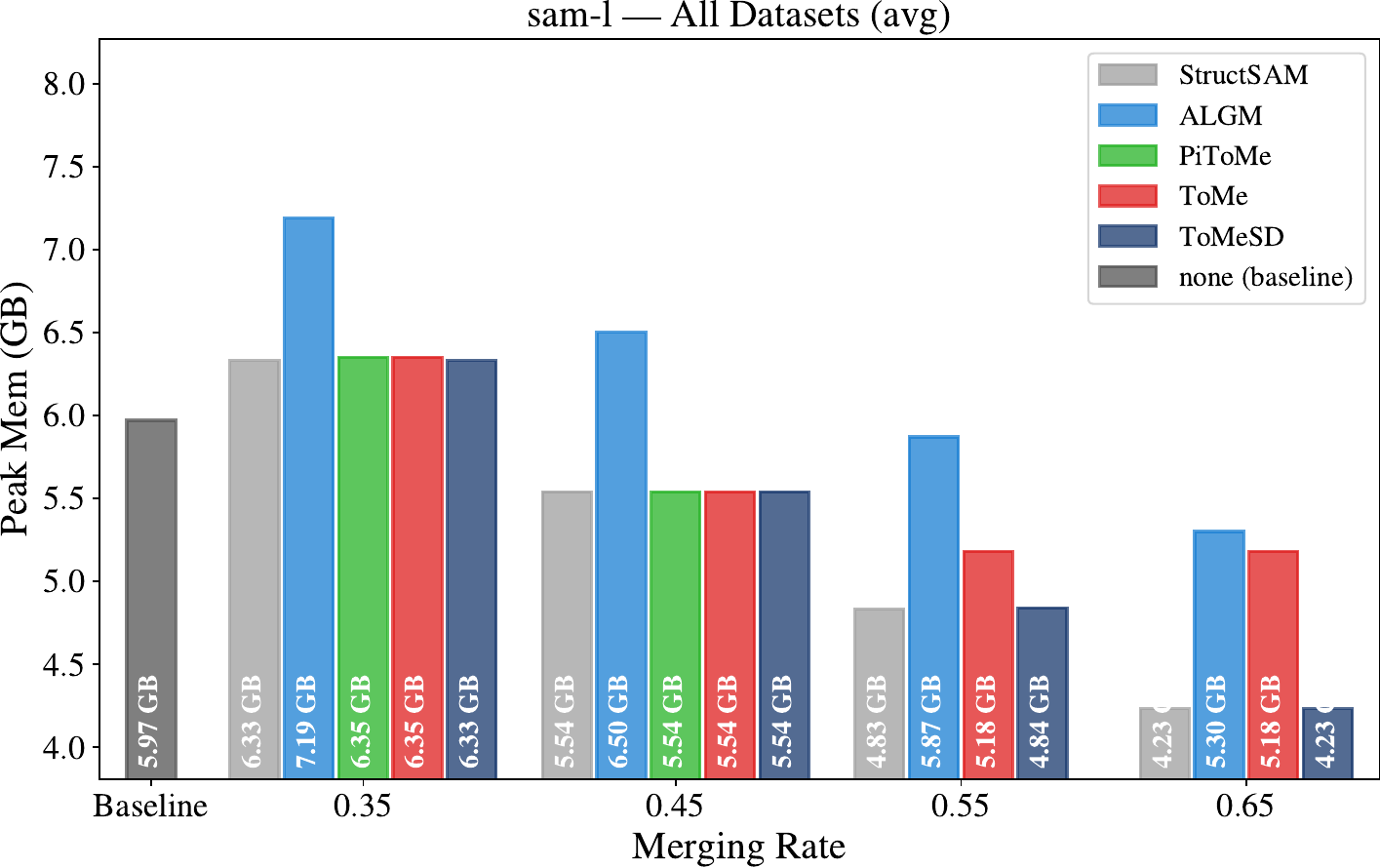}
\caption{Average peak memory results on all 4 datasets, SAM with ViT-l backbone. Note that ToMe does not applicable at above 50\% merging rate, we note the memory consumption of ToMe at 50\% for 0.55 and 0.65}
\label{fig:samb_peakmem_ds1}
\end{figure}

\section{Formal Proofs and Statements of Theoretical Guarantee for \cref{sec_theory_structsam}}
\label{appendix_proofs}

This appendix provides formal statements and proofs supporting the informal
\cref{thm_structsam_spectrum_stability} in the main paper.
We analyse StructSAM’s token merging inside an attention block through spectral graph theory:
merging induces a graph coarsening map on a token graph, while the unmerging step corresponds to
a canonical lifting that restores the original token resolution required by dense mask prediction.

\paragraph{Layerwise spectral discrepancy (main paper).}
At encoder layer \(\ell\), tokens in each attention window \(\mathcal{P}_{\ell,k}\) define a weighted graph
\(\mathcal{G}_{\ell,k}\) with normalized Laplacian \(\mathcal{L}_{\ell,k}\).
After merging and lifting, we obtain a lifted graph \(\mathcal{G}_{\ell,k,l}\) and its Laplacian
\(\mathcal{L}_{\ell,k,l}\).
We measure structural distortion at layer \(\ell\) via
\[
\sd_{\ell}
\;\triangleq\;
\sum_{k=1}^{K_\ell}
\bigl\|
\bslambda_{\ell,k}
-
\bslambda_{\ell,k,l}
\bigr\|_1,
\]
where \(\bslambda_{\ell,k}\) and \(\bslambda_{\ell,k,l}\) are eigenvalues of the original and lifted
normalized Laplacians, respectively.

\subsection{Preliminaries: coarsening, lifting, and eigenvalue inclusion}

We reuse \emph{Graph Coarsening} and \emph{Graph Lifting} (\cref{def_Gc,def_Gl})
window-wise to interpret token merging as a graph coarsening operation and to define the lifted proxy
used in our spectral analysis.

\begin{definition}[Graph Coarsening]\label{def_Gc}
Let \(\cG(\cV,\cE,\bfW)\) be a weighted graph with \(|\cV|=N\) and adjacency matrix
\(\bfW \in \mathbb{R}^{N\times N}\).
Let \(\cP=\{\cV_i\}_{i\in[n]}\) be a partition of \(\cV\) into \(n\) disjoint subsets.
The \emph{coarsened graph} of \(\cG\) with respect to \(\cP\) is the weighted graph
\(\cG_c(\cV_c,\cE_c,\bfW_c)\), where each subset \(\cV_i\) is collapsed into a single node \(\nu_i\in\cV_c\).
Its adjacency entries are defined by block-averaging:
\[
\bfW_c[i,j]
\;=\;
\frac{1}{|\cV_i|\,|\cV_j|}
\sum_{u \in \cV_i}\sum_{v \in \cV_j} \bfW[u,v],
\qquad i,j \in [n].
\]
Let \(\bfD\) be the degree matrix of \(\cG\) with \(\bfD[p,p]=d_p:=\sum_{q=1}^{N}\bfW[p,q]\),
and define the combinatorial and normalized Laplacians:
\[
\bfL = \bfD - \bfW,
\qquad
\cL = \bfI_N - \bfD^{-1/2}\bfW\bfD^{-1/2}.
\]
Similarly define \(\bfD_c\), \(\bfL_c=\bfD_c-\bfW_c\), and
\(\cL_c=\bfI_n-\bfD_c^{-1/2}\bfW_c\bfD_c^{-1/2}\) for \(\cG_c\).
We denote eigenvalues and eigenvectors of \(\cL\) by \((\bslambda,\bfu)\) and those of \(\cL_c\) by \((\bslambda_c,\bfu_c)\).
\end{definition}

\begin{definition}[Graph Lifting]\label{def_Gl}
Given a coarsened graph \(\cG_c(\cV_c,\cE_c,\bfW_c)\) induced by partition \(\cP=\{\cV_i\}_{i\in[n]}\),
the \emph{lifted graph} \(\cG_l(\cV,\cE_l,\bfW_l)\) is defined on the original node set \(\cV\) by
\[
\bfW_l[u,v] \;=\; \bfW_c[i,j],
\qquad \forall\, u\in \cV_i,\ v\in \cV_j,\ i,j\in[n].
\]
Let \(\bfD_l\) be the degree matrix of \(\cG_l\) with \(\bfD_l[p,p]=d_{l p}:=\sum_{q=1}^{N}\bfW_l[p,q]\),
and define
\[
\bfL_l = \bfD_l - \bfW_l,
\qquad
\cL_l = \bfI_N - \bfD_l^{-1/2}\bfW_l\bfD_l^{-1/2}.
\]
We denote eigenvalues and eigenvectors of \(\cL_l\) by \((\bslambda_l,\bfu_l)\).
\end{definition}

\begin{lemma}[Eigenvalue inclusion under lifting]\label{lem_structsam_eig_inclusion}
Let \(\cG_c\) be coarsened from \(\cG\) by a partition \(\cP\), and let \(\cG_l\) be the lifted graph.
Then the eigenvalues of \(\cL_l\) contain all eigenvalues of \(\cL_c\), and the remaining \((N-n)\)
eigenvalues equal \(1\).
\end{lemma}

\begin{proof}
A standard lifting argument shows that \(\cL_l\) has an invariant subspace of vectors constant on each cluster,
on which \(\cL_l\) is similar to \(\cL_c\). The orthogonal complement contributes eigenvalue \(1\) with multiplicity \(N-n\).
\end{proof}

\subsection{Token graphs for windowed and global attention}

Fix an encoder layer \(\ell\) with window partition \(\{\mathcal{P}_{\ell,k}\}_{k=1}^{K_\ell}\), where \(K_\ell\ge 1\).
For local-attention layers, \(\mathcal{P}_{\ell,k}\) are spatial windows; for global attention layers,
\(K_\ell=1\) and \(\mathcal{P}_{\ell,1}=\mathcal{X}\).
For each window \((\ell,k)\), define a weighted token graph
\(\cG_{\ell,k}(\cV_{\ell,k},\cE_{\ell,k},\bfW_{\ell,k})\) on the tokens in \(\mathcal{P}_{\ell,k}\).
Let its normalized Laplacian be
\[
\cL_{\ell,k}
\;=\;
\bfI_{N_{\ell,k}}-\bfD_{\ell,k}^{-1/2}\bfW_{\ell,k}\bfD_{\ell,k}^{-1/2},
\qquad N_{\ell,k}=|\cV_{\ell,k}|.
\]

StructSAM performs merging within each window (equivalently, within grid cells aligned to the window),
producing a coarsened graph \(\cG_{\ell,k,c}\); lifting yields \(\cG_{\ell,k,l}\) on \(N_{\ell,k}\) nodes.
Let \(\bslambda_{\ell,k}\) and \(\bslambda_{\ell,k,l}\) be eigenvalues of \(\cL_{\ell,k}\) and \(\cL_{\ell,k,l}\), respectively.

\subsection{Merge correctness event}

Fix \((\ell,k)\). Let \(\cP_{0,\ell,k}^{(s)}=\{\cV_{0,\ell,k,1}^{(s)},\dots,\cV_{0,\ell,k,s_{\ell,k}}^{(s)}\}\)
be an (unknown) semantic partition of the window at merge step \(s\) (e.g., local regions separated by edges).
At step \(s\), the algorithm merges a pair \((v_{a_{\ell,k,s}},v_{b_{\ell,k,s}})\) inside that window.
Define the within-region merge event
\begin{equation}\label{eq_event_E_lks}
\mathcal{E}_{\ell,k,s}
\;\triangleq\;
\left\{\exists i\in[s_{\ell,k}]\ \text{s.t.}\ v_{a_{\ell,k,s}},v_{b_{\ell,k,s}}\in \cV_{0,\ell,k,i}^{(s)}\right\}.
\end{equation}

\subsection{Assumptions}\label{sec_ass}

We keep \cref{ass_A1,ass_A2} identical in spirit to PiToME’s Theorem~1 and establish a
gradient-separation \cref{ass_A3_gradsep} that matches StructSAM’s flatness screening rule and implies
\(\delta_{\ell,k,s}\to 0\).

\begin{assumption}[Within-region concentration]\label{ass_A1}
For each merge step \(s\) and each true part \(\cV_{0,\ell,k,i}^{(s)}\),
\[
\mathbb{E}\big[\cos(v_u,v_v)\big]\to 1,
\qquad \forall\, v_u,v_v\in \cV_{0,\ell,k,i}^{(s)}.
\]
\end{assumption}

\begin{assumption}[Margin across regions]\label{ass_A2}
There exists a margin \(m\in(0,1)\) such that for all \(i\neq j\),
\[
\cos(v_u,v_v)\ \ge\ m\ >\ \cos(v_u,v_w),
\qquad
\forall\, v_u,v_v\in \cV_{0,\ell,k,i}^{(s)},\ \forall\, v_w\in \cV_{0,\ell,k,j}^{(s)}.
\]
\end{assumption}

\begin{assumption}[Gradient separation for flatness screening]\label{ass_A3_gradsep}
Fix a layer \(\ell\) and an attention window \(\mathcal{P}_{\ell,k}\).
Partition \(\mathcal{P}_{\ell,k}\) into disjoint grid cells
\(\{\mathcal{C}_{\ell,k,m}\}_{m=1}^{M_{\ell,k}}\) aligned to window geometry, each of size \(s\times s\).
Let \(S_\ell(x)=\mathbf{G}^{(\ell)}(x)\) be the feature gradient–based energy score.
Define the cell flatness score
\[
\phi_{\ell,k}(\mathcal{C}_{\ell,k,m}) \;:=\; -\max_{x\in \mathcal{C}_{\ell,k,m}} S_\ell(x).
\]
StructSAM selects the mergeable-cell set \(\mathcal{M}^{\mathrm{mer}}_{\ell,k}\) as the \(\rho M_{\ell,k}\) cells
with largest \(\phi_{\ell,k}\), and defines the protected set
\(\mathcal{M}^{\mathrm{pro}}_{\ell,k}=[M_{\ell,k}]\setminus \mathcal{M}^{\mathrm{mer}}_{\ell,k}\).

Assume there exist constants \(0<\tau_{\mathrm{in}}<\tau_{\mathrm{bd}}\) and a function
\(\delta_{\ell,k}(s)\in(0,1)\) such that for each merge step \(s\):
\begin{enumerate}[label=(A3\alph*)]
\item \textbf{Boundary-gradient lower bound.}
If a cell \(\mathcal{C}_{\ell,k,m}\) intersects two or more true parts of \(\cP^{(s)}_{0,\ell,k}\), then
\[
\max_{x\in \mathcal{C}_{\ell,k,m}} S_\ell(x) \;\ge\; \tau_{\mathrm{bd}}
\quad\text{with probability at least } 1-\delta_{\ell,k}(s).
\]
\item \textbf{Interior-gradient upper bound.}
If a cell \(\mathcal{C}_{\ell,k,m}\) is fully contained in a single true part of \(\cP^{(s)}_{0,\ell,k}\), then
\[
\max_{x\in \mathcal{C}_{\ell,k,m}} S_\ell(x) \;\le\; \tau_{\mathrm{in}}
\quad\text{with probability at least } 1-\delta_{\ell,k}(s).
\]
\item \textbf{Mergeable-cell budget.}
The merge rate \(\rho\) satisfies
\[
\rho \;\le\; \frac{\#\{\text{interior cells in }(\ell,k)\}}{M_{\ell,k}},
\]
so that mergeable cells can be chosen exclusively among interior cells whenever (A3a)–(A3b) hold.
\end{enumerate}
Moreover, \(\delta_{\ell,k}(s)\to 0\) as token resolution increases (equivalently, as cell size decreases at fixed image resolution).
\end{assumption}

\begin{lemma}[Gradient separation implies boundary protection]\label{lem_gradsep_implies_protection}
Under \cref{ass_A3_gradsep}, with probability at least \(1-2\,\delta_{\ell,k}(s)\), every mergeable cell selected by StructSAM
is an interior cell (i.e., it does not intersect a boundary). In particular, defining
\[
\delta_{\ell,k,s} \;:=\; 2\,\delta_{\ell,k}(s),
\]
we have \(\delta_{\ell,k,s}\to 0\).
\end{lemma}

\begin{proof}
On the event that (A3a) and (A3b) hold, every boundary cell has \(\max S_\ell \ge \tau_{\mathrm{bd}}\) while every interior cell has
\(\max S_\ell \le \tau_{\mathrm{in}}\), with \(\tau_{\mathrm{in}}<\tau_{\mathrm{bd}}\).
Hence all cells with smallest values of \(\max S_\ell\) are interior. By (A3c), StructSAM can select its mergeable-cell budget from these interior cells.
A union bound over the two failure events yields probability at least \(1-2\,\delta_{\ell,k}(s)\).
\end{proof}

\begin{lemma}[From gradient screening to merge correctness]\label{lem_gradsep_implies_correct_merge}
Under \cref{ass_A2,ass_A3_gradsep}, for each merge step \(s\),
\[
\mathbb{P}\!\left(\mathcal{E}_{\ell,k,s}\right)\ \ge\ 1-\delta_{\ell,k,s},
\qquad \text{where } \delta_{\ell,k,s}=2\,\delta_{\ell,k}(s)\to 0.
\]
\end{lemma}

\begin{proof}
By \cref{lem_gradsep_implies_protection}, with probability at least \(1-\delta_{\ell,k,s}\),
all mergeable cells are interior (contained in a single true part).
Within an interior cell, StructSAM chooses a destination token inside that same true part.
By the margin assumption \cref{ass_A2}, BSM assigns each source token to a destination in the same true part,
hence no cross-part merges occur and \(\mathcal{E}_{\ell,k,s}\) holds.
\end{proof}

\begin{proposition}[Failure of \cref{ass_A3_gradsep} under coarse cell partitions]\label{prop_A3_failure_coarse_grid}
There exist token graphs satisfying \cref{ass_A1,ass_A2} for which \cref{ass_A3_gradsep} fails solely due to
the choice of cell size, even when the gradient score \(S_\ell=\mathbf{G}^{(\ell)}\) perfectly separates boundary
tokens from interior tokens.

In particular, fix a window \(\mathcal{P}_{\ell,k}\) whose tokens lie on a \(H\times W\) grid, and consider a latent
partition into two true parts separated by a boundary curve that intersects every grid cell of a given
cell partition \(\{\mathcal{C}_{\ell,k,m}\}_{m=1}^{M_{\ell,k}}\).
Then \(\#\{\text{interior cells}\}=0\), and thus the mergeable-cell budget condition (A3c) fails for any \(\rho>0\).
Consequently, StructSAM must select mergeable cells that are boundary/mixed cells, so the boundary-exclusion
conclusion of \cref{ass_A3_gradsep} cannot hold.
\end{proposition}

\begin{proof}
Construct embeddings as follows. Let all tokens in true part 1 share a unit vector \(u\), and all tokens in
true part 2 share a unit vector \(v\) such that \(u^\top v=\gamma<m\). Then within-part cosine similarity equals \(1\),
so \cref{ass_A1} holds (trivially), and cross-part similarity equals \(\gamma<m\), so \cref{ass_A2} holds.

Now choose a cell partition with no interior cells, i.e.\ every cell intersects both true parts (for example,
take a single cell covering the entire window, or take cells so large that each cell straddles the boundary).
Then \(\#\{\text{interior cells}\}=0\), and the condition (A3c),
\(\rho \le \#\{\text{interior cells}\}/M_{\ell,k}\), fails for any \(\rho>0\).
Hence StructSAM must choose at least one mergeable cell that is mixed.
Even if the gradient score perfectly separates boundary tokens from interior tokens, there are no interior cells
to select, so the boundary-exclusion mechanism cannot be satisfied. 
\end{proof}

\subsection{Practicality of the assumptions in \cref{sec_ass}}\label{sec:a3_practicality_pca}

\paragraph{Why \cref{ass_A1,ass_A2} are plausible in foundation encoders.}
\cref{ass_A1,ass_A2} posit that token embeddings concentrate within coherent regions and admit a
margin across different regions. In SAM-style foundation encoders, this behaviour is empirically supported by
the strong spatial organisation visible in PCA projections of token features.
In \cref{fig:pca_layers}, the \emph{Feat. PCA (0\% merge)} panels exhibit piecewise-smooth colour structure that
aligns with salient objects and backgrounds across early (Layer~0) through deeper layers (Layer~11), suggesting
that tokens within a region occupy a compact neighbourhood in feature space, while tokens across different
regions remain separated. This is consistent with within-region similarity concentration and cross-region
margin, motivating \cref{ass_A1,ass_A2} in practice.

\paragraph{Decomposing \cref{ass_A3_gradsep}: score separation vs.\ geometric budget.}
Assumption~\ref{ass_A3_gradsep} decomposes into two ingredients.

\textbf{Score separation (A3a)–(A3b).}
These conditions require that the cell-wise statistic
\(\max_{x\in\mathcal{C}}\mathbf{G}^{(\ell)}(x)\) separates boundary/mixed cells from interior cells.
This is a standard edge-detection heuristic: a boundary-crossing cell typically contains at least one
high-gradient token, while interior cells remain low-gradient.
The PCA diagnostics in \cref{fig:pca_layers} support this separation across layers even under aggressive merging.
Specifically, comparing \emph{Feat. PCA (0\% merge)} and \emph{Feat. PCA (65\%)} shows that the dominant spatial
structure is largely preserved, while the \emph{Difference} maps remain sparse and localised rather than
diffuse. This indicates that the merge--unmerge perturbation concentrates on a small subset of locations
and does not globally scramble the feature geometry, making it plausible that gradient-energy remains a stable
signal for identifying boundary-sensitive cells across layers.

\textbf{Geometric/budget condition (A3c).}
Condition (A3c) is structural: the mergeable-cell ratio \(\rho\) (equivalently, the cell size \(s\))
must be chosen so that sufficiently many \emph{interior cells} exist to populate the mergeable set.
If \(s\) is too large (or \(\rho\) is too aggressive), interior cells may be absent or too few, forcing the
selection of boundary/mixed cells and invalidating (A3c), as formalised by
\cref{prop_A3_failure_coarse_grid}. In practice, StructSAM’s use of window-aligned cells and moderate cell sizes
typically yields many interior cells per window, so (A3c) is satisfied except in extreme high-merge regimes or
for very thin structures whose width is comparable to the chosen cell size.

\paragraph{Takeaway.}
Together, \cref{fig:pca_layers} and the sparse difference patterns across layers provide empirical support that
the representation geometry of SAM’s encoder makes \cref{ass_A1,ass_A2} and the score-separation component
(A3a)–(A3b) easy to satisfy, while (A3c) highlights the practical trade-off between merge rate and the
availability of interior cells under window-aligned partitioning.

\subsection{Formal theorem matching the main-paper informal statement in \cref{thm_structsam_spectrum_stability}}

\begin{theorem}[Formal: Layerwise spectrum stability of score-guided merging]\label{thm_structsam_spectrum_stability_formal}
Fix an encoder layer \(\ell\) with windows \(\{\mathcal{P}_{\ell,k}\}_{k=1}^{K_\ell}\).
Let \(\sd_{\ell}(\mathrm{SG})\) denote the spectral discrepancy induced by StructSAM’s score-guided merging,
and \(\sd_{\ell}(\mathrm{Base})\) that of a non–score-guided baseline (e.g., random or stride-based dst selection).
Assume bounded degrees and bounded weights in each window:
there exist constants \(0<d_{\min}\le d_{\max}<\infty\) and \(0<w_{\max}<\infty\) such that, for all merge steps,
\[
d_{\min}\ \le\ d_{\ell,k}^{(s)}(i)\ \le\ d_{\max},
\qquad
0\le \bfW_{\ell,k}^{(s)}[i,j]\le w_{\max}.
\]
Then:
\begin{enumerate}
\item Under \cref{ass_A1,ass_A2,ass_A3_gradsep}, we have \(\mathbb{E}[\sd_{\ell}(\mathrm{SG})]\to 0\).
\item If the baseline has a non-vanishing probability of cross-region merging, i.e., there exists \(\delta>0\) such that
\[
\inf_{k,s}\ \mathbb{P}\!\left(\mathcal{E}_{\ell,k,s}\right)\ \le\ 1-\delta,
\]
then \(\liminf\,\mathbb{E}[\sd_{\ell}(\mathrm{Base})]\ge c_0\,\delta>0\) for a constant \(c_0\) depending only on
\((m,d_{\min},d_{\max},w_{\max})\).
\end{enumerate}
\end{theorem}

\subsection{A baseline counterexample (ToMeSD-style dst selection)}

\begin{proposition}[ToMeSD admits non-vanishing cross-region merges]\label{prop_tomesd_counterexample}
Assume \cref{ass_A1,ass_A2}. Consider a window \(\mathcal{P}_{\ell,k}\) that contains two true parts
\(\cV_{0,\ell,k,1}^{(s)}\) and \(\cV_{0,\ell,k,2}^{(s)}\) separated by a boundary that intersects at least one grid cell
\(\mathcal{C}\) (a \emph{mixed cell}). Suppose the baseline destination-selection rule chooses a destination
token in \(\mathcal{C}\) with probability at least \(p_0>0\) and, conditional on selecting \(\mathcal{C}\), picks a
destination from the non-dominant part in \(\mathcal{C}\) with probability at least \(q_0>0\).
Then for infinitely many merge steps \(s\),
\[
\mathbb{P}(\mathcal{E}_{\ell,k,s}^c)\ \ge\ p_0 q_0,
\]
and hence \(\inf_{k,s}\mathbb{P}(\mathcal{E}_{\ell,k,s})\le 1-p_0q_0\).
\end{proposition}

\begin{proof}
On the event that a destination is chosen in the mixed cell \(\mathcal{C}\) from the non-dominant true part,
there exist source tokens in the dominant true part within \(\mathcal{C}\).
Because the baseline does not enforce boundary protection, at least one such source token must be assigned
to a destination in the other true part, and therefore a cross-region merge occurs, i.e.\ \(\mathcal{E}_{\ell,k,s}^c\) holds.
The claim follows from the lower bounds \(p_0,q_0\).
\end{proof}

\begin{figure}
    \centering
    \includegraphics[width=0.9\linewidth]{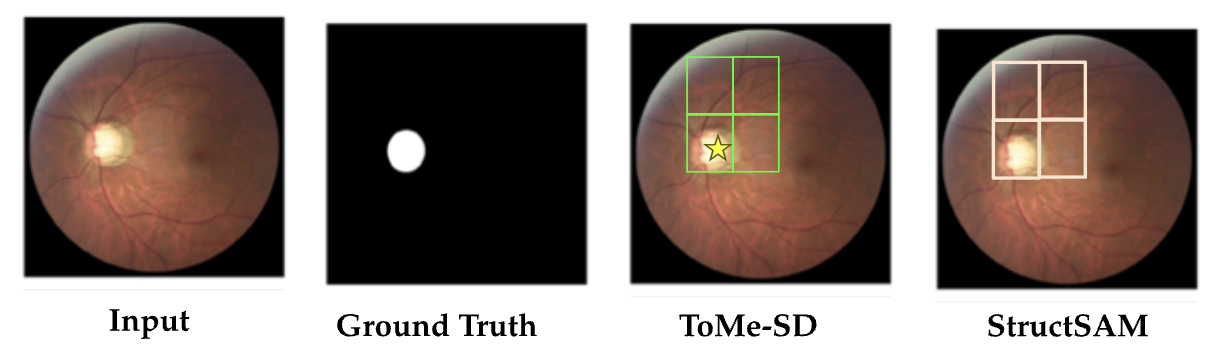}
    \caption{\textbf{Illustrative failure case of ToMeSD}. Stride-based destination selection may place a destination token inside a boundary-crossing (mixed) cell, forcing cross-region token merges that degrade dense segmentation quality. In the example shown for ToMe-SD, the destination token is sampled from the ground-truth cell, causing the algorithm to merge structurally inconsistent neighboring tokens and produce erroneous segmentation. In contrast, StructSAM uses flatness-based screening to protect mixed cells and preserve structural boundaries.}
    \label{fig:tomesd_structsam_counterexample}
\end{figure}

\noindent
\cref{fig:tomesd_structsam_counterexample} visualises the mixed-cell event in \cref{prop_tomesd_counterexample}:
ToMeSD may select a destination in a boundary-crossing cell, forcing cross-region merges, whereas StructSAM’s flatness screening avoids selecting such cells.

\subsection{Proof of \cref{thm_structsam_spectrum_stability_formal}}\label{sec_proof_thm_structsam_spectrum_stability_formal}
We follow the same high-level route as in PiToME’s Appendix E:
(i) control a one-step adjacency discrepancy,
(ii) translate it to a Laplacian perturbation bound,
(iii) convert Laplacian perturbation to eigenvalue drift via Hoffman--Wielandt,
(iv) telescope over merge steps and sum over windows.

\paragraph{One-step discrepancy.}
At merge step \(s\) in window \((\ell,k)\), StructSAM merges indices \((a_{\ell,k,s},b_{\ell,k,s})\).
Define the one-step (row/column) discrepancy
\begin{equation}\label{eq_Delta_def}
\Delta_{\ell,k,s}
\;\triangleq\;
\bigl\|\bfW_{\ell,k}^{(s)}[a_{\ell,k,s},:]-\bfW_{\ell,k}^{(s)}[b_{\ell,k,s},:]\bigr\|_1
+
\bigl\|\bfW_{\ell,k}^{(s)}[:,a_{\ell,k,s}]-\bfW_{\ell,k}^{(s)}[:,b_{\ell,k,s}]\bigr\|_1 .
\end{equation}
For symmetric affinities, the two terms coincide; we keep both for generality.

\begin{proposition}[Row/column drift under correct vs.\ incorrect merges]\label{prop_row_drift}
There exist constants \(c_{\mathrm{row}}>0\) and \(c_0>0\), depending only on the affinity construction and bounds,
such that:
\begin{enumerate}
\item On \(\mathcal{E}_{\ell,k,s}\),
\[
\Delta_{\ell,k,s}
\ \le\
c_{\mathrm{row}}\sqrt{2\bigl(1-\cos(v_{a_{\ell,k,s}},v_{b_{\ell,k,s}})\bigr)}.
\]
\item On \(\mathcal{E}_{\ell,k,s}^c\), under \cref{ass_A2} we have
\[
\Delta_{\ell,k,s}\ \ge\ c_0(1-m).
\]
\end{enumerate}
\end{proposition}

\begin{proof}
(1) For cosine-type affinities \(\bfW[i,j]=\psi(\cos(v_i,v_j))\) with \(\psi\) Lipschitz,
\(|\bfW[a,j]-\bfW[b,j]|\le L_\psi|\cos(v_a,v_j)-\cos(v_b,v_j)|\le L_\psi\|v_a-v_b\|_2\).
Summing over \(j\) inside the window yields \(\|\bfW[a,:]-\bfW[b,:]\|_1 \le c'\|v_a-v_b\|_2\),
and similarly for the column term. Using \(\|v_a-v_b\|_2^2=2(1-\cos(v_a,v_b))\) gives the claim.

(2) If \(\mathcal{E}_{\ell,k,s}^c\) occurs, the merged pair lies in different true parts.
By \cref{ass_A2}, within-part similarities are at least \(m\) while cross-part similarities are strictly below \(m\),
which implies a nontrivial mismatch in adjacency rows/columns to tokens in at least one of the parts.
This yields an \(\ell_1\) discrepancy bounded below by a constant multiple of \(1-m\).
\end{proof}

\begin{proposition}[Adjacency-to-Laplacian perturbation]\label{prop_L_perturb}
Let \(\cL_{\ell,k}^{(s)}\) be the normalized Laplacian before the merge at step \(s\),
and let \(\cL_{\ell,k,l}^{(s-1)}\) be the lifted normalized Laplacian after that merge.
Under the boundedness condition in \cref{thm_structsam_spectrum_stability_formal}, there exists \(c_{\mathrm{lap}}>0\) such that
\[
\bigl\|\cL_{\ell,k}^{(s)}-\cL_{\ell,k,l}^{(s-1)}\bigr\|_{F}
\ \le\
c_{\mathrm{lap}}\ \Delta_{\ell,k,s}.
\]
\end{proposition}

\begin{proof}
Write \(\cL=\bfI-\bfD^{-1/2}\bfW\bfD^{-1/2}\).
A single merge changes \(\bfW\) only through the merged pair’s rows/columns under coarsening and lifting,
hence \(\|\bfW-\bfW_l\|_F\) is controlled by \(\Delta_{\ell,k,s}\).
Degree matrices differ by row sums of \(\bfW\), so \(\|\bfD-\bfD_l\|_F\) is also controlled by \(\Delta_{\ell,k,s}\).
Using degree bounds, \(\|\bfD^{-1/2}\|_2\) and \(\|\bfD_l^{-1/2}\|_2\) are uniformly bounded.
A triangle inequality expansion then yields the stated Frobenius bound.
\end{proof}

\begin{lemma}[Hoffman--Wielandt]\label{lem_HW}
For symmetric matrices \(\bfA,\bfB\in\mathbb{R}^{N\times N}\) with eigenvalues
\(\bsalpha\) and \(\bsbeta\),
\[
\|\bsalpha-\bsbeta\|_1 \le \sqrt{N}\,\|\bfA-\bfB\|_F.
\]
\end{lemma}

\begin{proof}
By Hoffman--Wielandt, \(\|\bsalpha-\bsbeta\|_2\le \|\bfA-\bfB\|_F\). Apply Cauchy--Schwarz to obtain the \(\ell_1\) bound.
\end{proof}

\begin{proposition}[One-step eigenvalue drift]\label{prop_one_step_spec}
There exists \(c_{\mathrm{sp}}>0\) such that for each merge step \(s\),
\[
\bigl\|\bslambda_{\ell,k}^{(s)}-\bslambda_{\ell,k,l}^{(s-1)}\bigr\|_1
\ \le\
c_{\mathrm{sp}}\ \Delta_{\ell,k,s}.
\]
\end{proposition}

\begin{proof}
Apply \cref{lem_HW} with \(\bfA=\cL_{\ell,k}^{(s)}\) and \(\bfB=\cL_{\ell,k,l}^{(s-1)}\), then use \cref{prop_L_perturb}.
\end{proof}

\begin{proof}[Proof of \cref{thm_structsam_spectrum_stability_formal}]
\emph{Step 1 (telescoping).}
For fixed \((\ell,k)\), telescope over merge steps:
\[
\|\bslambda_{\ell,k}-\bslambda_{\ell,k,l}\|_1
\le
\sum_{s=n_{\ell,k}+1}^{N_{\ell,k}}
\|\bslambda_{\ell,k}^{(s)}-\bslambda_{\ell,k,l}^{(s-1)}\|_1
\le
c_{\mathrm{sp}}
\sum_{s=n_{\ell,k}+1}^{N_{\ell,k}}
\Delta_{\ell,k,s}.
\]
Summing over \(k=1,\dots,K_\ell\) yields
\[
\sd_\ell
\le
c_{\mathrm{sp}}
\sum_{k=1}^{K_\ell}\sum_{s=n_{\ell,k}+1}^{N_{\ell,k}}
\Delta_{\ell,k,s}.
\]

\emph{Step 2 (score-guided vanishing).}
By \cref{lem_gradsep_implies_correct_merge},
\(\mathbb{P}(\mathcal{E}_{\ell,k,s}^c)\le \delta_{\ell,k,s}\) with \(\delta_{\ell,k,s}\to 0\).
Using \cref{prop_row_drift} and the law of total expectation,
\[
\mathbb{E}[\Delta_{\ell,k,s}]
\le
c_{\mathrm{row}}
\mathbb{E}\Big[\sqrt{2\bigl(1-\cos(v_{a_{\ell,k,s}},v_{b_{\ell,k,s}})\bigr)}\Big]
+
c_0(1-m)\,\delta_{\ell,k,s}.
\]
Under \cref{ass_A1}, the cosine term converges to \(1\) on correct merges, hence the first term vanishes.
Since \(\delta_{\ell,k,s}\to 0\) by construction, we obtain \(\mathbb{E}[\Delta_{\ell,k,s}]\to 0\), and therefore
\(\mathbb{E}[\sd_\ell(\mathrm{SG})]\to 0\).

\emph{Step 3 (baseline non-vanishing).}
By \cref{prop_tomesd_counterexample}, ToMeSD-style destination selection can yield a non-vanishing lower bound
\(\mathbb{P}(\mathcal{E}_{\ell,k,s}^c)\ge \delta\) for some \(\delta>0\) on infinitely many merge steps.
By the lower bound in \cref{prop_row_drift}, this implies
\(\mathbb{E}[\Delta_{\ell,k,s}]\ge c_0(1-m)\delta\), which yields
\(\liminf\,\mathbb{E}[\sd_\ell(\mathrm{Base})]\ge c_0\delta>0\).
\end{proof}



\end{document}